\documentclass[12pt,a4paper,reqno]{amsart}

\newif\ifcomment
\commenttrue

\usepackage{xcolor,colortbl}
\usepackage{amsmath,amsfonts,bm}
\usepackage{amssymb,amsthm,mathtools}
\usepackage[backref=page, colorlinks=true, linkcolor=blue, citecolor=blue]{hyperref}
\usepackage{float}
\usepackage{multirow}
\usepackage{makecell}
\usepackage{longtable}
\usepackage{wrapfig}
\usepackage{bbm}
\usepackage{algorithm}
\usepackage{algpseudocode}
\usepackage{graphicx}
\usepackage{subfloat}
\usepackage[linewidth=1pt,framemethod=TikZ]{mdframed}
\usepackage{natbib}
\usepackage{soul}
\usepackage[nameinlink]{cleveref}
\newmdtheoremenv{theo}{Theorem}
\usepackage[noEnd=false]{algpseudocodex}
\usepackage{listings}
\usepackage[toc,titletoc,page]{appendix}

\allowdisplaybreaks

\def\eqref#1{equation~\ref{#1}}

\def\1{\bm{1}}

\algrenewcommand\algorithmicrequire{\textbf{Require:}}
\algrenewcommand\algorithmicensure{\textbf{Output:}}

\algnewcommand\algorithmicinput{\textbf{Input:}}
\algnewcommand\Input{\item[\algorithmicinput]}

\algnewcommand\algorithmichyper{\textbf{Hyper-parameters:}}
\algnewcommand\Hyper{\item[\algorithmichyper]}

\def\eps{{\epsilon}}

\def\supp{{\mathrm{supp}}}

\def\tpi{{\tilde{\pi}}}
\DeclareMathAlphabet{\mathsfit}{\encodingdefault}{\sfdefault}{m}{sl}
\SetMathAlphabet{\mathsfit}{bold}{\encodingdefault}{\sfdefault}{bx}{n}

\def\gA{{\mathcal{A}}}
\def\gB{{\mathcal{B}}}
\def\gC{{\mathcal{C}}}
\def\gD{{\mathcal{D}}}
\def\gE{{\mathcal{E}}}
\def\gF{{\mathcal{F}}}
\def\gG{{\mathcal{G}}}

\def\gL{{\mathcal{L}}}
\def\gM{{\mathcal{M}}}

\def\gO{{\mathcal{O}}}
\def\gP{{\mathcal{P}}}
\def\gQ{{\mathcal{Q}}}

\def\gS{{\mathcal{S}}}

\def\gX{{\mathcal{X}}}

\def\sE{{\mathbb{E}}}

\def\sN{{\mathbb{N}}}

\def\sR{{\mathbb{R}}}

\def\sT{{\mathbb{T}}}

\def\sV{{\mathbb{V}}}

\def\proj{{\mathrm{Proj}}}

\def\ubar#1{{\underline{#1}}}
\def\vol{{\mathrm{Vol}}}
\def\unif{{\mathrm{Uniform}}}

\def\criticmodule{{\fbox{\!CriticCompute\!}~}}

\newcommand{\KL}{D_{\mathrm{KL}}}

\newcommand{\subopt}{\mathrm{SubOpt}}

\DeclareMathOperator*{\argmax}{arg\,max}
\DeclareMathOperator*{\argmin}{arg\,min}
\DeclareMathOperator*{\arginf}{arg\,inf}

\usepackage{tikz}

\usepackage{xcolor}
\usepackage{environ}
\usepackage{xstring}

\newcommand{\keepcomment}{0}%
\newcommand{\keeprevise}{0}%

\definecolor{light-gray}{gray}{0.5}

\newcommand{\thanh}[1]{
    \IfEqCase{\keepcomment}{
        {1}{\textcolor{blue}{\emph{[\textbf{thanh}: #1]}}}
        {0}{}
    }
}
\newcommand{\raman}[1]{
    \IfEqCase{\keepcomment}{
        {1}{\textcolor{red}{\emph{[\textbf{Raman}: #1]}}}
        {0}{}
    }
}

\newcommand{\comment}[1]{
\IfEqCase{\keepcomment}{
        {1}{\textcolor{red}{#1}}
        {0}{}
    }
}

\newcommand{\revise}[1]{
    \IfEqCase{\keeprevise}{
        {1}{\textcolor{orange}{#1 }}
        {0}{#1 }
    }
}

\newcommand{\erase}[1]{
    \IfEqCase{\keeprevise}{
        {1}{\textcolor{red}{\st{#1} }}
        {0}{}
    }
}

\NewEnviron{commentblock}{%
  \color{light-gray}
  \ifnum\keepcomment=1
    \BODY
  \else
  \fi
}

\NewEnviron{draftblock}{%
  \color{red}
  \ifnum\keepcomment=1
    \BODY
  \else
  \fi
}

\newtheorem{defn}{Definition}

\newtheorem{assumption}{Assumption}[section]

\newtheorem{theorem}{Theorem}
\newtheorem{lemma}{Lemma}[section]
\newtheorem{prop}{Proposition}

\makeatletter
\def\paragraph{\@startsection{paragraph}{4}%
  \z@\z@{-\fontdimen2\font}%
  {\normalfont\bfseries}}
\makeatother

\usepackage[utf8]{inputenc} %
\usepackage[T1]{fontenc}    %
\usepackage{hyperref}       %
\usepackage{url}            %
\usepackage{booktabs}       %
\usepackage{amsfonts}       %
\usepackage{nicefrac}       %
\usepackage{microtype}      %
\usepackage{xcolor}         %
\usepackage{cleveref}

\title[Data Diversity, Posterior Sampling, and Beyond]{On Sample-Efficient Offline Reinforcement Learning: Data Diversity, Posterior Sampling, and Beyond}

\author[Thanh Nguyen-Tang \& Raman Arora]{Thanh Nguyen-Tang \& Raman Arora \\ D\lowercase{epartment of} C\lowercase{omputer} S\lowercase{cience}\\
  J\lowercase{ohns} H\lowercase{opkins} U\lowercase{niversity}, B\lowercase{altimore}, MD 21218 }

\begin{document}

\maketitle

\begin{abstract}

  We seek to understand what facilitates sample-efficient learning from historical datasets for sequential decision-making, a problem that is popularly known as offline reinforcement learning (RL). Further, we are interested in algorithms that enjoy sample efficiency while leveraging (value) function approximation. In this paper, we address these fundamental questions by (i) proposing a notion of data diversity that subsumes the previous notions of coverage measures in offline RL and (ii) using this notion to {unify} three distinct classes of offline RL algorithms based on version spaces (VS), regularized optimization (RO), and posterior sampling (PS). We establish that VS-based, RO-based, and PS-based algorithms, under standard assumptions, achieve \emph{comparable} sample efficiency, which recovers the state-of-the-art sub-optimality bounds for finite and linear model classes with the standard assumptions. This result is surprising, given that the prior work suggested an unfavorable sample complexity of the RO-based algorithm compared to the VS-based algorithm, whereas posterior sampling is rarely considered in offline RL due to its explorative nature. Notably, our proposed model-free PS-based algorithm for offline RL is {novel}, with sub-optimality bounds that are {frequentist} (i.e., worst-case) in nature.

\end{abstract}

\section{Introduction}
Learning from previously collected experiences is a vital capability for reinforcement learning (RL) agents, offering a broader scope of applications compared to online RL. This is particularly significant in domains where interacting with the environment poses risks or high costs. However, effectively extracting valuable policies from historical datasets remains a considerable challenge, especially in high-dimensional spaces where the ability to generalize across various scenarios is crucial. In this paper, our objective is to comprehensively examine the efficiency of offline RL in the context of (value) function approximation. We aim to analyze this within the broader framework of general data collection settings.

The problem of learning from historical datasets for sequential decision-making, commonly known as \emph{offline RL} or \emph{batch RL}, originated in the early 2000s~\citep{ernst2005tree,DBLP:conf/colt/AntosSM06,lange2012batch} and has recently regained significant attention~\citep{levine2020offline,uehara2022review}. In offline RL, where direct interaction with environments is not possible, our goal is to learn an effective policy by leveraging pre-collected datasets, typically obtained from different policies known as \emph{behavior policies}. The sample efficiency of an offline RL algorithm is measured by the sub-optimality of the policies it executes compared to a ``good''  comparator policy, which may or may not be an optimal policy. Due to the lack of exploration inherent in offline RL, designing an algorithm with low sub-optimality requires employing the fundamental principle of \emph{pessimistic extrapolation}. This means that the agent extrapolates from the offline data while considering the worst-case scenarios that are consistent with that data. Essentially, the \emph{diversity} present in the offline data determines the agent's ability to construct meaningful extrapolations. Hence, a suitable notion of data diversity plays a crucial role in offline RL.

To address the issue of data diversity, several prior methods have made the assumption that the offline data is \emph{uniformly diverse} -- this implies that the data should cover the entire trajectory space with some probability that is bounded from below~\citep{DBLP:journals/jmlr/MunosS08,chen2019information,nguyen-tang2022on}. %
This assumption is often too strong and not feasible in many practical scenarios. In more recent approaches~\citep{jin2021pessimism,xie2021bellman,uehara2021pessimistic,Chen2022OfflineRL,rashidinejad2022optimal}, the stringent assumption of uniform diversity has been relaxed to only require \emph{partial} diversity in the offline data. Various measures have been proposed to capture this partial diversity, such as single-policy concentrability coefficients~\citep{DBLP:conf/uai/LiuSAB19,rashidinejad2021bridging,yin2021towards}, relative condition numbers~\citep{agarwal2021theory,uehara2021pessimistic}, and Bellman residual ratios~\citep{xie2021bellman}, and \revise{Bellman error transfer coefficient~\citep{song2022hybrid}}. These measures aim to quantify the extent to which the data captures diverse states and behaviors. However, it should be noted that in some practical scenarios, these measures may become excessive or may not hold at all.

In terms of algorithmic approaches, existing sample-efficient offline RL algorithms explicitly construct pessimistic estimates of models or value functions to effectively learn from datasets with partial diversity. This is typically achieved through the construction of lower confidence bounds (LCBs)~\citep{jin2021pessimism,rashidinejad2021bridging} or version spaces (VS) \citep{xie2021bellman,zanette2021provable}. LCB-based algorithms incorporate a bonus term subtracted from the value estimates to enforce pessimism across all state-action pairs and stages. However, it has been observed that LCB-based algorithms tend to impose unnecessarily aggressive pessimism, leading to sub-optimal bounds~\citep{zanette2021provable}. On the other hand, VS-based algorithms search through the space of consistent hypotheses to identify the one with the smallest value in the initial states. These algorithms have demonstrated state-of-the-art bounds~\citep{zanette2021provable,xie2021bellman}. 

In contrast to LCB-based and VS-based algorithms, regularized (minimax) optimization (RO) and posterior sampling (PS) are more amenable to tractable implementations but are relatively new in the offline RL literature. The RO-based algorithm initially introduced by~\cite[Algorithm~1]{xie2021bellman} incorporates pessimism implicitly through a regularization term that promotes pessimism in the initial state. This approach eliminates the need for an intractable search over the version space. However,~\cite{xie2021bellman} demonstrate that the RO-based algorithm exhibits a significantly slower sub-optimality rate than standard VS-based algorithms. Specifically, the RO-based algorithm achieves a sub-optimality rate of $K^{-1/3}$, whereas VS-based algorithms achieve a faster rate of $K^{-1/2}$, where $K$ represents the number of episodes in the offline data.
\begin{table}
    \centering
    \def\arraystretch{2.1}%
    \resizebox{\textwidth}{!}{
    \begin{tabular}{|l|l|c|}
        \hline 
       \textbf{Algorithms}  & \textbf{Sub-optimality Bound} & \textbf{Data}  \\
       \hline
       VS in \citep{xie2021bellman} & $H b \sqrt{ \textcolor{orange}{C_2(\pi)} \cdot \textcolor{red}{\ln (|\gF| \cdot |\Pi^{all}|)} } \cdot \textcolor{blue}{K^{-1/2}}  $ &I\\ 
       \hline 
       RO in \citep{xie2021bellman} & $H b \sqrt{\textcolor{orange}{C_2(\pi)}} \cdot \sqrt[3]{ \textcolor{red}{\ln (|\gF| \cdot |\Pi^{soft}(T)|)}} \cdot \textcolor{blue}{K^{-1/3}} + Hb / \sqrt{T}$ & I\\ 
       \hline 
       MBPS in \citep{uehara2021pessimistic} & $H b \sqrt{\textcolor{orange}{C^{\text{Bayes}}} \cdot \textcolor{red}{\ln |\gM|}} \cdot \textcolor{blue}{K^{-1/2}}$ (Bayesian) & I\\ 
        
       \hline
      \cellcolor{lightgray} VS in \Cref{algorithm: vsc}  & \cellcolor{lightgray} $H b \sqrt{\textcolor{orange}{\gC(\pi; 1/\sqrt{K})} \cdot \textcolor{red}{\ln (|\gF| \cdot |\Pi^{soft}(T)|)} } \cdot \textcolor{blue}{K^{-1/2}} + Hb / \sqrt{T}$ & \cellcolor{lightgray} A \\
      \hline 
        \cellcolor{lightgray} RO in \Cref{algorithm: roc}  & \cellcolor{lightgray} $H b \sqrt{\textcolor{orange}{\gC(\pi; 1/\sqrt{K})} \cdot \textcolor{red}{\ln (|\gF| \cdot |\Pi^{soft}(T)|)} } \cdot \textcolor{blue}{K^{-1/2}} + Hb / \sqrt{T}$ & \cellcolor{lightgray} A \\ 
        \hline
       \cellcolor{lightgray} MFPS in \Cref{algorithm: psc}  & \cellcolor{lightgray} $H b \sqrt{\textcolor{orange}{\gC(\pi; 1/\sqrt{K})} \cdot \textcolor{red}{\ln (|\gF| \cdot |\Pi^{soft}(T)|)} } \cdot \textcolor{blue}{K^{-1/2}} + Hb / \sqrt{T}$ (frequentist) & \cellcolor{lightgray} A \\
       \hline
    \end{tabular}
    }
    \vspace{8pt}
    \caption{Comparison of our bounds with SOTA bounds for offline RL under partial coverage and function approximation, where gray cells mark our contributions. \textbf{Algorithms}: VS = version space, RO = regularized optimization, MBPS = model-based posterior sampling, and MFPS = model-free posterior sampling. \textbf{Sub-optimality bound}: $K$ = \#number of episodes, $\pi$ = an \emph{arbitrary} comparator policy, $H$ = horizon, $b$ = boundedness, $T$ = the number of algorithmic updates, $\ln |\gF|, \ln|\Pi^{soft}(T)|, \ln|\Pi^{all}|, \ln|\gM|$: complexity measures of some value function class $\gF$, ``induced'' policy class $\Pi^{soft}(T)$, the class of all comparator policies $\Pi^{all}$, and model class $\gM$, where typically $\Pi^{soft}(T) \subset \Pi^{all}, \forall T$. \textbf{Data}: I = independent episodes, A = adaptively collected data. Here $\gC(\pi; 1/\sqrt{K})$ and $C_2(\pi)$ are some measures of extrapolation from the offline data to target policy $\pi$.
    }
    \label{tab: bound comparison}
\end{table}

On the other hand, posterior sampling (PS)~\citep{Thompson1933ONTL, russo2014learning}, a popular and successful method in online RL, is rarely explored in the context of offline RL. PS involves sampling from a constructed posterior distribution over the model or value function and acting accordingly. However, PS is less commonly considered in offline RL due to its explorative nature, which stems from the randomness of the posterior distribution. This randomness is well-suited for addressing the exploration challenge in online RL tasks~\citep{zhang2022feel, DBLP:conf/nips/DannMZZ21, zhong2022posterior, agarwal2022non}. The only work that considers PS for offline RL is~\cite{uehara2021pessimistic}, where they maintain a posterior distribution over Markov decision process (MDP) models. However, this model-based PS approach is limited to small-scale problems where computing the optimal policy from an MDP model is computationally feasible. In addition, this work only provides a weak form of guarantees via Bayesian bounds.

In the context of (value) function approximation, achieving sample-efficient offline RL relies on certain conditions that facilitate effective learning. The identification of the minimum condition required for sample efficiency, as well as the algorithms that can exploit such conditions, is an important research question that we aim to address here. We advance our understanding by making the following contributions: (I) {We introduce a new notion of data diversity that subsumes and expands almost all the prior distribution shift measures in offline RL}, and (II) {We show that all VS-based, RO-based and PS-based algorithms are in fact (surprisingly) competitive to each other, i.e., under standard assumptions, they achieve the same sub-optimality bounds (up to constant and log factors)}. We summarize our key results in comparison with related work in \Cref{tab: bound comparison}. Our results further expand the class of sample-efficient offline RL problems (\Cref{fig: relation of coverage notions}) and provide more choices of offline RL algorithms with competitive guarantees and tractable approximations for practitioners to choose from. 

For establishing (II), we need to construct concrete VS-based, RO-based and PS-based algorithms. While the key components of the VS-based and RO-based algorithms appear in the literature~\citep{xie2021bellman}, we propose a novel, a first-of-its-kind,  model-free posterior sampling algorithm for offline RL. The algorithm contains two new ingredients:  a pessimistic prior that encourages pessimistic value functions when being sampled from the posterior distribution and integration of 
posterior sampling with the actor-critic framework that incrementally updates the learned policy.

\paragraph{Overview of Techniques.} 
Our analysis method presents a ``decoupling'' argument tailored for the batch setting, drawing inspiration from recent decoupling arguments in the online RL setting~\citep{foster2021statistical, Jin2021BellmanED, zhang2022feel, DBLP:conf/nips/DannMZZ21, zhong2022posterior, agarwal2022non}. The core idea behind our decoupling argument is to establish a relationship between the Bellman error under any comparator policy $\pi$ and the squared Bellman error under the behavior policy. This relationship is mediated through our novel concept of data diversity, denoted as $\gC(\pi; \eps_c)$, which is defined in detail in ~\Cref{defn: distribution mismatch measure}. This allows to separate the sub-optimality of a learned policy into two main sources of errors: the extrapolation error, which captures the out-of-distribution (OOD) generalization from the behavior policy to a target policy, and the in-distribution error, which focuses on generalization within the same behavior distribution. The OOD error is effectively managed by controlling the data diversity $\gC(\pi; \eps_c)$, while the in-distribution error is carefully addressed by utilizing the algorithmic structures and the martingale counterpart to Bernstein's inequality (i.e., Freedman’s inequality).

In the process of bounding the in-distribution error of our proposed PS algorithm that we built upon the technique of \cite{DBLP:conf/nips/DannMZZ21}, we correct a non-rigorous argument of \cite{DBLP:conf/nips/DannMZZ21} (which we discuss in detail in \Cref{para: non-rigorous argument of Dann}) and develop a new technical argument to handle the statistical dependence induced by the data-dependent target policy in the actor-critic framework. 
Our new argument carefully incorporates the uniform convergence argument into the in-expectation bounds of PS. We give a detailed description of this argument in \Cref{section: proof of key support theorem for ps}. As an immediate application, our technique fixes a non-rigorous mistake involving how to handle the statistical dependence induced by the min player in the self-play posterior sampling algorithm of \cite{xiong2022self}.

\section{Background and problem formulation}
\subsection{Episodic time-inhomogenous Markov decision process}

Let $\gS$ and $\gA$ denote  Lebesgue-measurable state and action spaces (possibly infinite), respectively. Let $\gP(\gS)$ denote the space of all probability distributions over $\gS.$ We consider an episodic time-inhomogeneous Markov decision process $M = (\gS, \gA, P, r, H)$, where $P = \{P_h\}_{h \in [H]} \in \{\gS \times \gA \rightarrow \gP(\gS)\}^H$ are the transition probabilities (where $[H] := \{1, \ldots, H\}$), $r = \{r_h\}_{h \in [H]} \in \{\gS \times \gA \rightarrow \sR\}^H$ is the mean reward functions, and $H \in \sN$ is the length of the horizon for each episode. 
 
For any policy $\pi = \{\pi_h\}_{h \in [H]} \in \{\gS \rightarrow \gP(\gA)\}^H$, the action-value functions and the value functions under policy $\pi$ are defined, respectively, as
    $Q^{\pi}_{h,M}(s,a) = \sE_{\pi}[ \sum_{i=h}^{H} r_i(s_i,a_i) \mid (s_h,a_h) = (s,a)]$, and  $V^{\pi}_{h,M}(s) = \sE_{\pi}[ \sum_{i=h}^{H} r_i(s_i,a_i) | s_h = s]$.
Here $\sE_{\pi}[\cdot]$ denotes the expectation with respect to the randomness of the trajectory $(s_h,a_h, \ldots, s_H, a_H)$, with $a_i \sim \pi_i(\cdot|s_i)$ and $s_{i+1} \sim P_i(\cdot|s_i,a_i)$ for all $i$. 
For any policy $\pi$, we define the visitation density probability functions $d^{\pi}_{M} = \{d^{\pi}_{h,M}\}_{h \in [H]} \in \{\gS \times \gA \rightarrow \sR_{+}\}^H$ as $d^{\pi}_h(s,a) := \frac{d\Pr((s_h,a_h) = (s,a)| \pi, M)}{d \rho(s,a)}$ where $\rho$ is the Lebesgue measure on $\gS \times \gA$ and $\Pr((s_h,a_h) = (s,a)| \pi, M)$ is the probability of policy $\pi$ reaching state-action pair $(s,a)$ at timestep $h$. 

The Bellman operator $\sT_h^{\pi}$ is defined as 
    $[\sT_h^{\pi} Q](s,a) := r_h(s,a) +  \sE_{s' \sim P_h(\cdot|s,a), a' \sim \pi_{h+1}(\cdot|s')} \left[ Q(s',a') \right]$,
for any $Q: \gS \times \gA \rightarrow \sR$. Let $\pi^*$ be an optimal policy, i.e., $Q^{\pi^*}_h(s,a) \geq Q^{\pi}_h(s,a), \forall (s,a,h,\pi) \in \gS \times \gA \times [H] \times \Pi^{all}$, where $\Pi^{all} := \{\gS \rightarrow \gP(\gA)\}^H$ is the set of all possible policies. 
For simplicity, we assume that the initial state $s_1$ is deterministic across all episodes.\footnote{This assumption is merely for the sake of clean presentation which does not affect any results.} We also assume that there is some $b > 0$ such that for any trajectory $(s_1, a_1, r_1, \ldots, s_H, a_H, r_H)$ generated under any policy, $|r_h| \leq b, \forall h$ and $|\sum_{h =1}^H r_h | \leq b$ almost surely.\footnote{Note that we allow the reward samples to be negative.} 

This boundedness assumption is standard and subsumes the boundedness conditions in the previous works, e.g., \cite{zanette2021provable} set $b = 1$ and \cite{jin2021pessimism} use $b=H$ (and further assume that $r_h \in [0,1], \forall h$).\footnote{We can replace the condition $|r_h| \leq b, \forall h$ with $1$-sub-Gaussian condition: $r_h \sim R_h(s_h,a_h)$ wherein $R_h(s_h,a_h)$ is sub-Gaussian with mean $r_h(s_h,a_h)$ -- which replaces $b$ in our main theorems by $b + \ln(KH/\delta)$.} Without loss of generality, we assume that $b \geq 1$. 

\paragraph{Additional notation.} For any $u: \gS \times \gA \rightarrow \sR$ and any $\pi: \gS \rightarrow \gP(\gA)$, we overload the notation $u(s, \pi) := \sE_{a \sim \pi(\cdot|s)}\left[u(s,a)\right]$. For any $f: \gS \times \gA \rightarrow \sR$, denote the supremum norm $\|f\|_{\infty} = \max_{(s,a) \in \gS \times \gA} |f(s,a)|$. We write $\sE[g]^2 := (\sE[g])^2$. For a probability measure $\nu$ on some measurable space $(\Omega, \gB)$, we denote by $\supp(\nu)$ the support of $\nu$, $\supp(\nu) := \{B \in \gB: \nu(B) > 0\}$. We denote $x \lesssim y$ to mean that $x = \gO(y)$.

\subsection{Offline data generation}

Denote the pre-collected dataset by $\gD := \{(s^t_h,a^t_h, r^t_h)\}_{h \in [H]}^{t \in [K]}$, where $s^t_{h+1} \sim P_h(\cdot| s^t_h, a^t_h)$ and $\sE[r^t_h | s^t_h, a^t_h] = r_h(s^t_h, a^t_h)$. We consider the adaptively collected data setting where the offline data is collected by \emph{time-varying} behavior policies $\{\mu^k\}_{k \in [K]}$, concretely, defined as follows. 

\begin{defn}[Adaptively collected data\footnote{It is essentially the ``measurability'' condition in~\cite{zanette2021provable} and ``compliance'' condition in~\cite{jin2021pessimism}.}]
     $\mu^k$ is a function of $\{(s^i_h,a^i_h, r^i_h)\}_{h \in [H]}^{i \in [k-1]},~\forall k \in [K]$. 
    
    \label{defn: adaptive collected data}
\end{defn}
For simplicity, we denote $\mu = \frac{1}{K} \sum_{k=1}^K \mu^k$, $d^{\mu} = \frac{1}{K} \sum_{k=1}^K d^{\mu^k}$, and $\sE_{\mu}[\cdot] = \frac{1}{K} \sum_{k=1}^K \sE_{\mu^k}[\cdot]$. The setting of adaptively collected data covers a common practice where the offline data is collected by using some adaptive experimentation \citep{zhan2021policy}. When $\mu^1 = \cdots = \mu^K$, it recovers the setting of independent episodes in \cite{duan2020minimax}.

\paragraph{Value sub-optimality.}
The goodness of a learned policy $\hat{\pi} = \hat{\pi}(\gD)$ against a comparator policy $\pi$ for the underlying MDP $M$ is measured by the (value) sub-optimality defined as
\begin{align}
    \subopt_{\pi}^M(\hat{\pi}) := V_1^{\pi}(s_1) - V_1^{\hat{\pi}}(s_1). 
    \label{eq: sub-optimality definition}
\end{align}
Whenever the context is clear, we drop $M$ in $Q^{\pi}_M$, $V^{\pi}_M$, $d^{\pi}_{M}$, and $ \subopt_{\pi}^M(\hat{\pi})$. 

\subsection{Policy and function classes}
\label{section: function approximation}
Next, we define the policy space and the action-value function space over which we optimize the value sub-optimality. We consider a (Cartesian product) function class $\gF = \gF_1 \times \cdots \times \gF_H \in \{\gS \times \gA \rightarrow [-b,b]\}^H$.

The function class $\gF$ induces the following (Cartesian product) policy class  $\Pi^{soft} (T) = \Pi_1^{soft} (T) \times \cdots \times \Pi_H^{soft} (T)$, where 
    $\Pi^{soft}_h (T):= \{\pi_h(a|s) \propto \exp ( \eta \sum_{i=1}^t g_i(s,a) ):  t \in [T],  g_i \in \gF_h, \forall i \in [t], \eta \in [0,1] \}$
for any $T \in \sN$. 

The motivation for the induced policy class $\Pi^{soft}(T)$ is from the soft policy iteration (SPI) update where we incrementally update the policy. 

We now discuss a set of assumptions that we impose on the policy and function classes. 

\begin{assumption}[Approximate realizability] There exist $\{\xi_h\}_{h \in [H]}$ where $\xi_h \geq 0$ such that, 
\begin{align*}
   \sup_{T \in \sN, \pi \in \Pi^{soft}(T), (s_h, a_h) \in \supp(d^{\mu}_h)} \inf_{f \in \gF} |f_h(s_h,a_h) - Q^{\pi}_h(s_h,a_h)| \leq \xi_h, ~~\forall h \in [H].
\end{align*}

\label{assumption: realizability}
\end{assumption}
\Cref{assumption: realizability} establishes that $\gF$ can realize $Q^{\pi}$ for any $\pi \in \Pi^{soft}(T)$ up to some error $\xi \in \sR^H$ in the supremum norm over the $\mu$-feasible state-action pairs. It strictly generalizes the assumption   in~\cite{zanette2021provable} which restricts $\xi_h = 0,~\forall h$ (i.e., assume realizability) 
and the assumption in \cite{xie2021bellman} which constrains the approximation error under any feasible state-action distribution.

The realizability in \emph{value} functions alone is known to be insufficient for sample-efficient offline RL \citep{wang2020statistical}; thus, one needs to impose a stronger assumption for polynomial sample complexity of model-free methods.\footnote{A stronger form of realizability is sufficient for polynomial sample complexity, e.g.,  realizability for a density ratio w.r.t. the behavior state-action distribution in dual-primal methods \citep{zhan2022offline,Chen2022OfflineRL,rashidinejad2022optimal} or realizability for the underlying MDP in model-based methods \citep{uehara2021pessimistic}. Instead, we pursue model-free value-based methods.} In this paper, we impose an assumption on the closedness of the Bellman operator. 

\begin{assumption}[General Restricted Bellman Closedness]
 There exists $\nu \in \sR^H$ such that
\begin{align*}
    \sup_{T \in \sN, f_{h+1} \in \gF_{h+1}, \tilde{\pi} \in \Pi^{soft}(T)} \inf_{f'_h \in \gF_h}\| f'_h - \sT_h^{\tilde{\pi}} f_{h+1} \|_{\infty} \leq \nu_h,~~\forall h \in [H]. 
\end{align*}

\label{assumption: approximate Bellman completeness}
\end{assumption}

\Cref{assumption: approximate Bellman completeness} ensures that the value function space $\gF$ and the induced policy class $\Pi^{soft}(T)$ for any $T \in \sN$ are closed under the Bellman operator up to some error $\nu \in \sR^H$ in the supremum norm. This assumption is a direct generalization of the Linear Restricted Bellman Closedness in \cite{zanette2020learning} from a linear function class to a general function class. As remarked by \cite{zanette2021provable}, the Linear Restricted Bellman Closedness is already strictly more general than the low-rank MDPs \citep{yang2019sample,jin2020provably}.

\subsection{Effective sizes of policy and function classes}
\label{section: effective sizes of function class and policy class}
When the function class and the policy class have finite elements, we use their cardinality $|\gF_h|$ and $|\Pi_h^{soft}(T)|$ to measure their sizes \citep{jiang2017contextual,xie2021bellman}. When they have infinite elements, we use log-covering numbers, defined as 
\begin{align*}
    d_{\gF}(\eps) := \max_{h \in [H]} \ln N(\eps; \gF_h, \|\cdot\|_{\infty}), \text{ and } d_{\Pi}(\eps, T) := \max_{h \in [H]}\ln N(\eps; \Pi^{soft}_h(T), \| \cdot \|_{1,\infty}),
\end{align*}
where $\|\pi - \pi'\|_{1, \infty} = \sup_{s \in \gS}\int_{\gA}|\pi(a|s) - \pi'(a|s)| d \rho(a)$ for any $\pi, \pi' \in \{\gS \rightarrow \gP(\gA)\}$ and $N(\eps; \gX, \|\cdot\|)$ denotes the covering number of a pseudometric space $(\gX, \|\cdot\|)$ with metric $\| \cdot \|$ \citep[e.g. Definition~4.1]{zhang_2023}. 

We also define a complexity measure that depends on a prior distribution $p_0$ over $\gF$ that we employ to favor certain regions of the function space. Our notion, presented in \Cref{defn: complexity measure for canonical posterior}, is simply a direct adaptation of a similar notation of \cite{DBLP:conf/nips/DannMZZ21} to the actor-critic setting.
\begin{defn}
For any function $f' \in \gF_{h+1}$ and any policy $\tilde{\pi} \in \Pi^{all}$, we define $\gF^{\tilde{\pi}}_h(\eps; f') := \{f \in \gF_h: \|f - \sT^{\tilde{\pi}}_h f' \|_{\infty} \leq \eps\}$, for any $\eps \geq 0$, and subsequently define 
\begin{align*}
    d_0(\eps) :=\!\!\!\!\!\!\!\!\! \sup_{T \in \sN, f \in \gF, \tilde{\pi} \in \Pi^{soft}(T)} \sum_{h=1}^H \ln \frac{1}{p_{0,h}(\gF_h^{\tilde{\pi}}(\eps; f_{h+1}))}, d'_0(\eps) := \!\!\!\!\!\!\!\!\!\sup_{T \in \sN, \tilde{\pi} \in \Pi^{soft}(T)} \sum_{h=1}^H \ln \frac{1}{p_{0,h}(\gF_h^{\tilde{\pi}}(\eps; Q_{h+1}^{\tilde{\pi}})) }. 
\end{align*}
\label{defn: complexity measure for canonical posterior}
\end{defn}

The quantity $d_0(\eps)$ and $d'_0(\eps)$ measures the concentration of the prior $p_0$ over all functions $f \in \gF$ that are $\eps$-close (element-wise) under $\sT^{\tilde{\pi}}$ and $\eps$-close (element-wise) to $Q^{\tilde{\pi}}_h$, respectively. If a stronger version of \Cref{assumption: realizability} is met, i.e., $Q^{\tilde{\pi}}_h \in \gF_h, \forall \tilde{\pi} \in \Pi^{all}_h, h \in [H]$, we have $d'_0(\eps) \leq d_0(\eps), \forall \eps$. For the finite function class $\gF$ and an uninformative prior $p_{0,h}(f_h) = 1 / |\gF_h|$, under a stronger version of \Cref{assumption: approximate Bellman completeness}, i.e., $\nu_h = 0, \forall h$, we have $d_0(\eps) \leq \sum_{h=1}^H \ln |\gF_h| = \ln |\gF|.$
For a parametric model, where each $f_h = f_h^{\theta}$ is represented by a $d$-dimensional parameter $\theta \in \Omega_h^{\theta} \subset \sR^d$, a prior over $\Omega^{\theta}$ induces a prior over $\gF$. If each $\Omega^{\theta}_h$ is compact, we can generally assume the prior that satisfies $\sup_{\theta} \ln \frac{1}{p_{0,h}(\theta': \| \theta - \theta' \| \leq \eps)} \leq d \ln(c_0/\eps)$ for some constant $c_0$. If $f_h = f_h^{\theta}$ is Lipschitz in $\theta$, we can assume that $\sup_{\theta} \ln \frac{1}{p_{0,h}(\theta': \| \theta - \theta' \| \leq \eps)} \leq c_1 d \ln(c_2/\eps)$ for some constants $c_1, c_2$. Overall, we can assume that 
    $d_0(\eps) \leq c_1 H d \ln(c_2/\eps)$. A similar discussion can be found in \cite{DBLP:conf/nips/DannMZZ21}. 
    
\section{Algorithms}
\label{section: algorithms}
\begin{algorithm}[H]
\begin{algorithmic}[1]
\Require Offline data $\gD$, function class $\gF$, learning rate $\eta > 0$, and iteration number $T$
\State Uniform policy $\pi^1 = \{\pi^1_h\}_{h \in [H]}$
\For{$t = 1, \ldots, T$}
\State  $\ubar{Q}^t = \criticmodule(\pi^t, \gD, \gF, \ldots) $ 
\label{GOPO: critic compute}

\State $\pi^{t+1}_h(a|s) \propto \pi^t_h(a|s) \exp ( \eta \ubar{Q}^t_h(s,a) ), \forall (s,a,h) $
\label{GOPO: soft policy update}
\EndFor
\Ensure A randomized policy $\hat{\pi}$ as a uniform distribution over $\{\pi^t\}_{t \in [T]}$
\end{algorithmic}
\caption{GOPO$(\gD, \gF, \eta, T, \criticmodule)$: Generic Offline Policy Optimization Framework}
\label{algorithm: generic framework}
\end{algorithm}
Next, we present concrete instances of PS-based, RO-based, and VS-based algorithms. The RO-based and VS-based algorithms presented here are slight refinements of their original versions in~\cite {xie2021bellman}. The PS-based algorithm is novel. All three algorithms resemble the actor-critic style update, inspired by \cite{zanette2021provable}. We refer to this generic framework as GOPO (Generic Offline Policy Optimization) presented in \Cref{algorithm: generic framework}. At each round $t$, a critic estimates the value $\ubar{Q}_h^t$ of the actor (i.e., policy $\pi^t$) using the procedure \criticmodule on \Cref{GOPO: critic compute}, and the actor improves the policy using a multiplicative weights update~\citep{v008a006} (\Cref{GOPO: soft policy update}). After $T$ iterations, GOPO  returns a policy $\hat{\pi}$ that is sampled uniformly from the set of the obtained policies $\{\pi^t\}_{t \in [T]}$.

To incorporate the pessimism principle, a critic should generate pessimistic estimates of the value of the actor $\pi^t$ in \Cref{GOPO: critic compute}. This is where the three approaches differ -- each invokes a different method to compute the critic. Here, we provide a detailed description of the critic module for each approach. To aid the presentation, we introduce the total temporal difference (TD) loss $\hat{L}_{\tilde{\pi}}$, defined as 
    $\hat{L}_{\tilde{\pi}}(f_h,f_{h+1}) := \sum_{k=1}^K  l_{\tilde{\pi}}(f_h,f_{h+1}; z^k_h)$,
where $z_h :=$ $ (s_h, a_h, r_h, s_{h+1})$, $z^k_h := (s^k_h, a^k_h, r^k_h, s^k_{h+1})$, and $ l_{\tilde{\pi}}(f_h,f_{h+1}; z_h) := (f_h(s_h,a_h) - r_h - f_{h+1}(s_{h+1}, \tilde{\pi}))^2$. 

\paragraph{Version space-based critic (VSC) (\Cref{algorithm: vsc}).} Given the actor $\pi^t$, at each step $h \in [H]$, VSC directly maintains a local regression constraint using the offline data: $\hat{L}_{\pi^t}(f_h,f_{h+1}) \leq  \inf_{g \in \gF}   \hat{L}_{\pi^t}(g_h,f_{h+1}) + \beta$, where $\beta$ is a confidence parameter and $\hat{L}_{\pi^t}(\cdot, \cdot)$ is serving as a proxy to the squared Bellman residual at step $h$. By taking the function that minimizes the initial value, VSC then finds the most pessimistic value function $\ubar{Q}^t$ from the version space $\gF(\beta; \pi^t) \subseteq \gF$. In general, the constrained optimization in \Cref{VSC: constrained optimization} is computationally intractable. Note that a minimax variant of GOPO+VSC first appeared in \cite{xie2021bellman}, where they directly perform an (intractable) search over the policy space, instead of using the multiplicative weights algorithm (\Cref{GOPO: soft policy update}) of \Cref{algorithm: generic framework}. 
\begin{algorithm}[H]
    \begin{algorithmic}[1]
        \State $$\gF(\beta; \pi^t) := \{f \in \gF:    \hat{L}_{\pi^t}(f_h,f_{h+1}) \leq \inf_{g \in \gF}   \hat{L}_{\pi^t}(g_h,f_{h+1}) + \beta, \forall h \in [H] \}$$
        \State $$\ubar{Q}^t \in \argmin_{f \in \gF(\beta; \pi^t)} f_1(s_1, \pi^t)$$
        \label{VSC: constrained optimization}
        \Ensure $\ubar{Q}^t$
    \caption{VSC($\gD, \gF, \pi^t, \beta$): Version Space-based Critic}
        \label{algorithm: vsc}
    \end{algorithmic}
    \end{algorithm}

\paragraph{Regularized optimization-based critic (ROC) (\Cref{algorithm: roc}).} Instead of solving the global constrained optimization in VSC, ROC solves $\arginf_{f \in \gF} \left\{\lambda f_1(s_1, \pi^t_1) + \gL_{\pi^t}(f) \right\}$, where $\lambda$ is a regularization parameter and $\gL_{\pi^t}(f)$, defined in \Cref{ROC: loss function in ROC} of \Cref{algorithm: roc}. Note that in ROC, pessimism is implicitly encouraged through the regularization term $\lambda f_1(s_1, \pi^t)$. We remark that, unlike VSC, ROC admits tractable approximations that use adversarial training and work competitively in practice \citep{cheng2022adversarially}. Note that a discounted variant of GOPO-ROC first appears in \citep{xie2021bellman} in discounted MDPs.

\begin{algorithm}[H]
    \begin{algorithmic}[1]
        \State  $$\gL_{\pi^t}(f) :=  \sum_{h=1}^H   \hat{L}_{\pi^t}(f_h,f_{h+1})  - \inf_{g \in \gF} \sum_{h=1}^H   \hat{L}_{\pi^t}(g_h,f_{h+1})$$
        \label{ROC: loss function in ROC}
        \State   $$\ubar{Q}^t \leftarrow \arginf_{f \in \gF} \left\{\lambda f_1(s_1,\pi^t) + \gL_{\pi^t}(f) \right\} $$
        \label{ROC: regularized optimization}
        \Ensure $\ubar{Q}^t$
    \caption{ROC($\gD, \gF, \pi^t, \lambda$): Regularized Optimization-based Critic}
        \label{algorithm: roc}
    \end{algorithmic}
\end{algorithm}
\paragraph{Posterior sampling-based critic (PSC) in \Cref{algorithm: psc}.} Instead of solving a regularized minimax optimization, PSC samples the value function $\ubar{Q}^t_h$ from the data posterior $\hat{\pi}(f| \gD, \pi^t) \propto \tilde{p}_0(f) \cdot p(\gD | f, \pi^t)$, where $\tilde{p}_0(f)$ is the prior over $\gF$ and $p(\gD | f, \pi^t)$ is the likelihood function of the offline data $\gD$. To formulate the likelihood function $p(\gD | f, \pi^t)$, we make use of the squared TD error $\hat{L}_{\pi^t}(\cdot, \cdot)$ and normalization method in \citep{DBLP:conf/nips/DannMZZ21} to construct an unbiased proxy of the squared Bellman errors. In particular, $p(\gD|f, \pi^t) = \prod_{h \in [H]} \frac{\exp( -\gamma \hat{L}_{\pi^t}(f_h, f_{h+1}))}{\sE_{f'_h \sim p_{0,h}} \exp( -\gamma \hat{L}_{\pi^t}(f'_h, f_{h+1}))}$, where $\gamma$ is a learning rate and $p_0$ is an (unregularized) prior over $\gF$. A value function sampled from the posterior with this likelihood function is encouraged to have small squared TD errors. The key ingredient in our algorithmic design is the ``pessimistic'' prior $\tilde{p}_0(f) = \exp ( - \lambda f_1(s_1, \pi_1) ) p_0(f)$ where we add a new regularization term $\exp ( - \lambda f_1(s_1, \pi_1) )$, with $\lambda$ being a regularization parameter -- which is inspired by the optimistic prior in the online setting \citep{zhang2022feel,DBLP:conf/nips/DannMZZ21}. This pessimistic prior encourages the value function sampled from the posterior to have a small value in the initial state, implicitly enforcing pessimism. We remark that PSC requires a sampling oracle and expectation oracle (to compute the normalization term in the posterior distribution), which could be amenable to tractable approximations, including replacing expectation oracle with a sampling oracle \citep{agarwal2022non} while the sampling oracle can be implemented via first-order sampling methods \citep{welling2011bayesian} or ensemble methods \citep{osband2016deep}.

\begin{algorithm}[H]
    \begin{algorithmic}[1]
        \State  
    $$\ubar{Q}^t \sim \hat{p}(f| \gD, \pi^t) \propto\exp \left( - \lambda f_1(s_1,\pi^t) \right) p_0(f) \prod_{h \in [H]} \frac{\exp\left( -\gamma \hat{L}_{\pi^t}(f_h, f_{h+1})\right)}{\sE_{f'_h \sim p_{0,h}} \exp\left( -\gamma \hat{L}_{\pi^t}(f'_h, f_{h+1})\right)}$$
    \label{line:psc posterior distribution}
    \Ensure $\ubar{Q}^t$
    \caption{PSC$(\gD, \gF, \pi^t, \lambda, \gamma, p_0)$: Posterior Sampling-based Critic}
    \label{algorithm: psc}
    \end{algorithmic}
\end{algorithm}

\section{Main results}

\label{section: main results}

In this section, we shall present the upper bounds of the sub-optimality of the policies executed by GOPO-VSC, GOPO-ROC, and GOPO-PSC. Our upper bounds are expressed in terms of a new notion of data diversity. 

\subsection{Data diversity} 

We now introduce the key notion of data diversity for offline RL. Since the offline learner does not have direct access to the trajectory of a comparator policy $\pi \in \Pi^{all}$, they can only observe partial information about the goodness of $\pi$ channeled through the ``transferability'' with the behavior policy $\mu$. The transferability from $\mu$ to $\pi$ depends on how \emph{diverse} the offline data induced by $\mu$ can be in supporting the extrapolation to $\pi$. Many prior works require uniform diversity where $\mu$ covers all feasible scenarios of all comparator policies $\pi$. The data diversity can be essentially captured by how well the Bellman error under the state-action distribution induced by $\mu$ can predict the counterpart quantity under the state-action distribution induced by $\pi$.

Our notion of data diversity, which is inspired by the notion of task diversity in transfer learning literature \citep{tripuraneni2020theory,watkins2023optimistic}, essentially encodes the ratio of some proxies of expected Bellman errors induced by $\mu$ and $\pi$, and is defined as follows. 

\begin{defn}
For any comparator policy $\pi \in \Pi^{all}$, we measure the data diversity of the behavior policy $\mu$ with respect to a target policy $\pi$ by 
\begin{align}
    \gC(\pi; \eps) := \max_{h \in [H]}  \chi_{(\gF_h - \gF_h)}(\eps; d^{\pi}_h,d^{\mu}_h) , \forall \eps \geq 0,
    \label{eq: distribution mismatch measure}
\end{align}
where $\gF_h - \gF_h$ is the Minkowski difference between the function class $\gF_h$ and itself, i.e., $\gF_h - \gF_h := \{f_h - f'_h: f_h, f'_h \in \gF\}$, and $\chi_{\gQ}(\eps; q,p)$ is the discrepancy between distributions $q$ and $p$ under the witness of function class $\gQ$ defined as
\begin{align*}
    \chi_{\gQ}(\eps; q,p) = \inf \left\{C \geq 0: (\sE_q[g])^2 \leq C \cdot \sE_p [g^2] + \eps, \forall g \in \gQ \right\}
\end{align*}
with $\gQ$ being a function class and $p$ and $q$ being two distributions over the same domain.  
\label{defn: distribution mismatch measure}
\end{defn}
Up to a small additive error $\eps$, a finite $\gC(\pi; \eps)$ ensures that a proxy of the Bellman error under the $\pi$-induced state-action distribution is controlled by that under the $\mu$-induced state-action distribution. Despite the abstraction in the definition of this data diversity, it is \emph{always} upper bounded by the single-policy concentrability coefficient \citep{DBLP:conf/uai/LiuSAB19,rashidinejad2021bridging} and the relative condition number \citep{agarwal2021theory,uehara2022representation,uehara2021pessimistic} that are both commonly used in many prior offline RL works. We further discuss our data diversity measure in more detail in \Cref{section: offline learning guarantees}.

\subsection{Offline learning guarantees}
\label{section: offline learning guarantees}
We now utilize data diversity to give learning guarantees of the considered algorithms for extrapolation to an arbitrary comparator policy $\pi \in \Pi^{all}$. To aid the representation, in all of the following theorems we are about to present, we shall 
set $\eta = \sqrt{\frac{\ln \vol(\gA)}{4(e-2) b^2 T}}$ in \Cref{algorithm: generic framework}, where $\vol(\gA)$ is the volume of the action set $\gA$ (e.g., $\vol(\gA) = |\gA|$ for finite $\gA$), and define, for simplicity, the misspecification errors $\zeta_{msp} :=  K \sum_{h=1}^H \left(  \nu_h^2 + b \nu_h \right)$, $\tilde{\zeta}_{msp} := \zeta_{msp} + b K \sum_{h=1}^H \xi_h$, $\bar{\nu} := \sum_{h=1}^H \nu_h$, the optimization error $\zeta_{opt} := H b \sqrt{ T^{-1} \ln \vol(\gA)}$, and the complexity measures $\tilde{d}_{opt}(\eps, T) := \max\{d_{\gF}(\eps), d_{\Pi}(\eps, T)\}$, and $\tilde{d}_{ps}(\eps, T) := \max\{ d_{\gF}(\eps), d_{\Pi}(\eps, T), \frac{d_0(\eps)}{\gamma H b^2}, \frac{d'_0(\eps)}{\gamma H b^2} \}$.

\begin{theorem}[Guarantees for GOPO-VSC]
Let $\hat{\pi}^{vs}$ be the output of \Cref{algorithm: generic framework} invoked with \criticmodule being VSC($\gD, \gF, \pi^t, \beta$) (\Cref{algorithm: vsc}) with $\beta = \gO(b^2 K \eps + b K \max_{h \in [H]} \xi_h + H b^2 \max\{\tilde{d}_{opt}(\eps, T), \ln (H/\delta)\})$. 
Fix any $\delta \in (0,1]$. Under \Cref{assumption: realizability}-\ref{assumption: approximate Bellman completeness}, with probability at least $1 - 2 \delta$ (over the randomness of the offline data), for any $\eps, \eps_c, \lambda > 0$, and any $\pi \in \Pi^{all}$, we have
\begin{align*}
    \sE \left[\subopt_{\pi}(\hat{\pi}^{vs}) | \gD \right] &\lesssim \frac{  H b^2 \cdot \max\{\tilde{d}_{opt}(\eps, T), \ln (H/\delta)\} + b^2 K H \eps + \tilde{\zeta}_{msp} }{\lambda} + \frac{\lambda H \cdot \gC(\pi; \eps_c)}{2 K} \\
    &+ H \eps_c +  \xi_1 + \bar{\nu} + \zeta_{opt}.
\end{align*}
\label{theorem: gopo-vsc}
\end{theorem}

\begin{theorem}[Guarantees for GOPO-ROC]
Let $\hat{\pi}^{ro}$ be the output of \Cref{algorithm: generic framework} invoked with \criticmodule being ROC($\gD, \gF, \pi^t, \lambda$) (\Cref{algorithm: roc}). 
Fix any $\delta \in (0,1]$. Under \Cref{assumption: realizability}-\ref{assumption: approximate Bellman completeness}, with probability at least $1 - 2 \delta$ (over the randomness of the offline data), for any $\eps, \eps_c, \lambda > 0$, and any $\pi \in \Pi^{all}$, we have
\begin{align*}
    \sE \left[\subopt_{\pi}(\hat{\pi}^{ro})| \gD \right] &\lesssim \frac{H b^2 \cdot \max\{\tilde{d}_{opt}(\eps, T), \ln \frac{H}{\delta}\} + b^2 K H \eps + \tilde{\zeta}_{msp}}{\lambda} + \frac{\lambda H \cdot \gC(\pi; \eps_c)}{2 K} \\
    &+ H \eps_c +  \xi_1 + \bar{\nu} + \zeta_{opt}. 
\end{align*}
\label{theorem: gopo-roc}
\end{theorem}

\begin{theorem}[Guarantees for GOPO-PSC]
Let $\hat{\pi}^{ps}$ be the output of \Cref{algorithm: generic framework} invoked with \criticmodule being PSC$(\gD, \gF, \pi^t, \lambda, \gamma, p_0)$ (\Cref{algorithm: psc}). Under \Cref{assumption: approximate Bellman completeness}, for any $\gamma \in [0, \frac{1}{144 (e-2) b^2}]$, and $\eps,\eps_c, \delta, \lambda > 0$, and any $\pi \in \Pi^{all}$, we have
\begin{align*}
    \sE \left[ \subopt_{\pi}(\hat{\pi}^{ps}) \right] &\lesssim \frac{\gamma H b^2 \cdot \max\{ \tilde{d}_{ps}(\eps, T), \ln \frac{\ln K b^2}{\delta} \} + \gamma b^2 K H \cdot \max\{ \eps, \delta \} + \gamma \zeta_{msp}}{\lambda} \\
    &+ \frac{\lambda H \cdot \gC(\pi; \eps_c)}{K \gamma} + H \eps_c + \eps + \bar{\nu} + \zeta_{opt}.
\end{align*}
\label{theorem: gopo-psc}
\end{theorem}
\vspace{-30pt}
Our results provide a family of upper bounds on the sub-optimality of each of $\{\hat{\pi}^{vs}, \hat{\pi}^{ro}, \hat{\pi}^{ps}\}$, indexed by our choices of the comparator $\pi$ with the data diversity $\gC(\pi; \eps_c)$, additive (extrapolation) error $\eps_c$, the discretization level $\eps$ in log-covering numbers, the ``failure'' probability $\delta$, and other algorithm-dependent parameters ($\lambda$ for $\hat{\pi}^{ro}$ and ($\lambda$, $\gamma$) for $\hat{\pi}^{ps}$). Note that the optimization error $\zeta_{opt}$ captures the error rate of the actor and can be made arbitrarily small with large iteration number $T$ whereas $\zeta_{msp}$, $\tilde{\zeta}_{msp}$, $\bar{\nu}$, and $\xi_1$ are simply misspecification errors aggregated over all stages. Also, note that our bound does not scale with the complexity of the comparator policy class $\Pi^{all}$.

We next highlight the key characteristics of our main results in comparison with existing work. 

\paragraph{(I) Tight characterization of data diversity.} Our bounds in all the above theorems are expressed in terms of $\gC(\pi; \eps_c)$.
Several remarks are in order. First, $\gC(\pi; \eps_c)$ is a \emph{non-increasing} function of $\eps_c$; thus
$\gC(\pi; \eps_c)$ is always smaller or at least equal to $\gC(\pi; 0)$. In fact, it is possible that $\gC(\pi; 0) = \infty$  yet $\gC(\pi; \eps_c) < \infty$ for some $\eps > 0$. For instance, if there exists $g \in \gQ$ such that $g(x) = 0, \forall x \in \supp(p)$ and $\{x: g(x) \neq 0\}$ has a positive measure under $q$, then $\chi_{\gQ}(0;q,p) = \infty$ while $\chi_{\gQ}( \sup_{g \in \gQ}\sE_q[g]^2;q,p) = 0$. Second, $\gC(\pi; 0)$ is always bounded from above by (often substantially smaller than) the \emph{single-policy concentrability coefficient} between the $\pi$-induced and $\mu$-induced state-action distribution \citep{DBLP:conf/uai/LiuSAB19,rashidinejad2021bridging}, which been used extensively in recent offline RL works \citep{yin2021towards,nguyen2021offline,nguyen2022instance,jin2022policy,zhan2022offline,nguyentang2023,zhao2023local}. This is essentially because $d^{\pi}$ can cover the region that is not covered by $d^{\mu}$ but still the integration of functions in $\gF_h - \gF_h$ over two distributions are close to each other. Third, $\gC(\pi; 0)$ is always upper bounded by the \emph{relative condition numbers} used in \citep{agarwal2021theory,uehara2022representation,uehara2021pessimistic}. Our data diversity at $\eps=0$ is similar to the notion of distribution mismatch in \citep{duan2020minimax,ji2022sample} \revise{and the Bellman error transfer coefficient of ~\citet{song2022hybrid}}, though our notion is motivated by transfer learning and discovered naturally from our decoupling argument. 
Our data diversity measure at $\eps=0$ is smaller than the Bellman residual ratio measure used in \cite{xie2021bellman} (follows using Jensen's inequality). Finally, the concurrent work of \cite{di2023pessimistic} proposed a notion of $D^2$-divergence to capture the data disparity of a data point to the offline data. Our data diversity is in general less restricted as we only need to ensure the diversity between two data distributions (of the target policy and the behavior policy), not necessarily between each of their individual data points.

In summary, $\gC(\pi; \eps_c)$, to the best of our knowledge, provides the tightest characterization of distribution mismatch compared to the prior data coverage notions. We sketch the relationships of the discussed notions in \Cref{fig: relation of coverage notions}, where with our data diversity notion, we show that the scenarios for the offline data in which offline RL is learnable are enlarged compared to the picture depicted by the prior data coverage notions.

\begin{figure}
  \begin{center}
    \includegraphics[width=0.36\textwidth]{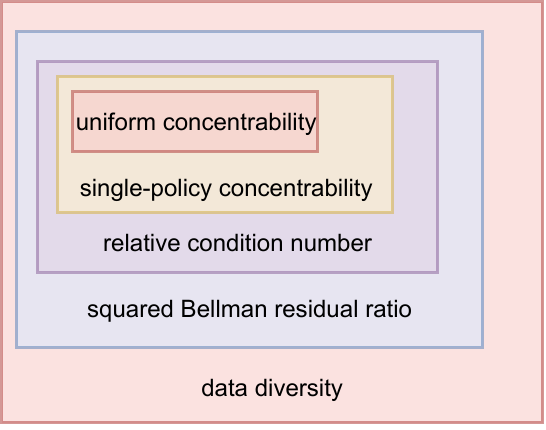}
  \end{center}
  \caption{The relations of sample-efficient offline RL classes under different data coverage measures. Given the same MDP and a target policy (e.g., an optimal policy of the MDP), each data coverage measure induces a corresponding set of behavior policies (represented by the rectangle labelled by the data coverage measure) from which the target policy is offline-learnable.} 
  \label{fig: relation of coverage notions}
\end{figure}

\paragraph{(II) Competing with all comparator policies simultaneously.} Similar to some recent results in offline RL, our offline RL algorithms compete with all comparator policies that are supported by offline data in some sense. In particular, the choice of the comparator $\pi$ provides the flexibility to \emph{automatically} compete with the best policy within a certain diversity level of our choice. For instance, if we want to limit the level $\gC(\pi; \eps_c) \leq C$ for some arbitrary $C > 0$, our bound automatically competes with $\pi = \argmax_{\pi \in \Pi^{all}} \{V_1^{\pi}(s): \gC(\pi; \eps_c) \leq C\}$. 
This is immensely meaningful since the offline data might not support extrapolation to an optimal policy in practice. 
\paragraph{(III) State-of-the-art bounds for \emph{standard} assumptions.} 
We compare our bounds with other recent guarantees of similar assumptions.\footnote{Recent primal-dual methods achieve favorable guarantees for offline RL. However, these guarantees are not directly comparable to the guarantees of our value-based methods due to a different set of assumptions. Nonetheless, we make a detailed discussion in \Cref{subection: compare with primal-dual methods}.} To ease comparison, we assume for simplicity, that there is no misspecification, i.e., $\nu_h = \xi_h = 0, \forall h \in [H]$, and $T \geq K \ln \vol(\gA)$, and we minimize the bounds in \Cref{theorem: gopo-roc} and \Cref{theorem: gopo-psc} with respect to $\lambda$. The three theorems can then be simplified into a unified result presented in  \Cref{prop: simplified results in a unified form}.

\begin{prop}[A unified guarantee for VS, RO and PS] Under \Cref{assumption: realizability}-\ref{assumption: approximate Bellman completeness} with no misspecification, i.e., $\nu_h = \xi_h = 0, \forall h \in [H]$,
$\forall \hat{\pi} \in \{\hat{\pi}^{vs}, \hat{\pi}^{ro}, \hat{\pi}^{ps}\}$, $\sE[\subopt_{\pi}(\hat{\pi})] = \tilde{\gO} (Hb \sqrt{\tilde{d}(1/K, T) \cdot \gC(\pi; 1/\sqrt{K})/K} + \xi_{opt})$, where $\tilde{d}(1/K, T) = \tilde{d}_{opt}(1/K, T)$ if $\hat{\pi} \in \{\hat{\pi}^{vs}, \hat{\pi}^{ro}\}$ and $\tilde{d}(1/K, T) = \tilde{d}_{ps}(1/K, T)$ if $\hat{\pi} = \hat{\pi}^{ps}$. In addition,
\begin{itemize}
    \item If $\gF_h$ and $\Pi^{soft}_h(T)$ have finite elements for all $h \in [H]$, we have $\tilde{d}(1/K,T) = \gO(\max_{h \in [H]} \max \{ \ln |\gF_h|, \ln |\Pi_h^{soft}(T)| \})$; 
    \item If $\gF_h = \{(s,a) \mapsto \langle \phi_h(s,a), w \rangle : \| w\|_2 \leq b\}$  is a linear model, where $\phi_h: \gS \times \gA \rightarrow \sR^d$ is a known feature map and w.l.o.g. $\max_{h} \|\phi_h\|_{\infty} \leq 1$, we have $\tilde{d}(1/K, T) = \gO(d \log(1 + K T b)), \forall T$.
\end{itemize}
\label{prop: simplified results in a unified form}
\end{prop}
\Cref{prop: simplified results in a unified form} essentially asserts that VS-based, RO-based, and PS-based algorithms obtain comparable guarantees for offline RL in the realizable case. We now compare our results to related work in various instantiating of function classes. 

\paragraph{Compared with \cite{xie2021bellman} when the function class is finite.} In this case, the analysis of the VS-based algorithms and RO-based algorithms of \cite{xie2021bellman} give the bounds that in our setting can be translated\footnote{\cite{xie2021bellman} consider discounted MDP and a \emph{restricted} policy class for the comparator class.} into: {$Hb \sqrt{ \max_{h} \ln(|\gF_h| |\Pi_h^{all}|) \cdot C_2(\pi) / K}$} and {$Hb \sqrt{ C_2(\pi) } \sqrt[3]{ \max_{h} \ln(|\gF_h| |\Pi_h^{soft}(T)|)/ K} + H b / \sqrt{T}$}, respectively, where {$C_2(\pi) := \max_{h \in [H], \tilde{\pi} \in \Pi^{all}, f \in \gF} \frac{\|f_h  - \sT^{\tilde{\pi}}_{h} f_{h+1}\|^2_{2,\pi}}{\|f_h  - \sT^{\tilde{\pi}}_{h} f_{h+1}\|^2_{2,\mu}}$}. Instead, our bounds for both the VS-based and RO-based algorithms are {$Hb \sqrt{ \max_{h} \ln(|\gF_h| |\Pi_h^{soft}(T)|) \cdot \gC(\pi; 1/\sqrt{K}) / K} + H b / \sqrt{T}$}. 
We improve upon the results of \cite{xie2021bellman} on several fronts. First, our diversity measures $\gC(\pi; 1/K)$ is always smaller than their measure $C_2(\pi)$, since $\gC(\pi; 1/\sqrt{K}) \leq \gC(\pi; 0) \leq C_2(\pi)$.  Second, for the VS-based algorithm, $\Pi^{soft}(T) \subset \Pi^{all}, \forall T$, our bound is always tighter. In fact, $|\Pi^{all}|$ is arbitrarily large that bounds depending on this quantity is vacuous. Third, for the RO-based algorithm, the rates in terms of $K$ in the bound of \cite{xie2021bellman} are slower than that in our bound. Specifically, if $\Pi_h^{soft}(T) = \tilde{\mathcal{O}}_T(1)$, then these rates are $K^{-1/3}$ vs $K^{-1/2}$ (with an optimal choice of $T = K$ for both bounds). If  we consider the worst case that $\Pi_h^{soft}(T) = \mathcal{O}(T \log |\gF_h|)$, then these rates are $K^{-1/5}$ vs $K^{-1/4}$ (with an optimal choice of $T = K^{2/5}$ and $T = \sqrt{K}$ in the respective bounds). 
Finally, our results hold under the general adaptively collected data rather than their independent episode setting. We summarize the bounds in the finite function class cases in \Cref{tab: bound comparison}, and give comparisons for the linear model cases in \Cref{tab: sub-optimality bounds in linear function classes}.

\paragraph{Compared with LCB-based algorithms.} When $\gF_h$ is a $d$-dimensional linear model with feature maps $\{\phi_h\}_{h \in [H]}$, our bounds reduce into {$H b \sqrt{d \cdot K^{-1} \cdot \gC(\pi; 1/\sqrt{K})}$} (\Cref{prop: simplified results in a unified form}), which matches the order of (and potentially tighter than) the bound in \cite{zanette2021provable}, since $\gC(\pi; 1/\sqrt{K})$ is always smaller (or at least equal to) than the relative condition number {$\max_{h}\sup_{x \in \sR^d} \frac{x^T \sE_{\pi}[\phi_h(s_h,a_h) \phi_h(s_h,a_h)^T] x}{x^T \sE_{\mu}[\phi_h(s_h,a_h) \phi_h(s_h,a_h)^T] x}$}. Compared with the bound of LCB-based algorithms in \cite{jin2021pessimism}, we improve a factor $\sqrt{d}$ and holds under the more general \Cref{assumption: approximate Bellman completeness} which includes low-rank MDPs. In a more refined analysis \citep{xiong2022nearly}, the LCB-based algorithm obtains the same dependence on $d$ for low-rank MDPs as our guarantees. However, this improvement relies on a uniform coverage assumption, i.e., $\min_{h \in [H]} \lambda_{\min} \left(\sE_{(s_h, a_h) \sim d^{\mu}_h} \left[\phi_h(s_h, a_h) \phi_h(s_h, a_h)^T \right] \right) > 0$, which we do not require. \cite{di2023pessimistic} generalize the results of \cite{xiong2022nearly} from linear MDPs to MDPs with general function approximation. However, they still rely on a uniform coverage assumption. 
Finally note that, for VS-based and RO-based algorithms, we provide high-probability bounds for a smoothing version of $\hat{\pi}$ over the randomization of the algorithms, not for $\hat{\pi}$ itself.
\paragraph{Compared with model-based PS} \cite{uehara2021pessimistic} consider model-based PS for offline RL and derive the \emph{Bayesian} sub-optimality bound of $H^2 \sqrt{C^{\text{Bayes}} \cdot \ln|\gM|/ K}$ where $C^{\text{Bayes}}$ is the Bayesian version of a relative condition number and $\gM$ is a finite model class. Two key distinctions are that our method in \Cref{algorithm: psc} is model-free, and our achieved bound is in the frequentist (i.e., worst-case) nature, which is a stronger result than the Bayesian bound of the same order. 

\subsection{Extended discussion}
\paragraph{Comparison with primal-dual methods for offline RL}
\label{subection: compare with primal-dual methods}
As opposed to the value-based methods we considered in our paper, an important alternative approach to offline RL is the primal-dual methods \citep{zhan2022offline,Chen2022OfflineRL,rashidinejad2022optimal,gabbianelli2023offline,ozdaglar2023revisiting}. However, the guarantees of primal-dual methods use a different set of assumptions than the value-based methods we consider (the former assumes realizability for the ratio between the state-action occupancy density of the target policy and the state-action occupancy density of the behavior policy, except for \cite{gabbianelli2023offline} where this realizability assumption is implicitly encoded under a stronger assumption of linear MDP). This makes the results presented in our paper and the results in the primal-dual methods not directly comparable. 

Since the work of \cite{gabbianelli2023offline} considers linear MDPs, it is more comparable (than the other primal-dual methods we mentioned) to the instantiating of our results to the linear function class. \cite{gabbianelli2023offline} consider primal-dual methods for offline RL in both infinite-horizon discounted MDP and average-reward MDP. Our analysis framework for the regularized optimization method in the episodic MDP should work for the infinite-horizon discounted MDP as well, where the regularized optimization achieves the optimal sample complexity of $\mathcal{O}(\epsilon^{-2})$ while the sample complexity in \cite{gabbianelli2023offline} in this setting is $\mathcal{O}(\epsilon^{-4})$.  However, \cite{gabbianelli2023offline} offers a better computational complexity ($\mathcal{O}(K)$ vs $\mathcal{O}(K^{7/5})$) and also works in the average-reward MDP setting which is beyond the episodic MDP setting considered in our work; though our bounds hold for general function approximation that is beyond the strong assumption of linear MDPs.

The concurrent work of \cite{zhu2023importance} combined the actor-critic framework with marginalized importance sampling (MIS) for an RO-based algorithm, which also improves the sub-optimal rate of order $1/K^{1/3}$ by \cite{xie2021bellman,cheng2022adversarially} to the optimal rate of order $1/\sqrt{K}$. Instead, we obtain the optimal rate of order $1/\sqrt{K}$ with a refined analysis for a standard RO-based algorithm. That is, unlike \cite{zhu2023importance}, we do not use MIS; consequently, we do not require the realizability assumption for the ratio between the state-action occupancy density of the target policy and that of the behavior policy. 

\paragraph{Linear function classes}
We consider the linear model cases, where there are known feature maps $\phi_h: \gX \times \gA \rightarrow \sR^d$ and w.l.o.g. $\max_{h \in [H]}\|\phi(\cdot,\cdot)\|_{\infty} \leq 1$, such that $\gF_h = \{\phi_h(\cdot, \cdot)^T w: w \in \sR^d, \| w\|_2^2 \leq b\}$. Recall that, in this case, e.g., it follows from \cite[Lemma~6]{zanette2021provable}, that we have 
\begin{align*}
    \log N(\eps; \gF_h, \|\cdot\|_{\infty}) &\leq d \log(1 + \frac{2}{\eps}), \\ 
    \log N(\eps; \Pi_h^{soft}(T), \|\cdot\|_{\infty}) &\leq d \log(1 + \frac{16 b T}{\eps}). 
\end{align*}
Thus, our bounds from \Cref{prop: simplified results in a unified form} can simplified as
\begin{align*}
    &\tilde{\gO} (Hb \sqrt{d \log(1 + 16 b K T) \cdot \gC(\pi; 1/\sqrt{K})/K} +  H b \sqrt{ T^{-1} \ln \vol(\gA)}) \\
    &= \tilde{\gO} \left( H b \sqrt{\frac{ \max\{d \gC(\pi; 1/\sqrt{K}), \ln \vol(\gA)\}}{K}}   \right)
\end{align*}
where we choose $T = K$. To simplify the comparison, we assume that $d \gC(\pi; 1/\sqrt{K}) \geq \ln \vol(\gA)$. 
Let us now compute various notions of data coverage in this linear model case. 
We first need to define the following quantities (various forms of covariance matrices). 
\begin{align*}
     \Sigma_h &:= \lambda I + \sum_{k=1}^K \phi_h(s^k_h, a^k_h) \phi_h(s^k_h, a^k_h)^T, \\ 
     \Lambda_h &:= \lambda I + \sum_{k=1}^K \phi_h(s^k_h, a^k_h) \phi_h(s^k_h, a^k_h)^T / [\sV_h V_{h+1}^{\pi}](s_h^k, a_h^k), \\ 
     \bar{\phi}_h^{\pi} &:= \sE_{\pi} [\phi_h(s_h, a_h)] ,\\ 
     \bar{\Sigma}_h &:= \sE_{\mu} \left[ \phi(s_h, a_h) \phi(s_h, a_h)^T  \right].
\end{align*}
We define the following distribution mismatch quantities, which were used in the literature. 
\begin{align*}
    C_{pevi}(\pi) &:= \max_{h \in [H]} \left( \sE_{\pi} \left[ \| \phi_h(s_h,a_h) \|_{\Sigma_h^{-1}} \right] \right)^2, \\ 
     C_{pevi-adv}(\pi) &:= \max_{h \in [H]} \left( \sE_{\pi} \left[ \| \phi_h(s_h,a_h) \|_{\Lambda_h^{-1}} \right] \right)^2, \\ 
    C_{pacle}(\pi) &:= \max_{h \in [H]}\| \bar{\phi}_h^{\pi} \|^2_{\Sigma_h^{-1}}, \\ 
    C_{bcp}(\pi) &:= \max_{h \in [H]} \left( \sE_{\pi} \left[ \| \phi_h(s_h, a_h) \|_{\bar{\Sigma}_h} \right] \right)^2.
\end{align*}
The sub-optimality bounds of various methods are summarized in \Cref{tab: sub-optimality bounds in linear function classes}. For comparing our data diversity measure with different notions of distribution mismatch, we have 
\begin{align*}
    C_{pevi}(\pi) \geq C_{pacle}(\pi) \approx \gC(\pi; 0) / K \leq C_{bcp}(\pi) / K.
\end{align*}
where the ``$\approx$'' denotes that the involved terms scale in the same order and can be implied by Fredman's matrix inequality (see \cite[Lemma~B.5]{duan2020minimax}) (under additional conditions). Note that $\gC(\pi; 1/\sqrt{K}) \leq \gC(\pi; 0)$, thus our data diversity is the tightest quantity among all that are considered.

Note that the data coverage measure in \cite{xiong2022nearly}, roughly speaking, can be bounded as follows:
\begin{align*}
    C_{pevi-adv}(\pi) \leq b^2 C_{pevi}(\pi),
\end{align*}
where we use the inequality $[\sV_h V_{h+1}^{\pi}](s_h^k, a_h^k) \leq b^2$. Thus the bound of \cite{xiong2022nearly} in general has a tighter dependence on $b$ (which implicitly depends on $H$) than all the bounds of all other works considered in \Cref{tab: sub-optimality bounds in linear function classes}, due to that \cite{xiong2022nearly} incorporated the variance information into the estimation via the variance-weighted value iteration algorithm. However, obtaining this improved bound in \cite{xiong2022nearly} relies on a uniform coverage assumption which we do not require.

\begin{table}[]
    \centering
    \def\arraystretch{1.7}%
    \resizebox{0.7\textwidth}{!}{
    \begin{tabular}{|c|c|}
    \hline
        \textbf{Algorithm} & \textbf{Sub-optimality bound} \\
        \hline
       PEVI \citep{jin2021pessimism}  & $Hb \sqrt{ C_{pevi}(\pi)}  \cdot d$ \\
       \hline
       PEVI-ADV+ \citep{xiong2022nearly} & $H \sqrt{ C_{pevi-avi}(\pi) \cdot d}$ \\ \hline
       PACLE \citep{zanette2021provable}  & $H b \sqrt{C_{pacle}(\pi) \cdot d} $ \\ 
       \hline
       VC \cite[Section~3]{xie2021bellman} & $ H b \sqrt{ C_{bcp}(\pi) \cdot d / K}$ \\ 
       \hline 
       RO \citep[Section~4]{xie2021bellman} & $Hb \sqrt{C_{bcp}(\pi)} \sqrt[3]{d/K}$ \\ 
       \hline
       \cellcolor{lightgray} Ours (VS, RO, PS) & \cellcolor{lightgray} $ H b \sqrt{\gC(\pi; 1/\sqrt{K}) \cdot d / K}$\\
       \hline
    \end{tabular}
    }
    \vspace{10pt}
    \caption{Sub-optimality bounds when the function class $\gF_h$ is linear in $\phi_h: \gS \times \gA \rightarrow \sR^d$.}
    \label{tab: sub-optimality bounds in linear function classes}
\end{table}

\section{Conclusion}
We contributed to the understanding of sample-efficient offline RL in the context of (value) function approximation. We proposed a notion of data diversity that generalizes the previous data coverage measures and importantly expands the class of sample-efficient offline RL. We studied three different algorithms: VS, RO, and PS, where the PS-based algorithm is our novel proposal. We showed that VS, RO, and PS all have same-order guarantees under standard assumptions. 

\section*{Acknowledgements}
This research was supported, in part, by DARPA GARD award HR00112020004, NSF CAREER award IIS-1943251,  funding from the Institute for Assured Autonomy (IAA) at JHU, and the Spring'22 workshop on ``Learning and Games'' at the Simons Institute for the Theory of Computing. \revise{We thank Nan Jiang (UIUC) for bringing to our attention the Bellman error transfer coefficient by ~\citet{song2022hybrid}.}

\newpage

\bibliographystyle{plainnat}

\bibliography{main}

\begin{thebibliography}{58}
\providecommand{\natexlab}[1]{#1}
\providecommand{\url}[1]{\texttt{#1}}
\expandafter\ifx\csname urlstyle\endcsname\relax
  \providecommand{\doi}[1]{doi: #1}\else
  \providecommand{\doi}{doi: \begingroup \urlstyle{rm}\Url}\fi

\bibitem[Agarwal and Zhang(2022)]{agarwal2022non}
Alekh Agarwal and Tong Zhang.
\newblock Non-linear reinforcement learning in large action spaces: Structural conditions and sample-efficiency of posterior sampling.
\newblock In \emph{Conference on Learning Theory}, pages 2776--2814. PMLR, 2022.

\bibitem[Agarwal et~al.(2021)Agarwal, Kakade, Lee, and Mahajan]{agarwal2021theory}
Alekh Agarwal, Sham~M Kakade, Jason~D Lee, and Gaurav Mahajan.
\newblock On the theory of policy gradient methods: Optimality, approximation, and distribution shift.
\newblock \emph{J. Mach. Learn. Res.}, 22\penalty0 (98):\penalty0 1--76, 2021.

\bibitem[Antos et~al.(2006)Antos, Szepesv{\'{a}}ri, and Munos]{DBLP:conf/colt/AntosSM06}
Andr{\'{a}}s Antos, Csaba Szepesv{\'{a}}ri, and R{\'{e}}mi Munos.
\newblock Learning near-optimal policies with bellman-residual minimization based fitted policy iteration and a single sample path.
\newblock In G{\'{a}}bor Lugosi and Hans~Ulrich Simon, editors, \emph{Learning Theory, 19th Annual Conference on Learning Theory, {COLT} 2006, Pittsburgh, PA, USA, June 22-25, 2006, Proceedings}, volume 4005 of \emph{Lecture Notes in Computer Science}, pages 574--588. Springer, 2006.

\bibitem[Arora et~al.(2012)Arora, Hazan, and Kale]{v008a006}
Sanjeev Arora, Elad Hazan, and Satyen Kale.
\newblock The multiplicative weights update method: a meta-algorithm and applications.
\newblock \emph{Theory of Computing}, 8\penalty0 (6):\penalty0 121--164, 2012.
\newblock \doi{10.4086/toc.2012.v008a006}.

\bibitem[Bartlett et~al.(2005)Bartlett, Bousquet, and Mendelson]{bartlett2005local}
Peter~L Bartlett, Olivier Bousquet, and Shahar Mendelson.
\newblock Local rademacher complexities.
\newblock \emph{The Annals of Statistics}, 33\penalty0 (4):\penalty0 1497--1537, 2005.

\bibitem[Chen and Jiang(2019)]{chen2019information}
Jinglin Chen and Nan Jiang.
\newblock Information-theoretic considerations in batch reinforcement learning.
\newblock In \emph{International Conference on Machine Learning}, pages 1042--1051. PMLR, 2019.

\bibitem[Chen and Jiang(2022)]{Chen2022OfflineRL}
Jinglin Chen and Nan Jiang.
\newblock Offline reinforcement learning under value and density-ratio realizability: the power of gaps.
\newblock In \emph{Uncertainty in Artificial Intelligence}, pages 378--388. PMLR, 2022.

\bibitem[Cheng et~al.(2022)Cheng, Xie, Jiang, and Agarwal]{cheng2022adversarially}
Ching-An Cheng, Tengyang Xie, Nan Jiang, and Alekh Agarwal.
\newblock Adversarially trained actor critic for offline reinforcement learning.
\newblock In \emph{International Conference on Machine Learning}, pages 3852--3878. PMLR, 2022.

\bibitem[Dann et~al.(2021)Dann, Mohri, Zhang, and Zimmert]{DBLP:conf/nips/DannMZZ21}
Christoph Dann, Mehryar Mohri, Tong Zhang, and Julian Zimmert.
\newblock A provably efficient model-free posterior sampling method for episodic reinforcement learning.
\newblock In \emph{Advances in Neural Information Processing Systems}, 2021.

\bibitem[Di et~al.(2023)Di, Zhao, He, and Gu]{di2023pessimistic}
Qiwei Di, Heyang Zhao, Jiafan He, and Quanquan Gu.
\newblock Pessimistic nonlinear least-squares value iteration for offline reinforcement learning.
\newblock \emph{arXiv preprint arXiv:2310.01380}, 2023.

\bibitem[Duan et~al.(2020)Duan, Jia, and Wang]{duan2020minimax}
Yaqi Duan, Zeyu Jia, and Mengdi Wang.
\newblock Minimax-optimal off-policy evaluation with linear function approximation.
\newblock In \emph{International Conference on Machine Learning}, pages 2701--2709. PMLR, 2020.

\bibitem[Ernst et~al.(2005)Ernst, Geurts, and Wehenkel]{ernst2005tree}
Damien Ernst, Pierre Geurts, and Louis Wehenkel.
\newblock Tree-based batch mode reinforcement learning.
\newblock \emph{Journal of Machine Learning Research}, 6, 2005.

\bibitem[Foster et~al.(2021)Foster, Kakade, Qian, and Rakhlin]{foster2021statistical}
Dylan~J Foster, Sham~M Kakade, Jian Qian, and Alexander Rakhlin.
\newblock The statistical complexity of interactive decision making.
\newblock \emph{arXiv preprint arXiv:2112.13487}, 2021.

\bibitem[Freedman(1975)]{Freedman1975OnTP}
David~A. Freedman.
\newblock On tail probabilities for martingales.
\newblock \emph{The Annals of Probability}, 3\penalty0 (1):\penalty0 100--118, 1975.
\newblock ISSN 00911798.

\bibitem[Gabbianelli et~al.(2023)Gabbianelli, Neu, Okolo, and Papini]{gabbianelli2023offline}
Germano Gabbianelli, Gergely Neu, Nneka Okolo, and Matteo Papini.
\newblock Offline primal-dual reinforcement learning for linear mdps.
\newblock \emph{arXiv preprint arXiv:2305.12944}, 2023.

\bibitem[Ji et~al.(2022)Ji, Chen, Wang, and Zhao]{ji2022sample}
Xiang Ji, Minshuo Chen, Mengdi Wang, and Tuo Zhao.
\newblock Sample complexity of nonparametric off-policy evaluation on low-dimensional manifolds using deep networks.
\newblock \emph{arXiv preprint arXiv:2206.02887}, 2022.

\bibitem[Jiang et~al.(2017)Jiang, Krishnamurthy, Agarwal, Langford, and Schapire]{jiang2017contextual}
Nan Jiang, Akshay Krishnamurthy, Alekh Agarwal, John Langford, and Robert~E Schapire.
\newblock Contextual decision processes with low bellman rank are pac-learnable.
\newblock In \emph{International Conference on Machine Learning}, pages 1704--1713. PMLR, 2017.

\bibitem[Jin et~al.(2020)Jin, Yang, Wang, and Jordan]{jin2020provably}
Chi Jin, Zhuoran Yang, Zhaoran Wang, and Michael~I Jordan.
\newblock Provably efficient reinforcement learning with linear function approximation.
\newblock In \emph{Conference on Learning Theory}, pages 2137--2143. PMLR, 2020.

\bibitem[Jin et~al.(2021{\natexlab{a}})Jin, Liu, and Miryoosefi]{Jin2021BellmanED}
Chi Jin, Qinghua Liu, and Sobhan Miryoosefi.
\newblock Bellman eluder dimension: New rich classes of {RL} problems, and sample-efficient algorithms.
\newblock In A.~Beygelzimer, Y.~Dauphin, P.~Liang, and J.~Wortman Vaughan, editors, \emph{Advances in Neural Information Processing Systems}, 2021{\natexlab{a}}.

\bibitem[Jin et~al.(2021{\natexlab{b}})Jin, Yang, and Wang]{jin2021pessimism}
Ying Jin, Zhuoran Yang, and Zhaoran Wang.
\newblock Is pessimism provably efficient for offline rl?
\newblock In \emph{International Conference on Machine Learning}, pages 5084--5096. PMLR, 2021{\natexlab{b}}.

\bibitem[Jin et~al.(2022)Jin, Ren, Yang, and Wang]{jin2022policy}
Ying Jin, Zhimei Ren, Zhuoran Yang, and Zhaoran Wang.
\newblock Policy learning" without''overlap: Pessimism and generalized empirical bernstein's inequality.
\newblock \emph{arXiv preprint arXiv:2212.09900}, 2022.

\bibitem[Lange et~al.(2012)Lange, Gabel, and Riedmiller]{lange2012batch}
Sascha Lange, Thomas Gabel, and Martin Riedmiller.
\newblock Batch reinforcement learning.
\newblock In \emph{Reinforcement learning}, pages 45--73. Springer, 2012.

\bibitem[Levine et~al.(2020)Levine, Kumar, Tucker, and Fu]{levine2020offline}
Sergey Levine, Aviral Kumar, George Tucker, and Justin Fu.
\newblock Offline reinforcement learning: Tutorial, review, and perspectives on open problems.
\newblock \emph{arXiv preprint arXiv:2005.01643}, 2020.

\bibitem[Liu et~al.(2019)Liu, Swaminathan, Agarwal, and Brunskill]{DBLP:conf/uai/LiuSAB19}
Yao Liu, Adith Swaminathan, Alekh Agarwal, and Emma Brunskill.
\newblock Off-policy policy gradient with stationary distribution correction.
\newblock In Amir Globerson and Ricardo Silva, editors, \emph{Proceedings of the Thirty-Fifth Conference on Uncertainty in Artificial Intelligence, {UAI} 2019, Tel Aviv, Israel, July 22-25, 2019}, volume 115 of \emph{Proceedings of Machine Learning Research}, pages 1180--1190. {AUAI} Press, 2019.

\bibitem[Massart(2000)]{Massart2000SomeAO}
Pascal Massart.
\newblock Some applications of concentration inequalities to statistics.
\newblock \emph{Annales de la Facult{\'e} des Sciences de Toulouse}, 9:\penalty0 245--303, 2000.

\bibitem[Munos and Szepesv{\'{a}}ri(2008)]{DBLP:journals/jmlr/MunosS08}
R{\'{e}}mi Munos and Csaba Szepesv{\'{a}}ri.
\newblock Finite-time bounds for fitted value iteration.
\newblock \emph{J. Mach. Learn. Res.}, 9:\penalty0 815--857, 2008.

\bibitem[Nguyen-Tang and Arora(2023)]{nguyentang2023}
Thanh Nguyen-Tang and Raman Arora.
\newblock {VIP}er: Provably efficient algorithm for offline {RL} with neural function approximation.
\newblock In \emph{The Eleventh International Conference on Learning Representations}, 2023.

\bibitem[Nguyen-Tang et~al.(2022{\natexlab{a}})Nguyen-Tang, Gupta, Nguyen, and Venkatesh]{nguyen2021offline}
Thanh Nguyen-Tang, Sunil Gupta, A.~Tuan Nguyen, and Svetha Venkatesh.
\newblock Offline neural contextual bandits: Pessimism, optimization and generalization.
\newblock In \emph{International Conference on Learning Representations}, 2022{\natexlab{a}}.

\bibitem[Nguyen-Tang et~al.(2022{\natexlab{b}})Nguyen-Tang, Gupta, Tran-The, and Venkatesh]{nguyen-tang2022on}
Thanh Nguyen-Tang, Sunil Gupta, Hung Tran-The, and Svetha Venkatesh.
\newblock On sample complexity of offline reinforcement learning with deep re{LU} networks in besov spaces.
\newblock \emph{Transactions of Machine Learning Research}, 2022{\natexlab{b}}.

\bibitem[Nguyen-Tang et~al.(2023)Nguyen-Tang, Yin, Gupta, Venkatesh, and Arora]{nguyen2022instance}
Thanh Nguyen-Tang, Ming Yin, Sunil Gupta, Svetha Venkatesh, and Raman Arora.
\newblock On instance-dependent bounds for offline reinforcement learning with linear function approximation.
\newblock In \emph{Proceedings of the AAAI Conference on Artificial Intelligence}, volume~37, pages 9310--9318, 2023.

\bibitem[Osband et~al.(2016)Osband, Blundell, Pritzel, and Van~Roy]{osband2016deep}
Ian Osband, Charles Blundell, Alexander Pritzel, and Benjamin Van~Roy.
\newblock Deep exploration via bootstrapped dqn.
\newblock \emph{Advances in neural information processing systems}, 29, 2016.

\bibitem[Ozdaglar et~al.(2023)Ozdaglar, Pattathil, Zhang, and Zhang]{ozdaglar2023revisiting}
Asuman~E Ozdaglar, Sarath Pattathil, Jiawei Zhang, and Kaiqing Zhang.
\newblock Revisiting the linear-programming framework for offline rl with general function approximation.
\newblock In \emph{International Conference on Machine Learning}, pages 26769--26791. PMLR, 2023.

\bibitem[Rashidinejad et~al.(2021)Rashidinejad, Zhu, Ma, Jiao, and Russell]{rashidinejad2021bridging}
Paria Rashidinejad, Banghua Zhu, Cong Ma, Jiantao Jiao, and Stuart Russell.
\newblock Bridging offline reinforcement learning and imitation learning: A tale of pessimism.
\newblock \emph{Advances in Neural Information Processing Systems}, 34, 2021.

\bibitem[Rashidinejad et~al.(2023)Rashidinejad, Zhu, Yang, Russell, and Jiao]{rashidinejad2022optimal}
Paria Rashidinejad, Hanlin Zhu, Kunhe Yang, Stuart Russell, and Jiantao Jiao.
\newblock Optimal conservative offline {RL} with general function approximation via augmented lagrangian.
\newblock In \emph{The Eleventh International Conference on Learning Representations}, 2023.

\bibitem[Russo and Van~Roy(2014)]{russo2014learning}
Daniel Russo and Benjamin Van~Roy.
\newblock Learning to optimize via posterior sampling.
\newblock \emph{Mathematics of Operations Research}, 39\penalty0 (4):\penalty0 1221--1243, 2014.

\bibitem[Song et~al.(2022)Song, Zhou, Sekhari, Bagnell, Krishnamurthy, and Sun]{song2022hybrid}
Yuda Song, Yifei Zhou, Ayush Sekhari, J~Andrew Bagnell, Akshay Krishnamurthy, and Wen Sun.
\newblock Hybrid rl: Using both offline and online data can make rl efficient.
\newblock \emph{arXiv preprint arXiv:2210.06718}, 2022.

\bibitem[Thompson(1933)]{Thompson1933ONTL}
William~R. Thompson.
\newblock On the likelihood that one unknown probability exceeds another in view of the evidence of two samples.
\newblock \emph{Biometrika}, 25:\penalty0 285--294, 1933.

\bibitem[Tripuraneni et~al.(2020)Tripuraneni, Jordan, and Jin]{tripuraneni2020theory}
Nilesh Tripuraneni, Michael Jordan, and Chi Jin.
\newblock On the theory of transfer learning: The importance of task diversity.
\newblock \emph{Advances in neural information processing systems}, 33:\penalty0 7852--7862, 2020.

\bibitem[Uehara and Sun(2022)]{uehara2021pessimistic}
Masatoshi Uehara and Wen Sun.
\newblock Pessimistic model-based offline reinforcement learning under partial coverage.
\newblock In \emph{International Conference on Learning Representations}, 2022.

\bibitem[Uehara et~al.(2022{\natexlab{a}})Uehara, Shi, and Kallus]{uehara2022review}
Masatoshi Uehara, Chengchun Shi, and Nathan Kallus.
\newblock A review of off-policy evaluation in reinforcement learning.
\newblock \emph{arXiv preprint arXiv:2212.06355}, 2022{\natexlab{a}}.

\bibitem[Uehara et~al.(2022{\natexlab{b}})Uehara, Zhang, and Sun]{uehara2022representation}
Masatoshi Uehara, Xuezhou Zhang, and Wen Sun.
\newblock Representation learning for online and offline {RL} in low-rank {MDP}s.
\newblock In \emph{International Conference on Learning Representations}, 2022{\natexlab{b}}.

\bibitem[Wang et~al.(2021)Wang, Foster, and Kakade]{wang2020statistical}
Ruosong Wang, Dean Foster, and Sham~M. Kakade.
\newblock What are the statistical limits of offline {RL} with linear function approximation?
\newblock In \emph{International Conference on Learning Representations}, 2021.

\bibitem[Watkins et~al.(2023)Watkins, Ullah, Nguyen-Tang, and Arora]{watkins2023optimistic}
Austin Watkins, Enayat Ullah, Thanh Nguyen-Tang, and Raman Arora.
\newblock Optimistic rates for multi-task representation learning.
\newblock In \emph{Thirty-seventh Conference on Neural Information Processing Systems}, 2023.

\bibitem[Welling and Teh(2011)]{welling2011bayesian}
Max Welling and Yee~W Teh.
\newblock Bayesian learning via stochastic gradient langevin dynamics.
\newblock In \emph{Proceedings of the 28th international conference on machine learning (ICML-11)}, pages 681--688, 2011.

\bibitem[Xie et~al.(2021)Xie, Cheng, Jiang, Mineiro, and Agarwal]{xie2021bellman}
Tengyang Xie, Ching-An Cheng, Nan Jiang, Paul Mineiro, and Alekh Agarwal.
\newblock Bellman-consistent pessimism for offline reinforcement learning.
\newblock \emph{Advances in neural information processing systems}, 34, 2021.

\bibitem[Xiong et~al.(2022)Xiong, Zhong, Shi, Shen, and Zhang]{xiong2022self}
Wei Xiong, Han Zhong, Chengshuai Shi, Cong Shen, and Tong Zhang.
\newblock A self-play posterior sampling algorithm for zero-sum markov games.
\newblock In \emph{International Conference on Machine Learning}, pages 24496--24523. PMLR, 2022.

\bibitem[Xiong et~al.(2023)Xiong, Zhong, Shi, Shen, Wang, and Zhang]{xiong2022nearly}
Wei Xiong, Han Zhong, Chengshuai Shi, Cong Shen, Liwei Wang, and Tong Zhang.
\newblock Nearly minimax optimal offline reinforcement learning with linear function approximation: Single-agent {MDP} and markov game.
\newblock In \emph{The Eleventh International Conference on Learning Representations}, 2023.

\bibitem[Yang and Wang(2019)]{yang2019sample}
Lin Yang and Mengdi Wang.
\newblock Sample-optimal parametric q-learning using linearly additive features.
\newblock In \emph{International Conference on Machine Learning}, pages 6995--7004. PMLR, 2019.

\bibitem[Yin and Wang(2021)]{yin2021towards}
Ming Yin and Yu-Xiang Wang.
\newblock Towards instance-optimal offline reinforcement learning with pessimism.
\newblock \emph{Advances in neural information processing systems}, 34, 2021.

\bibitem[Zanette et~al.(2020)Zanette, Lazaric, Kochenderfer, and Brunskill]{zanette2020learning}
Andrea Zanette, Alessandro Lazaric, Mykel Kochenderfer, and Emma Brunskill.
\newblock Learning near optimal policies with low inherent bellman error.
\newblock In \emph{International Conference on Machine Learning}, pages 10978--10989. PMLR, 2020.

\bibitem[Zanette et~al.(2021)Zanette, Wainwright, and Brunskill]{zanette2021provable}
Andrea Zanette, Martin~J Wainwright, and Emma Brunskill.
\newblock Provable benefits of actor-critic methods for offline reinforcement learning.
\newblock \emph{Advances in neural information processing systems}, 34:\penalty0 13626--13640, 2021.

\bibitem[Zhan et~al.(2023)Zhan, Ren, Athey, and Zhou]{zhan2021policy}
Ruohan Zhan, Zhimei Ren, Susan Athey, and Zhengyuan Zhou.
\newblock Policy learning with adaptively collected data.
\newblock \emph{Management Science}, 2023.

\bibitem[Zhan et~al.(2022)Zhan, Huang, Huang, Jiang, and Lee]{zhan2022offline}
Wenhao Zhan, Baihe Huang, Audrey Huang, Nan Jiang, and Jason Lee.
\newblock Offline reinforcement learning with realizability and single-policy concentrability.
\newblock In Po-Ling Loh and Maxim Raginsky, editors, \emph{Proceedings of Thirty Fifth Conference on Learning Theory}, volume 178 of \emph{Proceedings of Machine Learning Research}, pages 2730--2775. PMLR, 02--05 Jul 2022.

\bibitem[Zhang(2022)]{zhang2022feel}
Tong Zhang.
\newblock Feel-good thompson sampling for contextual bandits and reinforcement learning.
\newblock \emph{SIAM Journal on Mathematics of Data Science}, 4\penalty0 (2):\penalty0 834--857, 2022.

\bibitem[Zhang(2023)]{zhang_2023}
Tong Zhang.
\newblock \emph{Mathematical Analysis of Machine Learning Algorithms}.
\newblock Cambridge University Press, 2023.

\bibitem[Zhao et~al.(2023)Zhao, Yang, Wang, and Lee]{zhao2023local}
Yulai Zhao, Zhuoran Yang, Zhaoran Wang, and Jason~D Lee.
\newblock Local optimization achieves global optimality in multi-agent reinforcement learning.
\newblock \emph{arXiv preprint arXiv:2305.04819}, 2023.

\bibitem[Zhong et~al.(2022)Zhong, Xiong, Zheng, Wang, Wang, Yang, and Zhang]{zhong2022posterior}
Han Zhong, Wei Xiong, Sirui Zheng, Liwei Wang, Zhaoran Wang, Zhuoran Yang, and Tong Zhang.
\newblock A posterior sampling framework for interactive decision making.
\newblock \emph{arXiv preprint arXiv:2211.01962}, 2022.

\bibitem[Zhu et~al.(2023)Zhu, Rashidinejad, and Jiao]{zhu2023importance}
Hanlin Zhu, Paria Rashidinejad, and Jiantao Jiao.
\newblock Importance weighted actor-critic for optimal conservative offline reinforcement learning.
\newblock \emph{arXiv preprint arXiv:2301.12714}, 2023.

\end{thebibliography}
\newpage

\tableofcontents

\newpage

\begin{appendices}

\section{Preparation}
We now get into more involved parts where we present the proof process and the technical results for obtaining \Cref{theorem: gopo-vsc}, \Cref{theorem: gopo-roc}, and \Cref{theorem: gopo-psc}. In order to prove our main results in \Cref{section: main results}, we shall need some old tools and develop some new useful tools. For convenience, we start out with both old and new notations of quantities summarized in \Cref{table: notations and quantities} that we are going to use frequently in our proofs. 

\begin{table}[H]
    \centering
    \def\arraystretch{1.8}%
    \resizebox{1\textwidth}{!}{
    \begin{tabular}{l|l|l}
    \hline
    \text{Name} & \text{Notation} & \text{Expression} \\
    \hline
     \text{transition sample} & $z^k_h$ & $(s^k_h, a^k_h, r^k_h, a^k_{h+1})$ \\
     \text{transition sample} & $z_h$ & $(s_h, a_h,r_h, s_{h+1})$ \\ 
    \text{Bellman error}&$\gE_h^{\tilde{\pi}}(f_h, f_{h+1})(s_h,a_h)$ & $( \sT_{h}^{\tilde{\pi}} f_{h+1} - f_h)(s_h,a_h)$ \\ 
    \text{TD loss function} & $l_{\tilde{\pi}}(f_h,f_{h+1}; z_h)$ &$(f_h(s_h,a_h) - r_h - f_{h+1}(s_{h+1}, \tilde{\pi}))^2$ \\ 
    \text{empirical squared Bellman error (SBE)} & $\hat{L}_{\tilde{\pi}}(f_h,f_{h+1})$ & $\sum_{k=1}^K  l_{\tilde{\pi}}(f_h,f_{h+1}; z^k_h)$ \\ 
    \text{empirical bias-adjusted SBE}& $\gL_{\tpi}(f)$ &  $\sum_{h=1}^H   \hat{L}_{\tpi}(f_h,f_{h+1}) - \inf_{g \in \gF} \sum_{h=1}^H   \hat{L}_{\tpi}(g_h,f_{h+1})$ \\
    \text{excess TD loss}& $\Delta L_{\tilde{\pi}}(f_h,f_{h+1}; z_h)$ & $l_{\tilde{\pi}}(f_h, f_{h+1}; z_h) - l_{\tilde{\pi}}(\sT_h^{\tilde{\pi}}f_{h+1}, f_{h+1}; z_h)$\\ 
    -- & $\sE_{\mu}[\cdot]$ & $\frac{1}{K} \sum_{k=1}^K \sE_{\mu^k}[\cdot]$\\ 
    --& $\sE_k[\cdot] (:= \sE_{\mu^k}[\cdot])$ & $\sE\left[\cdot \bigg\vert \{z^i_h\}_{h \in [H]}^{i \in [k-1]} \right]$ \\ 
    \hline
    \end{tabular}
    }
\caption{A summary of notations and quantities of interest.}
\label{table: notations and quantities}
\end{table}

The quantity $l_{\tilde{\pi}}(f_h,f_{h+1}; z_h)$ can be viewed as \emph{a temporal difference (TD) loss function} defined on data point $z_h$ conditioned on each $f_{h+1}$ and $\tilde{\pi}$. The quantity $\sT_h^{\tilde{\pi}} f_{h+1}$ can be viewed as \emph{the Bellman regression function}, where, conditioned on each  $f_{h+1}$ and $\tilde{\pi}$, for any $(s_h, a_h)$, we have
\begin{align*}
    \sT_h^{\tilde{\pi}} f_{h+1}(s_h,a_h) = \sE_{r_h,s_{h+1} | s_h, a_h}[r_h + f_{h+1}(s_{h+1}, \tilde{\pi})] = \arginf_g \sE_{r_h,s_{h+1} | s_h, a_h}l_{\tilde{\pi}} (g, f_{h+1}; z_h).
\end{align*}
Thus, the quantity $\Delta L_{\tilde{\pi}}(f_h,f_{h+1}; z_h)$ can be referred to as the \emph{excess TD loss}, incurred by the predictor $f_h$, relative to the TD regression function $ \sT_h^{\tilde{\pi}} f_{h+1}$, on data $z_h$ and conditioned on $f_{h+1}$ and $\tilde{\pi}$.

\subsection{Variance condition and Bernstein's inequality}
We also define the $\sigma$-algebra $\gA^k_h := \sigma(\gD_{k-1} \cup \{ (s^k_{h'}, a^k_{h'}, r^k_{h'}) \}_{h' \in [h-1]} \cup (s^k_h, a^k_h))$ and denote $\sE_{k,h}[\cdot] := \sE[\cdot| \gA^k_h]$. The following lemma establishes the variance condition on the excess TD loss, a TD analogous to the variance condition that is widely used in the empirical process theory \citep{Massart2000SomeAO}. 
\begin{lemma}
For any $\gA^k_h$-measurable policy $\pi$, we have
\begin{align*}
    \sE_{k,h}[\Delta L_{\pi}(f_h, f_{h+1}; z^k_h) ] &= \gE^{\pi}_h(f_h, f_{h+1})(s^k_h,a^k_h)^2, \\
     \sE_{k,h}[\Delta L_{\pi}(f_h, f_{h+1}; z^k_h)^2 ] &\leq 36b^2 \gE^{\pi}_h(f_h, f_{h+1})(s^k_h,a^k_h)^2.
\end{align*}
\label{lemma: the squared residuals bound the expected value and the variance of the empirical minimax error}
\end{lemma}
\begin{proof}[Proof of \Cref{lemma: the squared residuals bound the expected value and the variance of the empirical minimax error}]
The result directly exploits the boundedness of the TD loss function and the squared loss is Lipschitz. In concrete, it is a direct application of \Cref{lemma: variation condition for least squares}.
\end{proof}

The following lemma establishes the martingale extension of Bernstein's inequality, typically called Freedman's inequality \citep{Freedman1975OnTP}. In this lemma, we prove a slightly modified version of the original Freedman's inequality for our own convenience. The proof for this lemma is elementary which we also show here. 
\begin{lemma}[Freedman's inequality]
Let $X_1, \ldots, X_T$ be \emph{any} sequence of real-valued random variables. Denote $\sE_t[\cdot] = \sE[\cdot| X_1, \ldots, X_{t-1}]$. Assume that $X_t \leq R$ for some $R > 0$ and $\sE_t[X_t]=0$ for all $t$. Define the random variables
\begin{align*}
    S := \sum_{t=1}^T X_t, \hspace{10pt} V := \sum_{i=1}^T \sE_t[X_t^2]. 
\end{align*}
Then for any $\delta > 0$, with probability at least $1 - \delta$, for any $\lambda \in [0,1/R]$, 
\begin{align*}
    S \leq (e - 2) \lambda V + \frac{\ln(1/\delta)}{\lambda}.
\end{align*}
\label{lemma: Freedman inequality}
\end{lemma}
\begin{proof}[Proof of \Cref{lemma: Freedman inequality}]
Let us define the following sequence of random variables: $Z_0 = 1, Z_t = Z_{t-1} \frac{e^{\lambda X_t}}{\sE_t[e^{\lambda X_t}]}$. We have 
\begin{align*}
    \sE_t[Z_t] &= \sE_t \left[ Z_{t-1} \frac{e^{\lambda X_t}}{\sE_t[e^{\lambda X_t}]} \right] = \frac{Z_{t-1}}{\sE_t[e^{\lambda X_t}]} \sE_t[e^{\lambda X_t}] = Z_{t-1}. 
\end{align*}
Thus, we have 
\begin{align*}
    \sE[Z_T] = \sE \sE_T[Z_T] = \sE[Z_{T-1}] = \ldots = \sE[Z_0] = 1.
\end{align*}
Note that 
\begin{align}
    Z_T = \frac{e ^{\lambda S}}{ \prod_{t=1}^T \sE_t[e^{\lambda X_t}]} = \frac{e ^{\lambda S}}{\sum_{t=1}^T e^{\ln \sE_t[e^{\lambda X_t}]}} = \exp \left( \lambda S - \sum_{t=1}^T \ln \sE_t[e^{\lambda X_t} ] \right). 
    \label{eq: temp Z_T}
\end{align}
Since $Z_T \geq 0$, it follows from Markov's inequality that, for any $\delta > 0$, we have 
\begin{align}
    \Pr(Z_T \geq 1/\delta) \leq \delta \sE[Z_T] = \delta.
    \label{eq: apply Markov's inequality to Z_T}
\end{align}
We now bound the logarithmic moment generating function $\ln \sE_t[e^{\lambda X_t}] $ using elementary inequalities: For any $\lambda \in [0,1/R]$, we have
\begin{align}
    \ln \sE_t[e^{\lambda X_t}] \leq \sE_t[e^{\lambda X_t}] - 1 \leq \lambda \sE_t[X_t] + (e-2) \lambda^2 \sE_t[X_t^2],
    \label{eq: bound the logarithmic moment generating function}
\end{align}
where the first inequality uses $\ln z \leq z - 1, \forall z \geq 0$ and the second inequality uses that $e^z \leq 1 + z + (e-2) z^2, \forall z \leq 1$ and that $\lambda X_t \leq 1$. 

Plugging \Cref{eq: bound the logarithmic moment generating function} into \Cref{eq: temp Z_T}, then all together into \Cref{eq: apply Markov's inequality to Z_T} complete the proof.

\end{proof}
\subsection{Functional projections for misspecification}
Since \Cref{assumption: realizability} and \Cref{assumption: approximate Bellman completeness} allow misspecification up to some errors $\xi$ and $\nu$, while we are working on the function class $\gF$, we rely on the following projection operators, \Cref{defn: projection of value function} and \Cref{defn: projection of Bellman operation}, to handle misspecification. 
\begin{defn}[Projection of action-value functions] For any $\tilde{\pi} \in \Pi^{soft}(T)$ for some $T \in \sN$, we define the projection of the state-action value function $\tilde{\pi}$ onto $\gF$ as
\begin{align*}
    \proj_{\gF}(Q^{\tilde{\pi}}) := \argmin_{f \in \gF}\left\{|f_h(s_h,a_h) - Q_h^{\tilde{\pi}}(s_h,a_h)|, \forall h \in [H], (s_h, a_h) \in \supp(d^{\mu}_h) \right\}. 
\end{align*}
\label{defn: projection of value function}
\end{defn}
By \Cref{assumption: realizability}, we have 
\begin{align*}
    |\proj_{\gF}(Q^{\tilde{\pi}})(s_h,a_h) - Q_h^{\tilde{\pi}}(s_h,a_h)| \leq \xi_h, ~~\forall h \in [H], (s_h, a_h) \in \supp(d^{\mu}_h). 
\end{align*}

\begin{defn}[Projection of Bellman operations]
For any $f \in \gF$ and $\tilde{\pi} \in \Pi^{soft}(T)$ for some $T$, we define the projection of the Bellman operation $\sT^{\tilde{\pi}} f$ onto $\gF$ as 
\begin{align*}
    \proj_{\gF} (\sT^{\tilde{\pi}} f) := \argmin_{f' \in \gF} \left\{\| f'_h - \sT_h^{\tilde{\pi}} f_{h+1} \|_{\infty}, \forall h \in [H] \right\}. 
\end{align*}
\label{defn: projection of Bellman operation}
\end{defn}
By \Cref{assumption: approximate Bellman completeness}, we have 
\begin{align*}
    \| \proj_{\gF} (\sT^{\tilde{\pi}} f) - \sT_h^{\tilde{\pi}} f_{h+1} \|_{\infty} \leq \nu_h, \forall h \in [H].
\end{align*} 

\subsection{Induced MDPs}
We now introduce the notion of \emph{induced} MDPs which is originally used in \citep{zanette2021provable}. 
\begin{defn}[Induced MDPs]
For any policy $\pi \in \Pi^{all}$ and any sequence of functions $Q = \{Q_h\}_{h \in [H]} \in \{\gS \times \gA \rightarrow \sR\}^H$, the $(Q,\pi)$-induced MDPs, denoted by $M(Q,\pi)$ is the MDP that is identical to the original MDP $M$ except only that the expected reward of $M(Q,\pi)$ is given by $\{r_h^{\pi,Q}\}_{h \in [H]}$, where 
\begin{align*}
    r_h^{\pi,Q}(s,a) := r_h(s,a) - \gE_h^{\pi}(f_h, f_{h+1})(s, a).
\end{align*}
\label{defn: induced mdp}
\end{defn}

By definition of $M(\pi,Q)$, $Q$ is the fixed point of the Bellman equation $Q_h = \sT^{\pi}_{h,M(\pi,Q)} Q_{h+1}$. 
\begin{lemma}
    For any $\pi \in \Pi^{all}$ and any sequence of functions $Q = \{Q_h\}_{h \in [H]} \in \{\gS \times \gA \rightarrow \sR\}^H$, we have 
\begin{align*}
    Q^{\pi}_{M(\pi, Q)} = Q,
\end{align*}
where $M(\pi, Q)$ is the induced MDP given in \Cref{defn: induced mdp}. 
\label{lemma: exact value function under the induced MDP}
\end{lemma}

\subsection{Error decomposition}
The key starting point for the proofs of all of the three main theorems is the following error decomposition that decomposes the sub-optimality into three sources of errors: the Bellman error under the comparator policy $\pi$, the gap values in the initial states, and the online-regret term due to the induced MDPs. In online RL, the sub-optimality of a greedy policy against an optimal policy can be decomposed into the sub-optimality in the Bellman errors and the error in the initial states \citep{DBLP:conf/nips/DannMZZ21}, using the standard value-function error decomposition in \cite[Lemma~1]{jiang2017contextual}. However, in our setting, we compete against an arbitrary policy $\pi$ (not necessarily an optimal policy) and the learned policy $\pi^t$ is not greedy with respect to the current action-value function $\ubar{Q}^t$ -- thus \cite[Lemma~1]{jiang2017contextual} cannot apply here. Instead, we develop an error decomposition -- \Cref{lemma: error decomposition} which generalizes what was implicit in \cite{zanette2021provable}. 
\begin{lemma}[Error decomposition]
    For any action-value functions $Q \in \{\gS \times \gA \rightarrow \sR\}^H$ and any policies $\pi, \tilde{\pi} \in \Pi^{all}$, we have 
    \begin{align*}
        \subopt^M_{\pi}(\tilde{\pi}) = \sum_{h=1}^H \sE_{\pi} [\gE_h^{\tilde{\pi}}(Q_h, Q_{h+1})(s_h,a_h)] + Q_1(s_1, \tilde{\pi}_1) - V_1^{\tilde{\pi}}(s_1) + \subopt_{\pi}^{M(Q, \tilde{\pi})}(\tilde{\pi}). 
    \end{align*}
    \label{lemma: error decomposition}
\end{lemma}
\begin{proof}[Proof of \Cref{lemma: error decomposition}]
We have 
\begin{align*}
    &\subopt^M_{\pi}(\tilde{\pi}) = V_1^{\pi}(s_1) - V_1^{\tilde{\pi}}(s_1) \\ 
    &= \left(V_1^{\pi}(s_1) - V_{1,M(Q,\tilde{\pi})}^{\pi}(s_1) \right)+ \left(V_{1,M(Q,\tilde{\pi})}^{\tilde{\pi}}(s_1) - V_1^{\tilde{\pi}}(s_1) \right) + \left(V_{1,M(Q,\tilde{\pi})}^{\pi}(s_1)  - V_{1,M(Q,\tilde{\pi})}^{\tilde{\pi}}(s_1) \right) \\ 
    &=\sum_{h=1}^H \sE_{\pi} [\gE_h^{\tilde{\pi}}(Q_h, Q_{h+1})(s_h,a_h)] + Q_1(s_1, \tilde{\pi}_1) - V_1^{\tilde{\pi}}(s_1) + \subopt_{\pi}^{M(Q, \tilde{\pi})}(\tilde{\pi}), 
\end{align*}
where in the last equality, for the first term, we use, by \Cref{defn: induced mdp}, that 
\begin{align*}
    V_1^{\pi}(s_1) - V_{1,M(Q,\tilde{\pi})}^{\pi}(s_1) &= \sum_{h=1}^H \sE_{\pi} \left[ r_h(s_h, a_h) - r_h^{\tilde{\pi},Q}(s_h, a_h) \right] =\sum_{h=1}^H \sE_{\pi} [\gE_h^{\tilde{\pi}}(Q_h, Q_{h+1})(s_h,a_h)],
\end{align*}
for the second term, we use, by \Cref{lemma: exact value function under the induced MDP}, that 
\begin{align*}
    V_{1,M(Q,\tilde{\pi})}^{\pi}(s_1) = Q_{1,M(Q,\tilde{\pi})}^{\pi}(s_1, \tilde{\pi}_1) = Q_1(s_1, \tilde{\pi}),
\end{align*}
and for the last term, we use the definition of value sub-optimality in \Cref{eq: sub-optimality definition}. 
\end{proof}

\subsection{Decoupling lemma}
One of the central tools for our proofs is the following decoupling lemma. The decoupling lemma essentially decouples the Bellman residuals under the $\pi$-induced state-action distribution into the squared Bellman residuals under the $\mu$-induced state-action distribution and the new data diversity measure in \Cref{defn: distribution mismatch measure} and additive terms of low order. 
\begin{lemma}[Decoupling argument]
Under \Cref{assumption: approximate Bellman completeness}, for any $f \in \gF$, any $\tilde{\pi} \in \Pi^{soft}(T)$ for some $T$, any $\pi \in \Pi^{all}$, any $\lambda > 0$, and any $\eps \geq 0$, we have 
\begin{align*}
    \sum_{h=1}^H \sE_{\pi} [\gE^{\tilde{\pi}}_h(f_h, f_{h+1})(s_h,a_h)] &\leq \frac{1}{2 \lambda}\sum_{h=1}^H \left( \sum_{k=1}^K \sE_{\mu^k} \left[ \gE_h^{\tilde{\pi}}(f_h, f_{h+1})(s_h,a_h)^2 \right] + K \nu_h^2 + 4 b K \nu_h\right) \\ 
    &+ \frac{\lambda H \cdot \gC(\pi; \eps)}{2 K} + H \eps + \sum_{h=1}^H \nu_h,
\end{align*}
where $\gC(\pi; \eps)$ is defined in \Cref{defn: distribution mismatch measure}. 
\label{lemma: decoupling argument}
\end{lemma}
\begin{proof}[Proof of \Cref{lemma: decoupling argument}]
    We have 
    \begin{align*}
    &\sum_{h=1}^H \sE_{\pi}[\gE_h^{\tilde{\pi}}(f_h, f_{h+1})(s_h,a_h)] =\sum_{h=1}^H \sE_{\pi} \left[ (\sT_h^{\tilde{\pi}} f_{h+1} - f_{h})(s_h,a_h) \right]\\ 
     &\leq \sum_{h=1}^H \sE_{\pi} \left[ (\proj_{\gF_h}(\sT_h^{\tilde{\pi}} f_{h+1}) - f_{h})(s_h,a_h) \right] + \bar{\nu} \\ 
     &\leq \sum_{h=1}^H \sqrt{\gC(\pi,\eps)\sE_{\mu} \left[ (\proj_{\gF_h}(\sT_h^{\tilde{\pi}} f_{h+1}) - f_{h})(s_h,a_h)^2 \right] } + H \eps + \bar{\nu} \\
      &\leq \sum_{h=1}^H \sqrt{\gC(\pi,\eps) \left( \sE_{\mu} [\gE_h^{\tilde{\pi}}(f_{h}, f_{h+1})(s_h,a_h)^2] + \nu_h^2 + 4 b \nu_h\right) } + H \eps + \bar{\nu} \\ 
      &\leq \sqrt{H \gC(\pi,\eps) \sum_{h=1}^H \left(\sE_{\mu} [\gE_h^{\tilde{\pi}}(f_{h}, f_{h+1})(s_h,a_h)^2] + \nu_h^2 + 4 b \nu_h\right)} + H \eps + \bar{\nu} \\ 
      &\leq \frac{K}{2 \lambda}\sum_{h=1}^H \left(\sE_{\mu} [\gE_h^{\tilde{\pi}}(f_{h}, f_{h+1})(s_h,a_h)^2] + \nu_h^2 + 4 b \nu_h\right) + \frac{\lambda H \gC(\pi,\eps)}{2 K} + H \eps + \bar{\nu},
\end{align*}
where the first inequality uses \Cref{assumption: approximate Bellman completeness}, the second inequality uses the definition of $C_{\pi}(\eps)$, the third inequality uses \Cref{assumption: approximate Bellman completeness} (again), the fourth inequality uses Cauchy-Schwartz inequality, and the last inequality uses the AM-GM inequality $ \sqrt{xy} \leq \frac{K}{2 \lambda} x + \frac{\lambda}{2 K} y$. 
\end{proof}

\subsection{Regret of the multiplicative weights algorithm for the actors}
Now we establish the regret bound for the online-regret term due to the induced MDPs. The result in the following lemma is quite standard and can be readily generalized from a similar result in \cite{zanette2021provable}. We present the proof here for completeness.
\begin{lemma}
Consider an arbitrary sequence of value functions $\{Q^t\}_{t \in [T]}$ such that $\max_{h,t}\| Q^t_h \|_{\infty} \leq b$ and define the following sequence of policies $\{\pi^t\}_{t \in [T+1]}$ where 
\begin{align*}
    \pi^1(\cdot|s) &= \unif(\gA), \forall s, \\ 
    \pi^{t+1}_h(a|s) &\propto \pi^t_h(a|s) \exp \left( \eta Q^t_h(s,a) \right), \forall (s,a,h,t). 
\end{align*}
Suppose $\eta = \sqrt{\frac{\ln \vol(\gA)}{4(e-2) b^2 T}}$ and $T \geq \frac{\ln \vol(\gA)}{(e-2)}$, where $\vol(\gA)$ denotes the volume of the action set $\gA$. \footnote{When $|\gA| < \infty$, $\vol(\gA) = |\gA|$.}  For an arbitrary policy $\pi \in \Pi^{all}$, we have
\begin{align*}
    \sum_{t=1}^T \left( V^{\pi}_{1,M(\pi^t, Q^t)}(s_1) - V_{1,M(\pi^t, Q^t)}^{\pi^t}(s_1) \right) \leq 4H b \sqrt{T \ln \vol(\gA)}.
\end{align*}
\label{lemma: bounding value difference under soft policy update}
\end{lemma}
\begin{proof}[Proof of \Cref{lemma: bounding value difference under soft policy update}] The proof for this lemma is quite standard as shown in \cite{zanette2021provable}. We rewrote the proof with a slight modification for completeness. 
For simplicity, we write $M_t := M(\pi^t, Q^t)$. We will see that the key property that enables this lemma is that $Q^t_h = Q^{\pi^t}_{h,M_t}$ (\Cref{lemma: exact value function under the induced MDP}), which allows us to relate the value difference lemma to the log policy ratio. Using the value difference lemma (\Cref{lemma: performance difference}), we have 
\begin{align*}
     V^{\pi}_{1,M_t}(s_1) - V_{1,M_t}^{\pi^t}(s_1)  = \sum_{h=1}^H \sE_{\pi} A^{\pi^t}_{h, M_t}(s_h, a_h),
\end{align*}
where $\sE_{\pi}$ is the expectation over the random trajectory $(s_1, a_1, \ldots, s_H, a_H)$ generated by $\pi$ (and the underlying MDP $M_t$\footnote{Note that $\Pr((s_1,a_1, \ldots, s_H, a_H) | \pi, M) = \Pr((s_1,a_1, \ldots, s_H, a_H) | \pi, M_t)$ since $M_t$ and $M$ have identical transition kernels.}). For any $V: \gS \rightarrow \sR$, it follows from the definition of $\{\pi^t\}$ update that we have 
\begin{align*}
    \log \frac{\pi^{t+1}_h(a|s)}{ \pi^{t}_h(a|s)} &= \eta Q^t_h(s,a) -  \log \left( \sE_{a \sim \pi_h^t(\cdot|s)} \left[\exp \left( \eta Q^t_h(s,a) \right) \right] \right) \\ 
    &= \eta (Q^t_h(s,a) - V(s))-  \log \left( \sE_{a \sim \pi_h^t(\cdot|s)} \left[\exp \left( \eta (Q^t_h(s,a) - V(s)) \right) \right] \right). 
\end{align*}
In the equation above, noting that $Q^t_h = Q^{\pi^t}_{h,M_t}$ (\Cref{lemma: exact value function under the induced MDP}) and replacing $V(s)$ by $V^{\pi^t}_{h,M_t}$, we have
\begin{align}
    \log \frac{\pi^{t+1}_h(a|s)}{ \pi^{t}_h(a|s)} = \eta A^{\pi^t}_{h, M_t}(s,a) - \log \left( \sE_{a \sim \pi_h^t(\cdot|s)} \left[\exp \left( \eta  A^{\pi^t}_{h, M_t}(s,a) \right) \right] \right),
    \label{eq: log of policy ratio to advantage function and log E exp}
\end{align}
where we define the advantage function $A^{\pi}_M = \{A^{\pi}_{h,M}\}_{h \in [H]}$ as 
\begin{align*}
    A^{\pi}_{h,M}(s,a) := Q^{\pi}_{h,M}(s,a) - V^{\pi}_{h,M}(s), \forall (s,a,h).
\end{align*}
Note that $|A^{\pi^t}_{h, M_t}(s,a)| \leq 2 b$. By choosing $\eta \in (0, 1/(2 b))$, we have 
\begin{align}
    &\log \left( \sE_{a \sim \pi_h^t(\cdot|s)} \left[\exp \left( \eta  A^{\pi^t}_{h, M_t}(s,a) \right) \right] \right) \nonumber \\ 
    &\leq \log \left( \sE_{a \sim \pi_h^t(\cdot|s)} \left[1 + \eta  A^{\pi^t}_{h, M_t}(s,a) +  (e-2) \eta^2  A^{\pi^t}_{h, M_t}(s,a)^2\right] \right) \nonumber\\
    &= \log \left( \sE_{a \sim \pi_h^t(\cdot|s)} \left[1 +  (e-2) \eta^2  A^{\pi^t}_{h, M_t}(s,a)^2\right] \right) \nonumber\\ 
    &\leq \log(1 + (e-2) \eta^2 4 b^2) \nonumber \\ 
    &\leq 4(e-2) b^2 \eta^2 
    \label{eq: relate log E exp to quadratic of lr}
\end{align}
where the first inequality uses that $e^x \leq 1 + x + (e-2) x^2, \forall x \leq 1$ and $| \eta  A^{\pi^t}_{h, M_t}(s,a)| \leq 1, \forall (s,a)$, the first equality uses that $\sE_{a \sim \pi_h^t(\cdot|s)} \left[ A^{\pi^t}_{h, M_t}(s,a) \right] = 0$, the second inequality uses that $|A^{\pi^t}_{h, M_t}(s,a)| \leq 2 b$, and the last inequality uses that $\log(1+x) \leq x, \forall x \geq 0$. Combining \Cref{eq: relate log E exp to quadratic of lr} and \Cref{eq: log of policy ratio to advantage function and log E exp}, we have 
\begin{align*}
    A^{\pi^t}_{h,M_t}(s,a) \leq \frac{1}
    {\eta} \log \frac{\pi^{t+1}_h(a|s)}{\pi^t_h(a|s)} + 4(e-2) b^2 \eta.
\end{align*}
Thus, for any $h \in [H]$, we have 
\begin{align*}
    &\sum_{t=1}^T \sE_{\pi} A^{\pi^t}_{h, M_t}(s_h, a_h) \\
    &\leq \frac{1}{\eta} \sum_{t=1}^T \left( \sE_{\pi} \left[ KL[\pi_h(\cdot|s_h) \| \pi_h^t(\cdot|s_h)] \right] - \sE_{\pi} \left[ KL[\pi_h(\cdot|s_h) \| \pi_h^{t+1}(\cdot|s_h) ] \right] \right) + 4(e-2) T b^2 \eta \\ 
    &= \frac{1}{\eta} \left(\sE_{\pi} \left[ KL[\pi_h(\cdot|s_h) \| \pi_h^1(\cdot|s_h)] \right] -  \sE_{\pi} \left[ KL[\pi_h(\cdot|s_h) \| \pi_h^{T+1}(\cdot|s_h)] \right] \right) + 4(e-2) T b^2 \eta \\ 
    &\leq \frac{1}{\eta} \sE_{\pi} \left[ KL[\pi_h(\cdot|s_h) \| \pi_h^1(\cdot|s_h)] \right]  + 4(e-2) T b^2 \eta \\
    &\leq \frac{1}{\eta} \log(\vol(\gA)) +  4(e-2) T b^2 \eta,
\end{align*}
where the second inequality uses the non-negativity of KL divergence, and the last inequality uses that $ KL[\pi_h(\cdot|s_h) \| \pi_h^1(\cdot|s_h)] = -H[\pi_h(\cdot|s_h)] +  \log( \vol(\gA)) \leq \log( \vol(\gA))$ where $\pi^1$ is uniform over $\gA$ and $\vol(\gA)$ denotes the volume over the compact set $\gA$. Combining all pieces together, we have 
\begin{align*}
    \sum_{t=1}^T \sE_{\pi} A^{\pi^t}_{h, M_t}(s_h, a_h) \leq \frac{1}{\eta} H \log(\vol(\gA)) +  4(e-2) T H b^2 \eta. 
\end{align*}
Minimizing the RHS of the above equation with respect to $\eta$ yields $\eta = \sqrt{\frac{\log \vol(\gA)}{4(e-2) b^2 T}}$ and 
\begin{align*}
     \sum_{t=1}^T \sE_{\pi} A^{\pi^t}_{h, M_t}(s_h, a_h) \leq 2 H \sqrt{4(e-2) b^2 T \log \vol(\gA)} \leq 4H b \sqrt{T \log \vol(\gA)}. 
\end{align*}
Finally, we need that 
\begin{align*}
    \eta = \sqrt{\frac{\log \vol(\gA)}{4(e-2) b^2 T}} \leq \frac{1}{2b},
\end{align*}
which implies $T \geq \frac{\ln \vol(\gA)}{(e-2)}$.
\end{proof}

We are now ready to establish the proofs of our three main theorems. 
\section{Proof of \Cref{theorem: gopo-vsc}}
To construct our proof for \Cref{theorem: gopo-vsc}, we first establish two following lemmas. The first lemma, \Cref{lemma: equivalent of elliptical potential} establishes that the in-distribution squared Bellman residuals are bounded by the unbiased proxy of the squared Bellman error $\gL_{\tpi}(f)$, up to some estimation and approximation errors. The second lemma, \Cref{lemma: bound the unregularized loss}, asserts that the unbiased proxy of the squared Bellman error at the projection of $Q^{\tpi}$ is close to zero, up to some estimation and approximation errors. 
\begin{lemma}
For any $\delta > 0, \eps > 0$ and any $T \in \sN$, under \Cref{assumption: approximate Bellman completeness}, with probability at least $1 - \delta$, it holds uniformly over all $f \in \gF$ and $\tilde{\pi} \in \Pi^{soft}(T)$ that
\begin{align*}
    \sum_{h=1}^H \sum_{k = 1}^K \sE_{k} \left[ \gE^{\tilde{\pi}}_h(f_h, f_{h+1})(s_h,a_h)^2 \right] &\leq 2 \gL_{\tpi}(f) + 40 b(b+ 2) K H \eps + 12 b K \sum_{h=1}^H \nu_h \\
     &+ 144 (e-2) b^2 H \left[ d_{\gF}(\eps) + d_{\Pi}(\eps, T) + \ln(1/\delta) \right],
\end{align*}
where  $\gL_{\tpi}(f) :=  \sum_{h=1}^H   \hat{L}_{\tpi}(f_h,f_{h+1}) - \inf_{g \in \gF} \sum_{h=1}^H   \hat{L}_{\tpi}(g_h,f_{h+1})$. 
\label{lemma: equivalent of elliptical potential}
\end{lemma}

\begin{lemma}
Under \Cref{assumption: realizability}, for any $T \in \sN$, with probability at least $1 - \delta$, it holds uniformly for any $\tpi \in \Pi^{soft}(T)$ that
\begin{align*}
    \gL_{\tpi}(\proj_{\gF}(Q^{\tpi})) &\leq 36 (e-2) b^2 H \left( 2 d_{\gF}(\eps)  + d_{\Pi}(\eps, T) + \ln \frac{H}{\delta} \right)  + 6 b (3b + 4) \eps K H \\
    &+ 15 b K  \sum_{h=1}^H \xi_h,
\end{align*}
where $\proj_{\gF}(Q^{\tpi})$ is the projection of $Q^{\tpi}$ onto $\gF$, formally defined in \Cref{defn: projection of value function}. 
\label{lemma: bound the unregularized loss}
\end{lemma}

\subsection{Proof of \Cref{theorem: gopo-vsc}}
With the two lemmas above, we are ready to prove \Cref{theorem: gopo-vsc}. This proof is also laying a foundational step for our proofs of \Cref{theorem: gopo-roc} and \Cref{theorem: gopo-psc} that we shall present shortly. The proofs for the two lemmas above are presented immediately after the proof of \Cref{theorem: gopo-vsc}.
\begin{proof}[Proof of \Cref{theorem: gopo-vsc}.]
Using \Cref{lemma: error decomposition}, we have
\begin{align*}
\subopt_{\pi}^M(\pi^t) &= \sum_{h=1}^H \sE_{\pi} [\gE_h^{\pi^t}(\ubar{Q}^t_h, \ubar{Q}^t_{h+1})(s_h,a_h)] + \Delta_1 \ubar{Q}_1(s_1, \pi_1^t) + \subopt_{\pi}^{M_t}(\pi^t)
\end{align*}
where we denote 
\begin{align*}
    M_t &:= M(\ubar{Q}^t, \pi^t), \\ 
    \Delta_1 \ubar{Q}_1(s_1, \pi^t) &:= \ubar{Q}_1(s_1, \pi^t) - V_1^{\pi^t}(s_1). 
\end{align*}
\paragraph{Bounding $\sum_{t=1}^T \subopt_{\pi}^{M_t}(\pi^t)$.}
Note that $\sum_{t=1}^T \subopt_{\pi}^{M_t}(\pi^t)$ can be controlled by standard tools from online learning (\Cref{lemma: bounding value difference under soft policy update}); thus it remains to control the first $H+1$ terms.

\paragraph{Bounding $\Delta_1 \ubar{Q}_1(s_1, \pi_1^t)$.}
Due to \Cref{lemma: bound the unregularized loss}, the event that $\proj_{\gF}(Q^{\pi^t}) \in \gF(\beta; \pi^t)$ holds occur at probability at least $1- \delta$. Furthermore, under this event, we have 
\begin{align*}
    \Delta_1 \ubar{Q}_1(s_1, \pi_1^t) &= \ubar{Q}_1(s_1, \pi^t) - V_1^{\pi^t}(s_1) \\
    &\leq \proj_{\gF_1}(Q_1^{\pi^t}) - V_1^{\pi^t}(s_1) \\ 
    &\leq \xi_1,
\end{align*}
where the first inequality exploits \Cref{VSC: constrained optimization} of \Cref{algorithm: vsc}, and the last inequality uses \Cref{assumption: realizability}.

\paragraph{Bounding $ \sum_{h=1}^H \sE_{\pi} [\gE_h^{\pi^t}(\ubar{Q}^t_h, \ubar{Q}^t_{h+1})(s_h,a_h)]$.}
It follows from \Cref{lemma: decoupling argument} that  
\begin{align*}
     \sum_{h=1}^H \sE_{\pi} [\gE_h^{\pi^t}(\ubar{Q}^t_h, \ubar{Q}^t_{h+1})(s_h,a_h)] &\leq \sqrt{H \gC(\pi; \eps) \sum_{h=1}^H \left(\sE_{\mu} [\gE_h^{\tilde{\pi}}(\ubar{Q}_{h}^t, \ubar{Q}^t_{h+1})(s_h,a_h)^2] + \nu_h^2 + 4 b \nu_h\right)} \\
     &+ H \eps + \bar{\nu}.
\end{align*}
The term $\sum_{h=1}^H \sE_{\mu} [\gE_h^{\tilde{\pi}}(\ubar{Q}_{h}^t, \ubar{Q}^t_{h+1})(s_h,a_h)^2]$ is bounded by \Cref{lemma: equivalent of elliptical potential}, with notice that $\gL_{\pi^t}(\ubar{Q}^t) \leq \beta $ (due to the definition of $\gF(\beta; \pi^t)$ in \Cref{algorithm: vsc}).

Combining the three steps above via the union bound completes our proof.
\end{proof}

We now prove the two support lemmas. 
\subsection{Proof of \Cref{lemma: equivalent of elliptical potential}}
\begin{proof}[Proof of \Cref{lemma: equivalent of elliptical potential}]
Let us consider any fixed $f \in \gF$ and any $\pi \in \Pi^{all}$. By \Cref{lemma: the squared residuals bound the expected value and the variance of the empirical minimax error}, we have
\begin{align*}
    \sE_{k,h}[\Delta L_{\pi}(f_h, f_{h+1}; z^k_h) ] &= \gE^{\pi}_h(f_h, f_{h+1})(s^k_h,a^k_h)^2, \\
     \sE_{k,h}[\Delta L_{\pi}(f_h, f_{h+1}; z^k_h)^2 ] &\leq 36b^2 \gE^{\pi}_h(f_h, f_{h+1})(s^k_h,a^k_h)^2.
\end{align*}
Combining with \Cref{lemma: Freedman inequality}, we have that with probability at least $1 - \delta$, for any $\iota \in [0,\frac{1}{13 b^2}]$,
\begin{align*}
&\sum_{k = 1}^K \sE_{k} [\gE_h^{\pi}(f_h,f_{h+1})(s_h,a_h)^2]  - \sum_{k=1}^K \Delta L_{\pi}(f_h, f_{h+1}; z^k_h)  \\ 
&\leq 36 (e - 2)  b^2 \iota    \sum_{k = 1}^K \sE_{k}[ \gE_h^{\pi}(f_h,f_{h+1})(s_h,a_h)^2]  + (1/\iota)\log(1/\delta). 
\end{align*}
By setting $\iota = \frac{1}{72 (e-2) b^2}$, the above inequality becomes 
\begin{align*}
    \sum_{k = 1}^K \sE_{k} [\gE_h^{\pi}(f_h,f_{h+1})(s_h,a_h)^2] \leq 2  \sum_{k=1}^K \Delta L_{\pi}(f_h, f_{h+1}; z^k_h) + 144 (e-2) b^2 \ln(1/\delta). 
\end{align*}

For any $\eps > 0$, and for any $f \in \gF, \pi \in \Pi^{soft}(T)$, by definition of $\eps$-covering,  there exist $f'$ and $\pi'$ in the $\eps$-cover of $\gF$ and $\Pi^{soft}(T)$, i.e., 
\begin{align*}
    \|f_h - f'_h \|_{\infty} \leq \eps, \| \pi_h - \pi'_h\|_{1,\infty} \leq \eps. 
\end{align*}
By simple calculations, we have 
\begin{align*}
    |\gE_h^{\pi}(f_h,f_{h+1})(s_h,a_h)^2 - \gE_h^{\pi'}(f'_h,f'_{h+1})(s_h,a_h)^2| &\leq 4b (b+2) \eps, \\ 
    |\Delta L_{\pi}(f_h, f_{h+1}; z^k_h) - \Delta L_{\pi'}(f'_h, f'_{h+1};z^k_h)| &\leq 18 b (b+2) \eps. 
\end{align*}

Thus, by the union bound, we have with probability at least $1 - \delta$, it holds uniformly over all $f \in \gF, \pi \in \Pi^{soft}(T)$ that 
\begin{align*}
    &\sum_{h=1}^H \sum_{k = 1}^K \sE_{k} [\gE_h^{\pi}(f_h,f_{h+1})(s_h,a_h)^2] \leq 2  \sum_{h=1}^H \sum_{k=1}^K \Delta L_{\pi}(f_h, f_{h+1}; z^k_h) + 40 b (b+2) K H \eps \\
    &+ 144 (e-2) b^2 \sum_{h=1}^H  \ln(N(\eps; \gF_h, \| \cdot\|_{\infty}) N(\eps; \Pi^{soft}_h(T), \| \cdot\|_{1,\infty})/\delta). 
\end{align*}

Finally, notice that 
\begin{align*}
    | l_{\pi}(\sT_h^{\pi} f_{h+1}, f_{h+1}; z_h)  - l_{\pi}(\proj_{\gF_h}(\sT_h^{\pi} f_{h+1}), f_{h+1}; z_h) | \leq 6b \nu_h. 
\end{align*}
Thus, we have
\begin{align*}
    \sum_{h=1}^H \sum_{k=1}^K \Delta L_{\pi}(f_h, f_{h+1}; z^k_h) \leq \gL_{\pi}(f) + 6 b K \sum_{h=1}^H \nu_h.
\end{align*}
We can then conclude our proof. 
\end{proof}

\subsection{Proof of \Cref{lemma: bound the unregularized loss}}
In order to prove \Cref{lemma: bound the unregularized loss}, we shall first prove the following lemma, which establishes the confidence radius of the empirical squared Bellman errors that we used to establish the version space in \Cref{algorithm: vsc}.

\begin{lemma}
Consider any $\delta > 0, \eps > 0, T \in \sN$, let 
\begin{align*}
    \beta_{\eps} &:= 36 (e-2) b^2  \left( 2 d_{\gF}(\eps)  + d_{\Pi}(\eps, T) + \ln(H/\delta) \right)  + 6 b (3b + 4) \eps K . 
\end{align*}
With probability at least $1 - \delta$, it holds uniformly over any $\pi \in \Pi^{soft}(T)$, $f \in \gF$, and $h \in [H]$ that
\begin{align*}
    \sum_{k=1}^K \left(\sT^{\pi}_h f_{h+1} (x^k_h) - r^k_h - f_{h+1}(s^k_{h+1}, \pi_{h+1}) \right)^2 \leq \inf_{g_h \in \gF_h}\sum_{k=1}^K \left(g_h(x^k_h) - r^k_h - f_{h+1}(s^k_{h+1}, \pi_{h+1}) \right)^2 + \beta_{\eps}. 
\end{align*}
\label{lemma: pessimism side of the confidence set}
\end{lemma}

\begin{proof}[Proof of \Cref{lemma: pessimism side of the confidence set}]
Let us fix any $h \in [H]$. For any $(f,g,\pi) \in \gF \times \gF \times \Pi^{soft}(T)$ and any $k \in [K]$, define the following random variable 
\begin{align*}
    Z_{k,h}(f,g,\pi) := \left(g_h(x^k_h) - r^k_h - f_{h+1}(s^k_{h+1}, \pi_{h+1}) \right)^2 -  \left(\sT^{\pi}_h f_{h+1}(x^k_h) - r^k_h - f_{h+1}(s^k_{h+1}, \pi_{h+1}) \right)^2.
\end{align*}
Denote 
\begin{align*}
    \sE_{k,h}[\cdot] := \sE \left[ \cdot \bigg | \{z^i_h\}_{h \in [H]}^{i \in [k-1]}, s_1^k,a_1^k,r_1^k, \ldots, s^k_{h-1}, a^k_{h-1}, r^k_{h-1}, s^k_h,a^k_h \right].
\end{align*}
By \Cref{lemma: the squared residuals bound the expected value and the variance of the empirical minimax error}, we have
\begin{align*}
    \sE_{k,h} \left[ Z_{k,h}(f,g,\pi) \right] &=  \gE_h^{\pi}(g_h, f_{h+1})(s^k_h, a^k_h)^2 ,  \\
     \sE_{k,h} \left[ Z^2_{k,h}(f,g,\pi) \right] &\leq 36 b^2 \gE_h^{\pi}(g_h, f_{h+1})(s^k_h, a^k_h)^2. 
\end{align*}
Thus, combing with \Cref{lemma: Freedman inequality}, for any $(f,g,\pi) \in \gF \times \gF \times \Pi^{soft}(T)$, with probability at least $1 - \delta$, for any $\iota \in [0,\frac{1}{13 b^2}]$,
\begin{align*}
    \sum_{k=1}^K \sE_{k,h} \left[ Z_{k,h}(f,g,\pi) \right] -  \sum_{k=1}^K Z_{k,h}(f,g,\pi) \leq 36 (e - 2) b^2 \iota   \sum_{k=1}^K  \gE_h^{\pi}(g_h, f_{h+1})(s^k_h, a^k_h)^2  + \frac{\ln(1/\delta)}{\iota}. 
\end{align*}
By setting $\iota = 1 / (36 (e-2) b^2)$, the above inequality becomes
\begin{align}
    -  \sum_{k=1}^K Z_{k,h}(f,g,\pi) \leq 36 (e-2) b^2 \ln(1/\delta). 
    \label{eq: bounding least-square difference of target functions}
\end{align}

For any $\eps > 0$, let $\gF^{\eps}$ and $\Pi^{\eps}$ be $\eps$-covers of $\gF$ and $\Pi^{soft}(T)$, respectively, with respect to $\|\cdot\|_{\infty}$ and $\| \cdot \|_{\infty, 1}$, respectively, where $\| u - v \|_{\infty} := \sup_{(s,a)}|u(s,a) - v(s,a) |$ and $\| \pi - \pi' \|_{\infty, 1} := \sup_{s} \sum_{a \in \gA} | \pi(a|s) - \pi'(a|s)|$. Using the union bound, it follows from \Cref{eq: bounding least-square difference of target functions} that with probability at least $1 - \delta$, it holds uniformly over any $h \in [H]$ and any  $(f,g,\pi) \in  \gF^{\eps} \times \gF^{\eps} \times \Pi^{\eps}$ that 
\begin{align*}
     -  \sum_{k=1}^K Z_{k,h}(f,g,\pi) \leq 18(e-2) b^2 \left[ \ln(H/\delta) + 2 d_{\gF}(\eps)+ d_{\Pi}(\eps, T) \right]. 
\end{align*}

For any $(f,g,\pi) \in  \gF \times \gF \times \Pi^{soft}(T)$, there exist $(f_{\eps}, g_{\eps}, \pi_{\eps}) \in \gF^{\eps} \times \gF^{\eps} \times \Pi^{\eps}$ such that 
\begin{align*}
    \|f_h - (f_{\eps})_h \|_{\infty} \leq \eps, \| g_h - (g_{\eps})_h\|_{\infty} \leq \eps, \|\pi_h - (\pi_{\eps})_h \|_{\infty, 1} \leq \eps, \forall h \in [H].
\end{align*}
It is easy to compute the discretization error that 
\begin{align*}
    Z_{k,h}(f,g,\pi) - Z_{k,h}(f_{\eps}, g_{\eps}, \pi_{\eps}) \leq 18 b(b+1) \eps. 
\end{align*}

Using the discretization argument and the union bound complete our proof.
\end{proof}

We are now ready to prove \Cref{lemma: bound the unregularized loss}.
\begin{proof}[Proof of \Cref{lemma: bound the unregularized loss}]
Consider the event that the inequality in \Cref{lemma: pessimism side of the confidence set} holds. Under this event, for any $\tpi \in \Pi^{soft}(T)$, we have
\begin{align*}
    &\sum_{k=1}^K l_{\tpi}(\proj_{\gF_h} (Q_h^{\tpi}), \proj_{\gF_{h+1}} (Q_{h+1}^{\tpi}); z^k_h) \leq  \sum_{k=1}^K l_{\tpi}(Q_h^{\pi^t}, Q_{h+1}^{\tpi}; z^k_h) + 6b K \xi_h \\
    &= \sum_{k=1}^K  l_{\tpi}(\sT^{\tpi}_h Q_{h+1}^{\tpi}, Q_{h+1}^{\tpi}; z^k_h) + 6b K \xi_h  \\ 
    &\leq  \sum_{k=1}^K l_{\tpi}(\sT^{\tpi}_h\proj_{\gF_{h+1}} (Q_{h+1}^{\tpi}), \proj_{\gF_{h+1}} (Q_{h+1}^{\tpi}); z^k_h) + 12b K \xi_h  \\ 
    &\leq \sum_{k=1}^K l_{\tpi}(g_h, \proj_{\gF_{h+1}}(Q^{\tpi}_{h+1}); z^k_h) + \beta_{\eps} + 12 b K \xi_h ~~~ \text{(for any $g_h \in \gF_h$)}\\ 
    &\leq \sum_{k=1}^K l_{\tpi}(g_h, Q^{\tpi}_{h+1}; z^k_h) + \beta_{\eps} + 15b K \xi_h,
\end{align*}
where we use \Cref{assumption: realizability} for the first, second, and last inequalities, the third inequality uses \Cref{lemma: pessimism side of the confidence set}, and the equality uses $Q_{h+1}^{\tpi} = \sT^{\tpi}_h Q_{h+1}^{\tpi}$. Rearranging the last inequality completes our proof.
\end{proof}

\section{Proof of \Cref{theorem: gopo-roc}}
\label{section: proof of reg theorem}
In this appendix, we present our complete argument to establish \Cref{theorem: gopo-roc}. In order to prove \Cref{theorem: gopo-roc}, the key is to establish a connection from the squared Bellman error under the data distribution $\mu$ to the regularized objective in \Cref{algorithm: roc}. This key idea should become clear in the following proof.  
\begin{proof}[Proof of \Cref{theorem: gopo-roc}]

Similar to the proof of \Cref{theorem: gopo-vsc}, our starting point is
using \Cref{lemma: error decomposition}:
\begin{align*}
\subopt_{\pi}^M(\pi^t) &= \sum_{h=1}^H \sE_{\pi} [\gE_h^{\pi^t}(\ubar{Q}^t_h, \ubar{Q}^t_{h+1})(s_h,a_h)] + \Delta_1 \ubar{Q}_1(s_1, \pi_1^t) + \subopt_{\pi}^{M_t}(\pi^t)
\end{align*}
and we bound $\sum_{t=1}^T \subopt_{\pi}^{M_t}(\pi^t)$ using \Cref{lemma: bounding value difference under soft policy update}. We now bound the remaining terms. 

For any $\gamma >0$, we have
\begin{align*}
    &\sum_{h=1}^H \sE_{\pi}[\gE_h^{\pi^t}(\ubar{Q}^t_h, \ubar{Q}_{h+1}^t)(s_h,a_h)] \\
      &\leq \frac{K}{2 \lambda}\sum_{h=1}^H \left(\sE_{\mu} [\gE_h^{\pi^t}(\ubar{Q}_h^t, \ubar{Q}^t_{h+1})(s_h,a_h)^2] + \nu_h^2 + 4 b \nu_h\right) + \frac{\lambda H \gC({\pi}, \eps)}{2 K} + H \eps + \bar{\nu} \\ 
      &\leq \frac{\gL_{\pi^t}(\ubar{Q}) + 0.5 \iota_1 + 0.5 K \bar{\nu}^2 + 2 b K \bar{\nu} }{\lambda} + \frac{\lambda H\gC({\pi}, \eps)}{2 K} + H \eps + \bar{\nu},
\end{align*}
where the first inequality uses \Cref{lemma: decoupling argument} and the second inequality uses \Cref{lemma: equivalent of elliptical potential}, and here $\iota_1 := 40 b(b+ 2) K H \eps + 12 b K \sum_{h=1}^H \nu_h + 144 (e-2) b^2 H \left[ d_{\gF}(\eps) + d_{\Pi}(\eps,T) + \ln(1/\delta) \right]$.
Thus, we have 
\begin{align*}
    &\sum_{h=1}^H \sE_{\pi} [\gE_h^{\pi^t}(\ubar{Q}^t_h, \ubar{Q}^t_{h+1})(s_h,a_h)] + \Delta_1 \ubar{Q}_1(s_1, \pi_1^t) \\ 
    &\leq \frac{\gL_{\pi^t}(\ubar{Q}) + \lambda \Delta_1 \ubar{Q}_1(s_1, \pi_1^t) + 0.5 \iota_1 + 0.5 K \sum_{h=1}^H \nu_h^2 + 2 b K \sum_{h=1}^H \nu_h }{\lambda} + \frac{\lambda H \gC({\pi}, \eps)}{2 K} \\
    &+ H \eps + \sum_{h=1}^H \nu_h \\ 
    &\leq \frac{\gL_{\pi^t}(\proj_{\gF}(Q^{\pi^t})) + \lambda \Delta_1 \proj_{\gF_1} (Q^{\pi^t}_1)(s_1, \pi_1^t) + 0.5 \iota_1 + \sum_{h=1}^H \nu_h^2 + 2 b K \sum_{h=1}^H \nu_h }{\lambda} \\
    &+ \frac{\lambda H \gC({\pi}, \eps)}{2 K} + H \eps + \sum_{h=1}^H \nu_h \\ 
    &\leq \frac{\iota_2 + \lambda \xi_1 + 0.5 \iota_1 + \sum_{h=1}^H \nu_h^2 + 2 b K \sum_{h=1}^H \nu_h  }{\lambda} + \frac{\lambda H \gC({\pi}, \eps)}{2 K} + H \eps + \sum_{h=1}^H \nu_h,
\end{align*}
where the second inequality uses the fact that $\ubar{Q}^t_h$ is a minimizer over $\gF \ni \proj_{\gF}(Q^{\pi^t})$ of $\gL_{\pi^t}(f) + \lambda f_1(s_1, \pi^t_1)$ (which has the same minimizer as $\gL_{\pi^t}(f) + \lambda \Delta_1 f_1(s_1, \pi^t_1)$), and the last inequality uses \Cref{lemma: bound the unregularized loss}, and here we define $\iota_2 := 36 (e-2) b^2 H \left( 2 d_{\gF}(\eps)  + d_{\Pi}(\eps,T) + \ln \frac{H}{\delta} \right)  + 6 b (3b + 4) \eps K H + 15 b K  \sum_{h=1}^H \xi_h$. 
\end{proof}

\section{Proof of \Cref{theorem: gopo-psc}}
\label{section: all proof arguments for PPS-PO}
In this appendix, we give our complete proof for \Cref{theorem: gopo-psc}. In order to develop our argument for proving \Cref{theorem: gopo-psc}, we shall start with a generalized form of posterior sampling in \Cref{section: equivalent form} and develop our key support result in \Cref{theorem: expected pessimistic squared Bellman error}. We then use \Cref{theorem: expected pessimistic squared Bellman error} and the similar machinery developed in \Cref{section: proof of reg theorem} to complete our argument for proving \Cref{theorem: gopo-psc}.

\subsection{Generalized form of posterior and \Cref{theorem: expected pessimistic squared Bellman error}}
\label{section: equivalent form}
We start with recalling the posterior distribution defined in \Cref{line:psc posterior distribution} of \Cref{algorithm: psc} as 
\begin{align}
    \hat{p}(f| \gD, \pi) \propto \exp \left( - \lambda f_1(s_1, \pi_1) \right) p_0(f) \prod_{h \in [H]} \frac{\exp\left( -\gamma \hat{L}_{\pi}(f_h, f_{h+1})\right)}{\sE_{f'_h \sim p_{0,h}} \exp\left( -\gamma \hat{L}_{\pi}(f'_h, f_{h+1})\right)}. 
    \label{eq: canonical posterior sampling form}
\end{align}

Similar to the proof strategy in \cite{DBLP:conf/nips/DannMZZ21}, we now consider a slightly more general form of the posterior distribution with an extra parameter $\alpha \in [0,1]$ and in an equivalent but more useful form. In concrete, consider any $\alpha \in [0,1]$ and define the potential functions: 
\begin{align*}
    \widehat{\Phi}_h(f, \pi; \gD) &:= -\ln p_{0}(f_h) + \alpha \gamma \sum_{k=1}^K \Delta L_{\tilde{\pi}}(f_h, f_{h+1}; z^k_h) \nonumber\\
    &+ \alpha \ln \sE_{\tilde{f}_h \sim p_{0}} \exp \left( -\gamma \sum_{k=1}^K \Delta L_{\tilde{\pi}}(f_h, f_{h+1}; z^k_h) \right), \\ 
    \widehat{\Phi}(f,\pi; \gD) &:= \sum_{h=1}^H \widehat{\Phi}_h(f, \pi; \gD), \\ 
    \Delta_1 f_1(s_1, \pi) &:= f_1(s_1, \pi) - V_1^{\pi}(s_1). 
\end{align*}
where recall that $\Delta L_{\tilde{\pi}}(f_h, f_{h+1}; z^k_h)$ is defined in \Cref{table: notations and quantities}. Define the generalized posterior distribution 
\begin{align}
        \hat{p}(f| \gD, \pi) & \propto \exp \left( -\widehat{\Phi}(f, \pi; \gD) - \lambda \Delta f_1(s_1, \pi) \right),
        \label{eq: generalized posterior}
\end{align}
where it is equivalent to the posterior defined in \Cref{eq: canonical posterior sampling form} when $\alpha = 1$. We shall use \Cref{eq: generalized posterior} for the posterior for the rest of this section. We shall also define the complexity measure of this generalized posterior -- a counterpart to that of the canonical posterior form in \Cref{defn: complexity measure for canonical posterior}. 
\begin{defn}
Define
    \begin{align*}
        \kappa_h(\alpha, \eps, \tilde{\pi}) &:= (1 - \alpha) \ln \sE_{f_{h+1} \sim p_0} \left[ p_{0,h} \left( \gF_h^{\tilde{\pi}}(\eps; f_{h+1}) \right)^{-\alpha/(1 - \alpha)} \right],  
        \label{eq: define kappa_h}
    \end{align*}
    where recall that $\gF_h^{\tilde{\pi}}(\eps; f_{h+1}) =\{f' \in \gF_h: \sup_{s,a} |\gE_h^{\tilde{\pi}}(f',f_{h+1})(s,a)| \leq \eps\}$ which is defined in \Cref{defn: complexity measure for canonical posterior}. 
Define the complexity measure
\begin{align}
     d_0(\eps, \alpha) &:= \sup_{T \in \sN, \tilde{\pi} \in \Pi^{soft}(T)}\sum_{h=1}^H \kappa_h(\alpha, \eps, \tilde{\pi}). 
\end{align}
\label{defn: complexity measure with alpha}
\end{defn}

Note that we have
\begin{align*}
   \lim_{\alpha \rightarrow 1^{-}} d_0(\eps, \alpha) &= d_0(\eps). 
\end{align*}

We now state our key milestone result -- \Cref{theorem: expected pessimistic squared Bellman error} to support the argument for proving \Cref{theorem: gopo-psc}. The proof of \Cref{theorem: expected pessimistic squared Bellman error} is deferred to \Cref{section: proof of key support theorem for ps}.

\paragraph{Notation $\sE_{\tilde{\pi} \sim P_t(\cdot |\gD)}$.}
Note that in \Cref{algorithm: psc}, each policy $\pi^t$ for $t \in [T]$ is a random variable that depends on both the offline data $\gD$ and the randomization of sampling from the posteriors. That is, when conditioned on the offline data $\gD$, each $\pi_t$ is still a random variable. We denote $P_t(\cdot| \gD)$ as the posterior distribution of $\pi^t$ conditioned on $\gD$. Note that for any $\tilde{\pi} \sim P_t(\cdot|\gD)$ and any $t \in [T]$, we have $\tilde{\pi} \in \Pi^{soft}(T)$.

\begin{prop}
    For any $\gamma \in [0, \frac{1}{144 (e-2) b^2}]$, $\eps > 0$, $\delta > 0$, $\alpha \in (0,1]$, $T \in \sN$, and any $t \in [T]$ and $\lambda > 0$, we have,
    \begin{align*}
     &\sE_{\gD} \sE_{\tilde{\pi} \sim P_t(\cdot| \gD)} \sE_{f \sim \hat{p}(\cdot|\gD, \tilde{\pi})} \left[0.125 \alpha \gamma K \sum_{h=1}^H \sE_{\mu}[\gE^{\tilde{\pi}}_h(f_h, f_{h+1})(s_h,a_h)^2] + \lambda \Delta f_1(s_1, \tilde{\pi}) \right] \\
     &\lesssim \lambda \eps + \alpha \gamma H b^2 \cdot \max\{ d_{\gF}(\eps), d_{\Pi}(\eps, T), \ln \frac{\ln K b^2}{\delta} \} + \alpha \gamma b^2 K H \cdot \max\{ \eps, \delta \} + \gamma H K \frac{\eps^2}{ \alpha}\\ 
     &+ \sum_{h=1}^H \sup_{\tilde{\pi}_h \in \Pi_h^{soft}(T)} \kappa_h(\alpha, \eps, \tilde{\pi}_h) + \sup_{\tilde{\pi} \in \Pi^{soft}(T)} \sum_{h=1}^H \ln \frac{1}{p_0(\gF_h(\eps; Q^{\tilde{\pi}_h}_h))}. 
\end{align*}
    \label{theorem: expected pessimistic squared Bellman error}
\end{prop}

We now have all main components needed to construct our argument for proving \Cref{theorem: gopo-psc}. 

\subsection{Proof of \Cref{theorem: gopo-psc}}
\label{section: proof of ps theorem}
\begin{proof}[Proof of \Cref{theorem: gopo-psc}]
We start with the error decomposition argument. 
\paragraph{Step 1: Error decomposition.}
Similar to the first step of the proof of \Cref{theorem: gopo-roc}, using \Cref{lemma: error decomposition}, we have
\begin{align*}
\subopt_{\pi}^M(\pi^t) &= \sum_{h=1}^H \sE_{\pi} [\gE_h^{\pi^t}(\ubar{Q}^t_h, \ubar{Q}^t_{h+1})(s_h,a_h)] + \Delta_1 \ubar{Q}_1(s_1, \pi_1^t) + \subopt_{\pi}^{M_t}(\pi^t)
\end{align*}
where we denote $M_t := M(\ubar{Q}^t, \pi^t)$ and $\Delta_1 \ubar{Q}_1(s_1, \pi^t) := \ubar{Q}_1(s_1, \pi^t) - V_1^{\pi^t}(s_1)$. Since term $\sum_{t=1}^T \subopt_{\pi}^{M_t}(\pi^t)$ can be controlled \Cref{lemma: bounding value difference under soft policy update}, it remains to control 
\begin{align*}
    J &:= \sE_{\gD} \left[ \sum_{h=1}^H \sE_{\pi} [\gE_h^{\pi^t}(\ubar{Q}^t_h, \ubar{Q}^t_{h+1})(s_h,a_h)] + \Delta_1 \ubar{Q}_1(s_1, \pi_1^t) \right] \\ 
    &= \sE_{\gD} \sE_{\tilde{\pi} \sim P_t(\cdot| \gD)} \sE_{f \sim \hat{p}(\cdot|\tilde{\pi}, \gD)} \left[ \sum_{h=1}^H \sE_{\pi} [\gE_h^{\tilde{\pi}}(f_h, f_{h+1})(s_h,a_h)] + \Delta_1 f_1(s_1, \tilde{\pi}_1) \right]. 
\end{align*}

\paragraph{Step 2: Decoupling argument.}
Using \Cref{lemma: decoupling argument}, we have
\begin{align*}
&\sum_{h=1}^H \sE_{\pi} [\gE_h^{\tilde{\pi}}(f_h, f_{h+1})(s_h,a_h)] + \Delta_1 f_1(s_1, \tilde{\pi}_1) \\ 
    &\leq \frac{0.125 K \gamma}{ \lambda } \sum_{h=1}^H \left(\sE_{\mu}[\gE^{\tilde{\pi}}_h(f_h, f_{h+1})(s_h,a_h)^2] + \nu_h^2 + 4 b \nu_h \right) + \frac{0.5 \lambda H \gC(\pi, \eps_c)}{K \gamma} + \Delta_1 f_1(s_1, \tilde{\pi}_1) \\
    &+ H \eps_c + \sum_{h=1}^H \nu_h \\ 
    &= \frac{0.125 K \gamma \sum_{h=1}^H \sE_{\mu}[\gE^{\tilde{\pi}}_h(f_h, f_{h+1})(s_h,a_h)^2] + \lambda \Delta_1 f_1(s_1, \tilde{\pi}_1) +  \iota_1 }{\lambda} + \frac{0.5 \lambda H \gC(\pi, \eps_c)}{K \gamma} \\ 
    &+ H \eps_c + \sum_{h=1}^H \nu_h 
\end{align*}
where $\iota_1 := 0.125 K \gamma \left( \sum_{h=1}^H \nu_h^2 + 4 b \sum_{h=1}^H \nu_h \right)$. 

Applying \Cref{theorem: expected pessimistic squared Bellman error}, taking the limit $\alpha \rightarrow 1^{-}$, and re-organizing the terms complete our proof.

\end{proof}

It remains to prove  \Cref{theorem: expected pessimistic squared Bellman error}, which is the focus of the remaining appendix. 
\subsection{Proof of \Cref{theorem: expected pessimistic squared Bellman error}}
\label{section: proof of key support theorem for ps}

Our proof strategy for \Cref{theorem: expected pessimistic squared Bellman error} builds upon \cite{DBLP:conf/nips/DannMZZ21} where the central idea in the proof is to upper and lower bound the log-partition function -- which in our case is as follows: 
\begin{align}
    Z_t := \sE_{\gD} \sE_{\tilde{\pi} \sim P_t(\cdot| \gD)} \sE_{f \sim \hat{p}(\cdot|\gD, \tilde{\pi})} \left[ \widehat{\Phi}(f, \tilde{\pi}; \gD) + \lambda \Delta f_1(s_1, \tilde{\pi}) + \ln \hat{p}(f| \gD, \tilde{\pi}) \right], 
    \label{eq: log partition function}
\end{align}
for any $t \in [T]$ and any $T \in \sN$. 
The key technical distinction is that we need to handle the statistical dependence induced by $\sE_{\tilde{\pi} \sim P_t(\cdot| \gD)}$ -- which is absent in \cite{DBLP:conf/nips/DannMZZ21}. In concrete, when $\tilde{\pi}$ depends on $\gD$, then 
\begin{align*}
    \sE \Delta L_h^{\tilde{\pi}}(f_h, f_{h+1})(s_h^k, a^k_h) \neq \gE_h^{\tilde{\pi}}(f_h, f_{h+1})(s^k_h, a^k_h)^2,
\end{align*}
since $\tilde{\pi}$ depends on $(s_h^k, a^k_h)$. We develop an machinery to handle such issue in posterior sampling by carefully controlling the variance of the variable of interest (thus we can leverage the variance-dependent concentration inequality in \Cref{lemma: Freedman inequality}) and integrating it into posterior sampling using a uniform convergence argument. Roughly speaking, several milestone results during the process of developing our proof argument, we need to bound the form of 
\begin{align*}
     \sE_{\gD} \sE_{\tilde{\pi} \sim P_t(\cdot| \gD)} \sE_{f \sim \hat{p}(\cdot|\gD, \tilde{\pi})} \left[S(f, \tilde{\pi}, \gD) \right]
\end{align*}
where $S(f, \tilde{\pi}, \gD)$ is a function of $f, \tilde{\pi}, \gD$. It is useful to view $S(f, \tilde{\pi}, \gD)$ as a stochastic process indexed by $(f, \tilde{\pi})$. In our machinery, we shall first construct an upper bound on the variance of the random process, namely 
\begin{align*}
    V(f,\tilde{\pi}) \geq \sE_{\gD} [S(f, \tilde{\pi}, \gD)^2].
\end{align*}

Using a discretization argument, the union bound and \Cref{lemma: Freedman inequality}, we have with probability at least $1 - \delta$, for any $f \in \gF, \tilde{\pi} \in \Pi^{soft}(T)$, for any $t \in [0,\frac{1}{\sup  S(f, \tilde{\pi}, \gD)}]$, we have
\begin{align*}
    S(f, \tilde{\pi}, \gD) \leq O_K(1) + \sE_{\gD} \left[ S(f, \tilde{\pi}, \gD) \right] + (e - 2) t \sE_{\gD} [S(f, \tilde{\pi}, \gD)^2] + \frac{\ln(N/\delta)}{t}
\end{align*}
where $O_K(1)$ is a discretization error that can be controlled, and $N$ is a covering number of $\gF \times \Pi^{soft}(T)$. Note that $ S(f, \tilde{\pi}, \gD)$ often involves the squared loss which satisfies the Bernstein condition (see \Cref{lemma: variation condition for least squares}) -- thus we can roughly bound $\sE_{\gD} [S(f, \tilde{\pi}, \gD)^2] \leq \alpha |\sE_{\gD} \left[ S(f, \tilde{\pi}, \gD) \right]|$ for some constant $\alpha$. To integrate the high-probability bound into in-expected bound, we use the argument: 
\begin{align*}
     &\sE_{\gD} \sE_{\tilde{\pi} \sim P_t(\cdot| \gD)} \sE_{f \sim \hat{p}(\cdot|\gD, \tilde{\pi})} \left[S(f, \tilde{\pi}, \gD) \right] \leq O_K(1) + \sE_{\gD} \sE_{\tilde{\pi} \sim P_t(\cdot| \gD)} \sE_{\gD} \left[ S(f, \tilde{\pi}, \gD) \right] \\ 
     &+ (e - 2) t \sE_{\gD} \sE_{\tilde{\pi} \sim P_t(\cdot| \gD)} \sE_{\gD} [S(f, \tilde{\pi}, \gD)^2] + \frac{\ln(N/\delta)}{t} + \delta \sup  S(f, \tilde{\pi}, \gD). 
\end{align*}
\subsubsection{Lower-bounding log-partition function.}
In this appendix, we give a lower bound of the log-partition function defined in \Cref{eq: log partition function}. The final lower bound is presented in \Cref{prop: lower bound Z}. In order to establish such a lower bound, we first present a series of support lemmas that will culminate into \Cref{prop: lower bound Z}. 

The following lemma decomposes the log-partition function $Z$ into different terms that we shall control separately. 
\begin{lemma}
 For any $t \in [T]$ and any $T \in \sN$, we have
 \begin{align*}
    &Z_t \geq \underbrace{\sE_{\gD} \sE_{\tilde{\pi} \sim P_t(\cdot| \gD)} \sE_{f \sim \hat{p}(\cdot|\gD, \tilde{\pi})} \left[ \lambda \Delta f_1(s_1, \tilde{\pi}) + (1 - 0.5 \alpha) \ln \frac{\hat{p}(f_1| \gD, \tilde{\pi})}{p_{0}(f_1)} \right]}_{A_t} \\ 
    &+ 0.5 \alpha \sum_{h =1}^H  \underbrace{\sE_{\gD} \sE_{\tilde{\pi} \sim P_t(\cdot| \gD)}\sE_{f \sim \hat{p}(\cdot|\gD, \tilde{\pi})}  \left[ 2 \gamma  \sum_{k=1}^K \Delta L_{\tilde{\pi}}(f_h, f_{h+1}; z^k_h)  + \ln \frac{\hat{p}(f_h, f_{h+1} | \gD, \tilde{\pi})}{p_{0}(f_h, f_{h+1})}\right]}_{B_{h,t}} + \\ 
    &\sum_{h =1}^H \underbrace{ \sE_{\gD} \sE_{\tilde{\pi} \sim P_t(\cdot| \gD)} \sE_{f \sim \hat{p}(\cdot|\gD, \tilde{\pi})} \left[ \alpha  \ln \sE_{f'_h \sim p_0} \exp \left( - \gamma \sum_{k=1}^K \Delta L_{\tilde{\pi}}(f'_h, f_{h+1}; z^k_h) \right) + (1 - \alpha) \ln \frac{\hat{p}(f_{h+1}| \gD, \tilde{\pi})}{p_0(f_{h+1})}\right]}_{C_{h,t}}.
 \end{align*}
 \label{lemma: decompose Z}
\end{lemma}
\begin{proof}[Proof of \Cref{lemma: decompose Z}]
This is a simple adaptation of the decomposition in \citep[Lemma~6]{DBLP:conf/nips/DannMZZ21}.
\end{proof}

We now control each term of the above decomposition of $Z$ separately -- where a majority of these steps are where our technical arguments depart from those in \cite{DBLP:conf/nips/DannMZZ21}. In particular, \Cref{lemma: bound negative likelihood + log-partition function}, \Cref{lemma: bound term B}, and \Cref{lemma: empirical Bellman error to expected Bellman error} are our \emph{new} technical results. 
\paragraph{Bounding $A_t$.}
\begin{lemma}
We have 
\begin{align*}
    A_t \geq \lambda \sE_{\gD} \sE_{\tilde{\pi} \sim P_t(\cdot| \gD)} \sE_{f \sim \hat{p}(\cdot|\gD, \tilde{\pi})} \Delta f_1(s_1, \tilde{\pi}).
\end{align*}
\label{lemma: bound A}
\end{lemma}
\begin{proof}[Proof of \Cref{lemma: bound A}]
    It simply follows from that: 
    \begin{align*}
        \sE_{\gD} \sE_{\tilde{\pi} \sim P_t(\cdot| \gD)} \sE_{f \sim \hat{p}(\cdot|\gD, \tilde{\pi})} \left[ (1 - 0.5 \alpha) \ln \frac{\hat{p}(f_1| \gD, \tilde{\pi})}{p_{0}(f_1)} \right] = (1 - 0.5 \alpha) \KL[\hat{p}(\cdot|\gD, \tpi) \| p_0] \geq 0.
    \end{align*}
\end{proof}

\paragraph{Bounding $B_{h,t}$.}

\begin{lemma}
For any $f, \tilde{\pi}$, $0 \leq \gamma \leq \frac{1}{72 (e-2) b^2}$, and $h \in [H]$, we have 
\begin{align*}
    &\ln \sE_{(s_{h+1}, r_h) \sim P_h(\cdot|s_h, a_h)} \exp \left( - 2 \gamma  \Delta L_{\tilde{\pi}}(f_h, f_{h+1}; z_h) \right) \leq -2\gamma(1 - 72 (e-2)  \gamma  b^2) \gE_h^{\tilde{\pi}}(f_h, f_{h+1})(s_h,a_h)^2.
\end{align*}
\label{lemma: log E exp <= 0}
\end{lemma}
\begin{proof}[Proof of \Cref{lemma: log E exp <= 0}]
For simplicity, we write $\sE = \sE_{(s_{h+1}, r_h) \sim P_h(\cdot|s_h, a_h) }$. We have
\begin{align*}
    \ln \sE \exp \left( - 2 \gamma  \Delta L_{\tpi}(f_h, f_{h+1}; z_h) \right) &\leq \sE \exp \left( - 2 \gamma  \Delta L_{\tpi}(f_h, f_{h+1}; z_h) \right) - 1 \\ 
    &\leq -2 \gamma \sE \Delta L_{\tpi}(f_h, f_{h+1}; z_h) + (e-2) 4 \gamma^2  \sE \Delta L_{\tpi}(f_h, f_{h+1}; z_h)^2  \\ 
    &\leq  -2\gamma(1 - (e-2) 2 \gamma 36 b^2) \gE^2_h(f_h, f_{h+1}, \tilde{\pi})(s_h,a_h) 
\end{align*}
where the first inequality uses $\ln x \leq x - 1, \forall x \geq 0$, the second inequality uses $e^x \leq 1 + x + (e-2)x^2, \forall |x| \leq 1$ and $|2 \gamma \sE \Delta L_{\tpi}(f_h, f_{h+1}; z_h)| \leq 18 \gamma b^2 \leq 1$, the third inequality uses \Cref{lemma: the squared residuals bound the expected value and the variance of the empirical minimax error} and $\gamma \leq \frac{1}{72(e-2)b^2}$.
\end{proof}

\begin{lemma}
Define the random variable 
\begin{align*}
    \xi_h^{\tilde{\pi}}(f_h, f_{h+1}; z_h) &:= -2 \gamma \Delta L_{\tilde{\pi}}(f_h, f_{h+1}; z_h) - \ln \sE_{(s_{h+1}, r_h) \sim P_h(\cdot|s_h, a_h)} \exp \left( - 2 \gamma  \Delta L_{\tilde{\pi}}(f_h, f_{h+1}; z_h) \right).
\end{align*}
For any $\gamma \in [0, \frac{1}{144 (e-2) b^2}]$, $t \in [0, \frac{1}{26 \gamma b^2}]$, $\eps > 0$, $\delta > 0$, $T \in \sN$ with probability at least $1 - \delta$, it holds uniformly over all $\tilde{\pi} \in \Pi^{soft}(T), f_h \in \gF_h, f_{h+1} \in \gF_{h+1}$ that
\begin{align*}
    \sum_{k=1}^K \xi_h^{\tilde{\pi}}(f_h, f_{h+1}; z_h^k)  \leq  D + c \sum_{k=1}^K e_k^2, 
\end{align*}
where 
\begin{align}
\begin{cases}
    D &:= 120 \gamma b (b+2) K \eps +  \frac{2 d_{\gF}(\eps) + d_{\Pi}(\eps,T) + \ln(1/\delta)}{t}, \\
    c &:= 320 b^2 \gamma^2 (e-2) t, \\ 
    e_k &:= \gE_h^{\tilde{\pi}}(f_h,f_{h+1})(s^k_h, a^k_h).
\end{cases}
    \label{eq: define D, c, e_k}
\end{align}
\label{lemma: bound negative likelihood + log-partition function}
\end{lemma}
\begin{proof}[Proof of \Cref{lemma: bound negative likelihood + log-partition function}]
    For simplicity, denote 
\begin{align}
\begin{cases}
    u_k &:= \ln \sE_{(s_{h+1}, r_h) \sim P_h(\cdot|s_h, a_h)} \exp \left( - 2 \gamma  \Delta L_{\tilde{\pi}}(f_h, f_{h+1}; z^k_h) \right), \\ 
    v_k &:= - 2 \gamma  \Delta L_{\tilde{\pi}}(f_h, f_{h+1}; z^k_h), \\ 
    w_k &:= v_k - u_k, \\ 
    e_k &:= \gE_h^{\tilde{\pi}}(f_h,f_{h+1})(s^k_h, a^k_h). 
\end{cases}
\label{eq: new variables u,v,w,e}
\end{align}
We have 
\begin{align*}
    u_k &\geq  \sE_{(s_{h+1}, r_h) \sim P_h(\cdot|s_h, a_h)} \ln \exp \left( - 2 \gamma  \Delta L_{\tilde{\pi}}(f_h, f_{h+1}; z^k_h) \right) = -2 \gamma e_k^2,
\end{align*}
where the first inequality uses Jensen's inequality for concave function $\ln(\cdot)$ and the equality uses \Cref{lemma: the squared residuals bound the expected value and the variance of the empirical minimax error}. Now using \Cref{lemma: log E exp <= 0} with $\gamma \leq \frac{1}{144 (e-2) b^2}$, we have 
\begin{align}
    u_k &\leq  -\gamma e_k^2. 
    \label{eq: upper bound u_k}
\end{align}
We also have $\sE v_k = - 2 \gamma e_k^2$ by \Cref{lemma: the squared residuals bound the expected value and the variance of the empirical minimax error}. Thus, we have 
\begin{align*}
    |u_k| \leq 2 \gamma e_k^2, \text{ and } \sE[w_k] = -2 \gamma e_k^2 - u_k \leq 0. 
\end{align*}
Hence, we have 
\begin{align*}
    \sE w_k^2 &= \sE (v_k - u_k)^2 \\ 
    &\leq 2 \sE (v_k^2 + u_k^2) \\ 
    &\leq 288 b^2 \gamma^2  e_k^2 + 8 \gamma^2 e_k^4 \\ 
    &\leq 320 b^2 \gamma^2 e_k^2 
\end{align*}
where the first inequality uses Cauchy-Schwartz inequality, the second inequality uses \Cref{lemma: the squared residuals bound the expected value and the variance of the empirical minimax error} and that $|u_k| \leq 2 \gamma e_k^2$, and the last inequality uses that $|e_k| \leq 2b$. Also note that $|w_k| \leq |v_k| + |u_k| \leq 2 \gamma (9b^2) + 2 \gamma (4b^2) = 26 \gamma b^2$. Thus, by \Cref{lemma: Freedman inequality}, for any $\delta > 0$, for any $ t \in [0,\frac{1}{26 \gamma b^2}]$, with probability at least $1 - \delta$, we have
\begin{align*}
    \sum_{k=1}^K w_k &\leq \sum_{k=1}^K \sE w_k + (e-2) t \cdot \sE \sum_{k=1}^K w_k^2 + \frac{\ln(1/\delta)}{t} \\ 
    &\leq 320 b^2 \gamma^2 (e-2) t \sum_{k=1}^K e_k^2 +  \frac{\ln(1/\delta)}{t}.
\end{align*}

We apply the discretization argument and the union bound to obtain that: For any $\delta > 0, \eps > 0$, $T \in \sN$ it holds uniformly over all $\tilde{\pi} \in \Pi^{soft}_h(T), f_h \in \gF_h, f_{h+1} \in \gF_{h+1}$ that 
\begin{align*}
    \sum_{k=1}^K w_k &\leq 120 \gamma b (b+2) K \eps + 320 b^2 \gamma^2 (e-2) t \sum_{k=1}^K e_k^2 +  \frac{2 d_{\gF}(\eps) + d_{\Pi}(\eps,T) + \ln(1/\delta)}{t}.
\end{align*}
\end{proof}

\begin{lemma}
For any $\gamma \in [0, \frac{1}{144 (e-2) b^2}]$, $\eps > 0$, $\delta > 0$, we have
\begin{align*}
    B_{h,t}  &\geq  0.5 \gamma \sE_{\gD} \sE_{\tilde{\pi} \sim P_t(\cdot| \gD)} \sE_{f \sim \hat{p}(\cdot|\gD, \tilde{\pi})} \left[ \sum_{k=1}^K  \gE^{\tilde{\pi}}_h(f_h, f_{h+1})(s_h^k,a_h^k)^2 \right] \\ 
    &\geq - 120 \gamma b (b+2) K \eps -  640 (e-2) b^2 \gamma \left(2 d_{\gF}(\eps) + d_{\Pi}(\eps,T) + \ln(1/\delta) \right) - 26 \gamma b^2 K \delta. 
\end{align*}
\label{lemma: bound term B}
\end{lemma}
\begin{proof}[Proof of \Cref{lemma: bound term B}]
Define the random variables $u_k, v_k, w_k, e_k$ as Equation (\ref{eq: new variables u,v,w,e}). Recall $D,c$ are defined in Equation (\ref{eq: define D, c, e_k}) for any $t \in [0, \frac{1}{26 \gamma b^2}]$. Define the event $E$ such that the inequality 
\begin{align}
    & \sum_{k=1}^K \xi_h^{\tilde{\pi}}(f_h, f_{h+1}; z_h^k)  \leq  \underbrace{320 b^2 \gamma^2 (e-2) t}_{c} \sum_{k=1}^K e_k^2 + D,
    \label{eq: bound xi}
\end{align}
holds uniformly over all $\tilde{\pi} \in \Pi^{soft}(T), f_h \in \gF_h, f_{h+1} \in \gF_{h+1}$.
By \Cref{lemma: bound negative likelihood + log-partition function}, we have 
\begin{align*}
    \Pr(E) \geq 1 - \delta, \text{ thus } \Pr(E^c) \leq \delta.  
\end{align*}
We have
\begin{align}
    &\sE_{\gD} \sE_{\tilde{\pi} \sim P_t(\cdot| \gD)} \sE_{f \sim \hat{p}(\cdot|\gD, \tilde{\pi})}  \left[ \sum_{k=1}^K (-w_k + c e_k^2) + \ln \frac{\hat{p}(f_h, f_{h+1} | \gD, \tilde{\pi})}{p_{0}(f_h, f_{h+1})}\right] \nonumber\\ 
    &\geq \sE_{\gD} \sE_{\tilde{\pi} \sim P_t(\cdot| \gD)} \inf_p \sE_{f \sim p}  \left[  \sum_{k=1}^K (-w_k + c e_k^2) + \ln \frac{p(f_h, f_{h+1} )}{p_{0}(f_h, f_{h+1})}\right] \nonumber\\ 
    &= -\sE_{\gD} \sE_{\tilde{\pi} \sim P_t(\cdot| \gD)} \ln \sE_{f_h,f_{h+1} \sim p_0} \exp \left( \sum_{k=1}^K (w_k- c e_k^2)  \right) \nonumber\\ 
    &= -\sE_{\gD} 1\{E\} \sE_{\tilde{\pi} \sim P_t(\cdot| \gD)} \ln \sE_{f_h,f_{h+1} \sim p_0} \exp \left( \sum_{k=1}^K (w_k - c e_k^2) \right)  \nonumber\\ 
    & -\sE_{\gD} 1\{E^c\} \sE_{\tilde{\pi} \sim P_t(\cdot| \gD)} \ln \sE_{f_h,f_{h+1} \sim p_0} \exp \left( \sum_{k=1}^K (w_k - c e_k^2)  \right) \nonumber\\
    &\geq -\sE_{\gD} 1\{E\} \sE_{\tilde{\pi} \sim P_t(\cdot| \gD)} \ln \sE_{f_h,f_{h+1} \sim p_0} \exp \left( \sum_{k=1}^K (w_k - c e_k^2) \right) - 26 \gamma b^2 K \delta \nonumber\\ 
    &\geq  - D - 26 \gamma b^2 K \delta,
    \label{eq: lower bound -xi + cek^2}
\end{align}
where the first equality uses \Cref{lemma: log partition is lower bound for f + log p/p_0}, the second inequality uses that $\Pr(E^c) \leq \delta$ and $\sum_{k=1}^K (w_k - c e_k^2) \leq \sum_{k=1}^K w_k \leq 26 \gamma b^2 K$ and the last inequality uses \Cref{eq: bound xi}. Thus, using the same notations as Lemma \Cref{lemma: bound negative likelihood + log-partition function}, we have 
\begin{align*}
    B_{h,t} &= \sE_{\gD} \sE_{\tilde{\pi} \sim P_t(\cdot| \gD)} \sE_{f \sim \hat{p}(\cdot|\gD, \tilde{\pi})} \left[ -\sum_{k=1}^K v_k  + \ln \frac{\hat{p}(f_h, f_{h+1} | \gD, \tilde{\pi})}{p_{0}(f_h, f_{h+1})} \right] \\ 
    &= \sE_{\gD} \sE_{\tilde{\pi} \sim P_t(\cdot| \gD)} \sE_{f \sim \hat{p}(\cdot|\gD, \tilde{\pi})} \left[ \sum_{k=1}^K (-w_k + c e_k^2)  + \ln \frac{\hat{p}(f_h, f_{h+1} | \gD, \tilde{\pi})}{p_{0}(f_h, f_{h+1})} \right] \\ 
    &+ \sE_{\gD} \sE_{\tilde{\pi} \sim P_t(\cdot| \gD)} \sE_{f \sim \hat{p}(\cdot|\gD, \tilde{\pi})} \left[ \sum_{k=1}^K (-u_k - c e_k^2) \right] \\ 
    &\geq - D - 26 \gamma b^2 K \delta + (\gamma - c) \sE_{\gD} \sE_{\tilde{\pi} \sim P_t(\cdot| \gD)} \sE_{f \sim \hat{p}(\cdot|\gD, \tilde{\pi})} \left[ \sum_{k=1}^K  e_k^2 \right] 
\end{align*}
where the inequality uses \Cref{eq: lower bound -xi + cek^2} and \Cref{eq: upper bound u_k}. Finally, setting 
\begin{align*}
    t = \frac{1}{640 b^2 (e-2) \gamma} < \frac{1}{13 b^2 \gamma} 
\end{align*}
completes our proof.

\end{proof}

\paragraph{From squared Bellman errors to \emph{in-expectation} squared Bellman errors and fixing a non-rigorous argument of \cite{DBLP:conf/nips/DannMZZ21}.}
\label{para: non-rigorous argument of Dann}

\Cref{lemma: bound term B} only bounds $B_h$ with the squared Bellman errors $\sum_{k=1}^K  \gE^{\tilde{\pi}}_h(f_h, f_{h+1})(s_h^k,a_h^k)^2 $ while the \emph{in-expectation} squared Bellman errors $\sum_{k=1}^K \sE_{\mu^k}[\gE^{\tilde{\pi}}_h(f_h, f_{h+1})(s_h,a_h)^2]$ are what we need for showing \Cref{theorem: expected pessimistic squared Bellman error}. There is no an immediate path to go from the squared Bellman error to the in-expectation squared Bellman errors as the order of $\sE_{\gD}$ and $\sE_{\tilde{\pi} \sim P_t(\cdot| \gD)} \sE_{f \sim \hat{p}(\cdot|\gD, \tilde{\pi})}$ are \textbf{not} exchangeable, i.e.,
\begin{align}
    &\sE_{\gD} \sE_{\tilde{\pi} \sim P_t(\cdot| \gD)} \sE_{f \sim \hat{p}(\cdot|\gD, \tilde{\pi})} \left[ \sum_{k=1}^K  \gE^{\tilde{\pi}}_h(f_h, f_{h+1})(s_h^k,a_h^k)^2 \right] \nonumber \\
    &\neq \sE_{\gD} \sE_{\tilde{\pi} \sim P_t(\cdot| \gD)} \sE_{f \sim \hat{p}(\cdot|\gD, \tilde{\pi})} \left[ \sum_{k=1}^K  \sE_{\mu^k} [\gE^{\tilde{\pi}}_h(f_h, f_{h+1})(s_h,a_h)^2] \right]. 
    \label{eq: dann mistake}
\end{align}
A similar caveat arises in the online setting in \cite{DBLP:conf/nips/DannMZZ21} as well. In particular, a non-rigorous argument of \cite[Lemma~8]{DBLP:conf/nips/DannMZZ21} is that they conclude (an online analogue of) the LHS of \Cref{eq: dann mistake} is equal to (an online analogue of) its RHS.

To fix this issue without ultimately incurring a sub-optimality rate that is slower than $1/\sqrt{K}$, we need to change the squared Bellman error into the in-expectation squared Bellman error, up to some estimation error that scales faster than $K^{\alpha}$ for any $\alpha > 0$. Note that, a standard Azuma–Hoeffding inquality (and the union bound) give an estimation error that scales with $K^{1/2}$. To achieve the logarithmic dependence on $K$, the following lemma exploits the \emph{non-negativity} of the squared Bellman error and uses the localization argument of \cite{bartlett2005local} to obtain an estimation error rate that scales polylogarithmic with $K$. 

\begin{lemma}[Improved online-to-batch argument for non-negative R.V.s \citep{nguyen2022instance}]
Let $\{X_k\}$ be any real-valued stochastic process adapted to the filtration $\{\mathcal{F}_k\}$, i.e. $X_k$ is $\mathcal{F}_k$-measurable. Suppose that for any $k$, $X_k \in [0,H]$ almost surely for some $H > 0$. For any $K > 0$, with probability at least $1 - \delta$, we have:
\begin{align*}
    \sum_{k=1}^K \mathbb{E} \left[ X_k | \mathcal{F}_{k-1} \right] \leq 2 \sum_{k=1}^K X_k + \frac{16}{3} H \log(\log_2(KH)/\delta) + 2 . 
\end{align*}
\label{lemma:improved_online_to_batch}
\end{lemma}
With \Cref{lemma:improved_online_to_batch}, we now actually make a connection from the squared Bellman error to the \emph{in-expectation} squared Bellman error in the following lemma, which incorporates the uniform convergence argument into the posterior sampling in the same spirit with our earlier argument in \Cref{section: proof of key support theorem for ps}.  
\begin{lemma}
For any $\delta, \eps > 0$ and $T \in \sN$, any $t \in [T]$, we have
\begin{align*}
    &\sE_{\gD} \sE_{\tilde{\pi} \sim P_t(\cdot| \gD)} \sE_{f \sim \hat{p}(\cdot|\gD, \tilde{\pi})} \sum_{k=1}^K \gE^{\tilde{\pi}}_h(f_h, f_{h+1})(s_h^k,a_h^k)^2 \\ 
    &\geq 0.5 \sE_{\gD} \sE_{\tilde{\pi} \sim P_t(\cdot| \gD)} \sE_{f \sim \hat{p}(\cdot|\gD, \tilde{\pi})} \sum_{k=1}^K \sE_{\mu^k} \left[ \gE^{\tilde{\pi}}_h(f_h, f_{h+1})(s_h,a_h)^2 \right]\\ 
    &- b (b+2) K \eps - \frac{32}{3}b^2 \left( 2 d_{\gF}(\eps) + d_{\Pi}(\eps, T) +  \ln \frac{\ln 4 Kb^2}{\delta} \right) - 1 - 2 K b^2 \delta.
\end{align*}
\label{lemma: empirical Bellman error to expected Bellman error}

\end{lemma}
\begin{proof}[Proof of \Cref{lemma: empirical Bellman error to expected Bellman error}]
 For simplicity, we denote $X(f,\gD) := \sum_{k=1}^K \gE^{\tilde{\pi}}_h(f_h, f_{h+1})(s_h^k,a_h^k)^2 $ and $X(f) := \sE_{\gD}[X(f,\gD)]$, and $\Delta := 8 b K \eps + \frac{64}{3}b^2 \left( 2 d_{\gF}(\eps) + d_{\Pi}(\eps, T) +  \frac{\ln \ln 4 Kb^2}{\delta} \right) + 2$. We define the event: 
\begin{align*}
    E = \left\{\gD: X(f) \leq 2 X(f,\gD) + \Delta, \forall f_h \in \gF_h, f_{h+1} \in \gF_{h+1}, \tilde{\pi} \in \Pi^{soft}_h(T) \right\}.
\end{align*}
Due to the non-negativity of $\gE^{\tilde{\pi}}_h(f_h, f_{h+1})(s,a)^2$, \Cref{lemma:improved_online_to_batch} and the union bound, we have 
\begin{align*}
    \Pr(E) \geq 1 - \delta \text{ and } \Pr(E^c) \leq \delta. 
\end{align*}
We have 
\begin{align*}
     &2 \sE_{\gD} \sE_{\tilde{\pi} \sim P_t(\cdot| \gD)} \sE_{f \sim \hat{p}(\cdot|\gD, \tilde{\pi})} X(f,\gD) \\
     &= 2 \sE_{\gD} \sE_{\tilde{\pi} \sim P_t(\cdot| \gD)} \sE_{f \sim \hat{p}(\cdot|\gD, \tilde{\pi})} X(f,\gD) 1\{ E\} + 2 \sE_{\gD} \sE_{\tilde{\pi} \sim P_t(\cdot| \gD)} \sE_{f \sim \hat{p}(\cdot|\gD, \tilde{\pi})} X(f,\gD) 1\{ E^c\} \\ 
     &\geq 2 \sE_{\gD} \sE_{\tilde{\pi} \sim P_t(\cdot| \gD)} \sE_{f \sim \hat{p}(\cdot|\gD, \tilde{\pi})} X(f,\gD) 1\{  E\} \\ 
     &\geq \sE_{\gD} 1\{ E\} \sE_{\tilde{\pi} \sim P_t(\cdot| \gD)} \sE_{f \sim \hat{p}(\cdot|\gD , \tilde{\pi})} (X(f) - \Delta) \\ 
     &= \sE_{\gD} 1\{E\} \sE_{\tilde{\pi} \sim P_t(\cdot| \gD)} \sE_{f \sim \hat{p}(\cdot|\gD , \tilde{\pi})} X(f) - \Delta \Pr(E) \\ 
     &\geq \sE_{\gD} 1\{ E\} \sE_{\tilde{\pi} \sim P_t(\cdot| \gD)} \sE_{f \sim \hat{p}(\cdot|\gD , \tilde{\pi})} X(f) - \Delta \\ 
     &= \sE_{\gD} \sE_{\tilde{\pi} \sim P_t(\cdot| \gD)}  \sE_{f \sim \hat{p}(\cdot|\gD , \tilde{\pi})} X(f) - \Delta - \sE_{\gD} 1\{ E^c\} \sE_{\tilde{\pi} \sim P_t(\cdot| \gD)} \sE_{f \sim \hat{p}(\cdot|\gD , \tilde{\pi})} X(f) \\ 
     &\geq \sE_{\gD}  \sE_{\tilde{\pi} \sim P_t(\cdot| \gD)} \sE_{f \sim \hat{p}(\cdot|\gD , \tilde{\pi})} X(f) - \Delta - \sE_{\gD} 1\{ E^c\}  \sE_{f \sim \hat{p}(\cdot|\gD , \tilde{\pi})} 4 K b^2 \\ 
     &= \sE_{\gD}  \sE_{f \sim \hat{p}(\cdot|\gD , \tilde{\pi})} X(f) - \Delta - 4 K b^2 \Pr(E^c) \\ 
     &\geq \sE_{\gD}  \sE_{f \sim \hat{p}(\cdot|\gD , \tilde{\pi})} X(f) - \Delta - 4 K b^2 \delta
\end{align*}
where the fourth inequality uses $|X(f)| \geq 4 K b^2$ and the last inequality uses $\Pr(E^c) \leq \delta$.
\end{proof}

\paragraph{Bounding $C_{h,t}$.}
\begin{lemma}
For any $\eps > 0$ and $T \in \sN$, we have
\begin{align*}
    C_{h,t} \geq -\max_{\tilde{\pi} \in \Pi_h^{soft}(T)} \kappa_h(\alpha, \eps, \tilde{\pi}) - \gamma \alpha 6 b K \eps.
\end{align*}
where $\kappa_h(\alpha, \eps, \tilde{\pi})$ is defined in \Cref{eq: define kappa_h}. 
\label{lemma: lower bound C_h}
\end{lemma}

\begin{proof}[Proof of \Cref{lemma: lower bound C_h}]
We have
\begin{align*}
    &C_{h,t} = \sE_{\gD} \sE_{\tilde{\pi} \sim P_t(\cdot| \gD)} \sE_{f \sim \hat{p}(\cdot|\gD, \tilde{\pi})} \left[ \alpha \ln \sE_{f'_h \sim p_0} \exp \left( - \gamma \sum_{k=1}^K \Delta L_{\tilde{\pi}}(f'_h, f_{h+1}; z^k_h) \right) + (1 - \alpha) \ln \frac{\hat{p}(f_{h+1}| \gD, \tilde{\pi})}{p_0(f_{h+1})}\right] \\ 
    &= (1 - \alpha)\sE_{\gD} \sE_{\tilde{\pi} \sim P_t(\cdot| \gD)} \sE_{f \sim \hat{p}(\cdot|\gD, \tilde{\pi})} \left[ \frac{\alpha}{1 - \alpha} \ln \sE_{f'_h \sim p_0} \exp \left( - \gamma \sum_{k=1}^K \Delta L_{\tilde{\pi}}(f'_h, f_{h+1}; z^k_h) \right) +  \ln \frac{\hat{p}(f_{h+1}| \gD, \tilde{\pi})}{p_0(f_{h+1})}\right] \\ 
    &\geq -(1 - \alpha) \sE_{\gD} \sE_{\tilde{\pi} \sim P_t(\cdot| \gD)} \ln \sE_{f_{h+1} \sim p_0} \left( \sE_{f'_h \sim p_0} \exp \left( - \gamma \sum_{k=1}^K \Delta L_{\tilde{\pi}}(f'_h, f_{h+1}; z^k_h) \right) \right)^{\frac{-\alpha}{1 - \alpha}} \\ 
    &\geq -\max_{\tilde{\pi} \in \Pi^{soft}(T)}\kappa_h(\alpha, \eps, \tilde{\pi}) - \gamma \alpha 6 b K \eps. 
\end{align*}
where the first inequality uses \Cref{lemma: log partition is lower bound for f + log p/p_0} and the last inequality uses the following inequalities: For any $f_h \in \gF_h(\eps, f_{h+1}, \tilde{\pi})$, we have 
\begin{align*}
    |\Delta L_{\tilde{\pi}}(f_h, f_{h+1}; z_h)| &\leq 6b |\gE_h^{\tilde{\pi}}(f_h, f_{h+1}) | \leq 6b \eps; \text{ thus } \\
    \sE_{f'_h \sim p_0} \exp \left( - \gamma \sum_{k=1}^K \Delta L_{\tilde{\pi}}(f'_h, f_{h+1}; z_h) \right) &\geq  p_{0,h}(\gF_h^{\tilde{\pi}}(\eps, f_{h+1})) \cdot \exp (- \gamma 6 b K \eps). 
\end{align*}
\end{proof}

We are now ready to state the complete form of the lower bound of $Z$.
\begin{prop}
For any $\gamma \in [0, \frac{1}{144 (e-2) b^2}]$, $\eps > 0$, $\delta > 0$, and any $t \in [T]$, we have, 
\begin{align*}
    Z &\geq \lambda \sE_{\gD} \sE_{\tilde{\pi} \sim P_t(\cdot| \gD)} \sE_{f \sim \hat{p}(\cdot|\gD, \tilde{\pi})} \Delta f_1(s_1, \tilde{\pi}) \\ 
    &+ 0.1 25 \alpha \gamma  \sum_{h=1}^H  \sE_{\gD} \sE_{\tilde{\pi} \sim P_t(\cdot| \gD)} \sE_{f \sim \hat{p}(\cdot|\gD, \tilde{\pi})} \sum_{k=1}^K \sE_{\mu^k} \left[ \gE^{\tilde{\pi}}_h(f_h, f_{h+1})(s_h,a_h)^2 \right] \\ 
    & - 0.5 \alpha H \left( 120 \gamma b (b+2) K \eps +  640 (e-2)  \gamma b^2 \left( 2 d_{\gF}(\eps) + d_{\Pi}(\eps,T) + \ln(1/\delta) \right) \right) - 13 \alpha \gamma b^2 K H \delta  \\ 
    & - 0.2 5 \alpha \gamma  H \left(  b (b+2) K \eps + \frac{32}{3}b^2 \left( 2 d_{\gF}(\eps) + d_{\Pi}(\eps, T) +  \ln \frac{ \ln 4 Kb^2}{\delta} \right) + 1 + 2 K b^2 \delta \right) \\ 
    & - \sum_{h=1}^H \max_{\tilde{\pi} \in \Pi^{soft}(T)} \kappa_h(\alpha, \eps, \tilde{\pi}) - \gamma \alpha 6 b K H \eps.
\end{align*}
    \label{prop: lower bound Z}
\end{prop}

\begin{proof}[Proof of \Cref{prop: lower bound Z}]
Using  \Cref{lemma: decompose Z}, it suffices to bound terms $A$, $B_h$, and $C_h$ defined in \Cref{lemma: decompose Z}. For this purpose, we use 
\begin{itemize}
    \item \Cref{lemma: bound A}: To bound $A_t$,  
    \item \Cref{lemma: bound term B} and \Cref{lemma: empirical Bellman error to expected Bellman error}: To bound $B_{h,t}$, 
    \item \Cref{lemma: lower bound C_h}: To bound $C_{h,t}$. 
\end{itemize}
The result is then simply a direct combination of the above lemmas. 
\end{proof}

\subsubsection{Upper-bounding log-partition function}
In this appendix, we upper bound the log-partition function $Z$. While we follow the proof flow in \cite{DBLP:conf/nips/DannMZZ21}, due to the statistical dependence in the actor-critic framework of our algorithm, we require different technical arguments to establish this result. In particular, \Cref{lemma: upper bound Delta L} and \Cref{lemma: bound - gamma Delta L} are our \emph{new} technical lemmas.  
\begin{prop}
    For any $\eps, \delta > 0$, $\gamma > 0$, and $t \in [T]$, we have 
   \begin{align*}
       Z_t &\leq \lambda \eps - \inf_{\tilde{\pi} \in \Pi^{soft}(T)} \sum_{h=1}^H \ln p_0(\gF_h(\eps; \tilde{\pi})) + 4 \gamma \left( \alpha  + \frac{3(e-2)}{ \alpha }\right) H K \eps^2 \\ 
       &+ 60 \alpha \gamma b(b+2) K H \eps + \alpha \gamma b^2 H \left(13   + 36 (e-2) \right) \left( 2 d_{\gF}(\eps) + d_{\Pi}(\eps, T) + \ln(1/\delta) \right) + 18 \alpha  \gamma K H b^2 \delta.
   \end{align*}
   where recall that $Z_t$ is defined in \Cref{eq: log partition function}. 
   \label{prop: upper bound Z}
\end{prop}
\begin{proof}[Proof of \Cref{prop: upper bound Z}]
We have
\begin{align*}
    Z_t &= \sE_{\gD} \sE_{\tilde{\pi} \sim P_t(\cdot| \gD)} \sE_{f \sim \hat{p}(\cdot|\gD, \tilde{\pi})} \left[ \widehat{\Phi}(f, \tilde{\pi}; \gD) + \lambda \Delta f_1(s_1, \tilde{\pi}) + \ln \hat{p}(f| \gD, \tilde{\pi}) \right] \\ 
    &= \sE_{\gD} \sE_{\tilde{\pi} \sim P_t(\cdot| \gD)} \inf_{p} \sE_{f \sim p} \left[ \widehat{\Phi}(f, \tilde{\pi}; \gD) + \lambda \Delta f_1(s_1, \tilde{\pi}) + \ln p(f) \right] \\ 
    &\leq \sE_{\gD} \sE_{\tilde{\pi} \sim P_t(\cdot| \gD)}  \inf_{p} \sE_{f \sim p} \left[ \ln \frac{p(f)}{p_{0}(f)} + \alpha \gamma \sum_{h=1}^H \sum_{k=1}^K \Delta L_{\tilde{\pi}}(f_h, f_{h+1}; z^k_h)  + \lambda \Delta f_1(s_1, \tilde{\pi}) \right] \nonumber\\
    &+ \sE_{\gD} \sE_{\tilde{\pi} \sim P_t(\cdot| \gD)}  \inf_{p} \sE_{f \sim p} \left[ \alpha \sum_{h=1}^H \ln \sE_{\tilde{f}_h \sim p_{0}} \exp \left( -\gamma \sum_{k=1}^K \Delta L_{\tilde{\pi}}(\tilde{f}_h, f_{h+1}; z^k_h) \right)  \right]
\end{align*}
where the second equality uses the fact that $\KL [p \| \hat{p}] \geq 0$ with the minimum occurring at $p = \hat{p}$ and the inequality uses the triangle inequality. The first term is bounded by \Cref{lemma: upper bound Delta L} and \Cref{lemma: upper bound gap terms with eps-closedness}, and the second term is bounded by \Cref{lemma: bound - gamma Delta L}
\end{proof}

It remains to state and prove \Cref{lemma: upper bound Delta L}, \Cref{lemma: upper bound gap terms with eps-closedness} and \Cref{lemma: bound - gamma Delta L}. 

The following lemma bounds the in-expectation of the loss $\Delta L_{\tpi}$ by the in-expectation of the squared Bellman error.
\begin{lemma}
    For any distribution $p$ over $\gF$, for any $\eps, \delta > 0$, $\gamma > 0$, any $t \in [T]$, we have
    \begin{align*}
        &\sE_{\gD} \sE_{\tilde{\pi} \sim P_t(\cdot| \gD)}  \sE_{f \sim p} \left[ \alpha \gamma  \sum_{k=1}^K \Delta L_{\tilde{\pi}}(f_h, f_{h+1}; z^k_h)  \right] \\
        &\leq \gamma \left( \alpha  + \frac{3(e-2)}{ \alpha }\right) \sE_{\gD} \sE_{\tilde{\pi} \sim P_t(\cdot| \gD)}  \sE_{f \sim p} \left[ \sum_{k=1}^K \gE_h^{\tilde{\pi}}(f_h, f_{h+1})(s^k_h, a^k_h)^2 \right] \\ 
        &+ 30 \alpha \gamma b(b+2) K \eps + 13 \alpha \gamma b^2 \left( 2 d_{\gF}(\eps) + d_{\Pi}(\eps, T) + \ln(1/\delta) \right) + 9 \alpha  \gamma K b^2 \delta. 
    \end{align*}
    \label{lemma: upper bound Delta L}
\end{lemma}

\begin{proof}[Proof of \Cref{lemma: upper bound Delta L}]
For simplicity, define 
\begin{align*}
    x_k &:= \alpha \gamma  \Delta L_{\tilde{\pi}}(f_h, f_{h+1}; z^k_h), \\ 
    e_k &:= \gE_h^{\tilde{\pi}}(f_h, f_{h+1})(s^k_h, a^k_h). 
\end{align*}
    By \Cref{lemma: the squared residuals bound the expected value and the variance of the empirical minimax error}, we have
    \begin{align*}
        \sE [x_k] &= \alpha \gamma e_k^2, \\ 
        \sE [x_k^2] &\leq 36 b^2 \alpha^2 \gamma^2 e_k^2. 
    \end{align*}
    Thus, by \Cref{lemma: Freedman inequality}, for any $\delta > 0$, with probability at least $1 - \delta$, for any $t \in [0, \frac{1}{13 \alpha \gamma b^2}]$ we have 
    \begin{align*}
        \sum_{k=1}^K x_k &\leq \sum_{k=1}^K \sE[x_k] +  t(e-2) \sum_{k=1}^K \sE[x_k^2] + \frac{\ln(1/\delta)}{t} \\ 
        &\leq \left( \alpha \gamma + t(e-2) 36 b^2 \gamma^2\right) \sum_{k=1}^K e_k^2 + \frac{\ln(1/\delta)}{t}. 
    \end{align*}
    Using the discretization argument and the union bound, we have that: For any $\eps >0, \delta > 0$, we have 
    \begin{align*}
        \Pr(E) \geq 1 - \delta, \text{ thus } \Pr(E^c) \leq \delta, 
    \end{align*}
    where $E$ denotes that event that for any $t \in [0, \frac{1}{13 \alpha \gamma b^2}]$, 
    \begin{align*}
        \sum_{k=1}^K x_k &\leq 30 \alpha \gamma b(b+2) K \eps +  \left( \alpha \gamma + t(e-2) 36 b^2 \alpha^2 \gamma^2\right) \sum_{k=1}^K e_k^2 + \frac{2 d_{\gF}(\eps) + d_{\Pi}(\eps,T) + \ln(1/\delta)}{t}, 
    \end{align*}
    any $f_h \in \gF_h$, $f_{h+1} \in \gF_{h+1}, \tilde{\pi} \in \Pi^{soft}_h(T)$. 
    Thus, we have 
    \begin{align*}
        &\sE_{\gD} \sE_{\tilde{\pi} \sim P_t(\cdot| \gD)}  \sE_{f \sim p} \left[ \sum_{k=1}^K x_k  \right] = \sE_{\gD} 1\{E\} \sE_{\tilde{\pi} \sim P_t(\cdot| \gD)}  \sE_{f \sim p} \left[ \sum_{k=1}^K x_k  \right] + \sE_{\gD} 1\{E^c\}\sE_{\tilde{\pi} \sim P_t(\cdot| \gD)}  \sE_{f \sim p} \left[ \sum_{k=1}^K x_k  \right] \\ 
        &\leq \sE_{\gD} 1\{E\} \sE_{\tilde{\pi} \sim P_t(\cdot| \gD)}  \sE_{f \sim p} \bigg[ 30 \alpha \gamma b(b+2) K \eps \\ 
        &+  \left( \alpha \gamma + t(e-2) 36 \alpha^2 b^2 \gamma^2\right) \sum_{k=1}^K e_k^2 + \frac{2 d_{\gF}(\eps) + d_{\Pi}(\eps,T) + \ln(1/\delta)}{t}  \bigg] + 9 \alpha  \gamma K b^2 \delta \\ 
        &\leq \sE_{\gD} \sE_{\tilde{\pi} \sim P_t(\cdot| \gD)}  \sE_{f \sim p} \bigg[ 30 b(b+2) K \eps \\ 
        &+  \left( \alpha \gamma + t(e-2) 36 \alpha^2 b^2 \gamma^2\right) \sum_{k=1}^K e_k^2 + \frac{2 d_{\gF}(\eps) + d_{\Pi}(\eps,T) + \ln(1/\delta)}{t}  \bigg] + 9 \alpha  \gamma K b^2 \delta.
    \end{align*}
    Picking $t = \frac{1}{13 \alpha \gamma b^2}$ completes the proof.
\end{proof}

The following lemma bounds the in-expectation negation of the loss proxy $\Delta L_{\tpi}$.  
\begin{lemma}
   For any $\delta > 0, \eps > 0, \gamma > 0$, any $\tilde{f}_h \in \gF_h$, any $t \in [T]$,  and any distribution $p$ over $\gF$, we have
   \begin{align*}
       \sE_{\gD} \sE_{\tilde{\pi} \sim P_t(\cdot| \gD)}  \sE_{f \sim p} \left[ -\gamma \sum_{k=1}^K \Delta L_{\tilde{\pi}}(\tilde{f}_h, f_{h+1}; z^k_h) \right] &\leq   36 (e-2) \gamma b^2 \left( d_{\gF}(\eps) + d_{\Pi}(\eps,T) + \ln(1/\delta) \right) \\ 
       &+ 9  \gamma K b^2 \delta + 30 \gamma b(b+2) K \eps. 
   \end{align*}
   \label{lemma: bound - gamma Delta L}
\end{lemma}
\begin{proof}[Proof of \Cref{lemma: bound - gamma Delta L}]
    For simplicity, define 
\begin{align*}
    y_k &:= - \gamma  \Delta L_{\tilde{\pi}}(f_h, f_{h+1}; z^k_h), \\ 
    e_k &:= \gE_h^{\tilde{\pi}}(f_h, f_{h+1})(s^k_h, a^k_h). 
\end{align*}
    By \Cref{lemma: the squared residuals bound the expected value and the variance of the empirical minimax error}, we have
    \begin{align*}
        \sE [y_k] &= -\gamma e_k^2, \\ 
        \sE [y_k^2] &\leq 36 b^2  \gamma^2 e_k^2. 
    \end{align*}
    Thus, by \Cref{lemma: Freedman inequality}, for any $\delta > 0$, with probability at least $1 - \delta$, for any $t \in [0, \frac{1}{13  \gamma b^2}]$ we have 
    \begin{align*}
        \sum_{k=1}^K y_k &\leq \sum_{k=1}^K \sE[y_k] +  t(e-2) \sum_{k=1}^K \sE[y_k^2] + \frac{\ln(1/\delta)}{t} \\ 
        &\leq -\gamma\left( 1 - 36 t (e-2) b^2 \gamma \right) \sum_{k=1}^K e_k^2 + \frac{\ln(1/\delta)}{t}. 
    \end{align*}
    Setting $t = \frac{1}{36 (e-2) b^2 \gamma} < \frac{1}{13 \gamma b^2}$ in the above inequality, we obtain 
    \begin{align*}
        \sum_{k=1}^k y_k \leq 36 (e-2) b^2 \gamma \ln(1/\delta). 
    \end{align*}
    Using the discretization argument and the union bound, we have that: For any $\eps >0, \delta > 0$, we have 
    \begin{align*}
        \Pr(E) \geq 1 - \delta, \text{ thus } \Pr(E^c) \leq \delta, 
    \end{align*}
    where $E$ denotes that event , 
    \begin{align*}
        \sum_{k=1}^K y_k &\leq 30 \gamma b(b+2) K \eps +   36 (e-2) \gamma b^2 \left( d_{\gF}(\eps) + d_{\Pi}(\eps,T) + \ln(1/\delta) \right), 
    \end{align*}
    any $f_{h+1} \in \gF_{h+1}, \tilde{\pi} \in \Pi^{soft}_h(T)$. Thus, we have 
    \begin{align*}
        &\sE_{\gD} \sE_{\tilde{\pi} \sim P_t(\cdot| \gD)}  \sE_{f \sim p} \left[ \sum_{k=1}^K y_k  \right] \\
        &= \sE_{\gD} 1\{E\} \sE_{\tilde{\pi} \sim P_t(\cdot| \gD)}  \sE_{f \sim p} \left[ \sum_{k=1}^K y_k  \right] + \sE_{\gD} 1\{E^c\}\sE_{\tilde{\pi} \sim P_t(\cdot| \gD)}  \sE_{f \sim p} \left[ \sum_{k=1}^K y_k  \right] \\ 
        &\leq 30 \gamma b(b+2) K \eps +   36 (e-2) \gamma b^2 \left( d_{\gF}(\eps) + d_{\Pi}(\eps,T) + \ln(1/\delta) \right) + 9  \gamma K b^2 \delta.
    \end{align*}
\end{proof}

The following lemma bounds the in-expectation squared Bellman errors with the regularization term and the data distribution term, under the infimum realization of the data distribution $p$.
\begin{lemma}
For any $\eps > 0, \beta \geq 0$, any $t \in [T]$,  we have
\begin{align*}
     &\sE_{\gD} \sE_{\tilde{\pi} \sim P_t(\cdot| \gD)} \inf_{p} \sE_{f \sim p} \bigg[
       \lambda \Delta f_1(s_1, \tilde{\pi}) + \ln \frac{p(f)}{p_0(f)} + \beta \sum_{h=1}^H \sum_{k=1}^K \gE_h^{\tilde{\pi}}(f_h, f_{h+1})(s^k_h, a^k_h)^2\bigg] \\ 
       &\leq \lambda \eps - \inf_{\tilde{\pi} \in \Pi^{soft}(T)}\ln p_0(\gF_h(\eps; Q_h^{\tilde{\pi}})) + 4 \beta H K \eps^2. 
\end{align*}
where recall that $\gF_h(\eps; f_{h+1})$ is defined in \Cref{defn: complexity measure for canonical posterior}. 
\label{lemma: upper bound gap terms with eps-closedness}
\end{lemma}
\begin{proof}[Proof of \Cref{lemma: upper bound gap terms with eps-closedness}]
For any $f \in \gF(\eps; Q^{\tilde{\pi}})$, we have
\begin{align*}
    \|f_h - Q_h^{\tilde{\pi}}\|_{\infty} \leq \eps, \forall h. 
\end{align*}
Thus, we have
\begin{align*}
    |\gE_h^{\tilde{\pi}}(f_h, f_{h+1})(s,a)| &\leq \|\sT^{\tilde{\pi}}_h f_h - f_{h+1} \|_{\infty} = \|\sT^{\tilde{\pi}}_h f_h - \sT^{\tilde{\pi}}_h Q_h^{\tilde{\pi}} - f_{h+1} + Q_{h+1}^{\tilde{\pi}} \|_{\infty} \\ 
    &\leq \|\sT^{\tilde{\pi}}_h f_h - \sT^{\tilde{\pi}}_h Q_h^{\tilde{\pi}} \|_{\infty} + \| f_{h+1} - Q_{h+1}^{\tilde{\pi}} \|_{\infty} \\ 
    &\leq 2 \eps. 
\end{align*}
    Thus, by choosing 
    \begin{align*}
        p(f) = \frac{p_0(f) 1\{f \in \gF(\eps; \tilde{\pi})  \}}{p_0(\gF(\eps; \tilde{\pi}))},
    \end{align*}
    we have 
    \begin{align*}
        &\inf_{p} \sE_{f \sim p} \bigg[
       \lambda \Delta f_1(s_1, \tilde{\pi}) + \ln \frac{p(f)}{p_0(f)} + \beta \sum_{h=1}^H \sum_{k=1}^K \gE_h^{\tilde{\pi}}(f_h, f_{h+1})(s^k_h, a^k_h)^2\bigg] \\ 
       &\leq \lambda \eps - \ln p_0(\gF_h(\eps; \tilde{\pi})) + 4 \beta H K \eps^2. 
    \end{align*}
\end{proof}

We by now have everything needed to prove \Cref{theorem: expected pessimistic squared Bellman error}. 
\subsubsection{Proof of \Cref{theorem: expected pessimistic squared Bellman error}}
\begin{proof}[Proof of \Cref{theorem: expected pessimistic squared Bellman error}]
By \Cref{prop: lower bound Z}, we have 
\begin{align*}
Z_t &\geq \lambda \sE_{\gD} \sE_{\tilde{\pi} \sim P_t(\cdot| \gD)} \sE_{f \sim \hat{p}(\cdot|\gD, \tilde{\pi})} \Delta f_1(s_1, \tilde{\pi}) \\ 
&+ 0.1 25 \alpha \gamma  \sum_{h=1}^H  \sE_{\gD} \sE_{\tilde{\pi} \sim P_t(\cdot| \gD)} \sE_{f \sim \hat{p}(\cdot|\gD, \tilde{\pi})} \sum_{k=1}^K \sE_{\mu^k} \left[ \gE^{\tilde{\pi}}_h(f_h, f_{h+1})(s_h,a_h)^2 \right] \\ 
& - 0.5 \alpha H \left( 120 \gamma b (b+2) K \eps +  640 (e-2)  \gamma b^2 \left( 2 d_{\gF}(\eps) + d_{\Pi}(\eps,T) + \ln(1/\delta) \right) \right) - 13 \alpha \gamma b^2 K H \delta  \\ 
& - 0.2 5 \alpha \gamma  H \left(  b (b+2) K \eps + \frac{32}{3}b^2 \left( 2 d_{\gF}(\eps) + d_{\Pi}(\eps,T) +  \frac{\ln \ln 4 Kb^2}{\delta} \right) + 1 + 2 K b^2 \delta \right) \\ 
& - \sum_{h=1}^H \max_{\tilde{\pi} \in \Pi^{soft}(T)} \kappa_h(\alpha, \eps, \tilde{\pi}) - \gamma \alpha 6 b K H \eps.
\end{align*}
By \Cref{prop: upper bound Z}, we have 
\begin{align*}
       Z_t &\leq \lambda \eps - \inf_{\tilde{\pi} \in \Pi^{soft}(T)} \sum_{h=1}^H \ln p_0(\gF_h(\eps; \tilde{\pi})) + 4 \gamma \left( \alpha  + \frac{3(e-2)}{ \alpha }\right) H K \eps^2 \\ 
       &+ 60 \alpha \gamma b(b+2) K H \eps + \alpha b^2 \gamma H  \left(13    + 36 (e-2) \right) \left( 2 d_{\gF}(\eps) + d_{\Pi}(\eps,T) + \ln(1/\delta) \right) + 18 \alpha  \gamma K H b^2 \delta.
\end{align*}
Thus, we have 
\begin{align*}
     &\sE_{\gD} \sE_{\tilde{\pi} \sim P_t(\cdot| \gD)} \sE_{f \sim \hat{p}(\cdot|\gD, \tilde{\pi})} \left[0.125 \alpha \gamma K \sum_{h=1}^H \sE_{\mu}[\gE^{\tilde{\pi}}_h(f_h, f_{h+1})(s_h,a_h)^2] + \lambda \Delta f_1(s_1, \tilde{\pi}) \right] \\
     &\lesssim \lambda \eps + \alpha \gamma H b^2 \cdot \max\{ d_{\gF}(\eps), d_{\Pi}(\eps,T), \ln \frac{\ln K b^2}{\delta} \} + \alpha \gamma b^2 K H \cdot \max\{ \eps, \delta \} + \gamma H K \frac{\eps^2}{ \alpha}\\ 
     &+ \sum_{h=1}^H \max_{\tilde{\pi}_h \in \Pi_h^{soft}} \kappa_h(\alpha, \eps, \tilde{\pi}_h) + \sup_{\tilde{\pi} \in \Pi^{soft}(T)} \sum_{h=1}^H \ln \frac{1}{p_0(\gF_h(\eps; Q_h^{\tilde{\pi}_h}))}. 
\end{align*}
\end{proof}

\section{Proof of \Cref{prop: simplified results in a unified form}}
In this appendix, we prove \Cref{prop: simplified results in a unified form}, which is a simple reduction from \Cref{theorem: gopo-vsc}, \Cref{theorem: gopo-roc}, and \Cref{theorem: gopo-psc}.

\begin{proof}[Proof of \Cref{prop: simplified results in a unified form}] We recall that \Cref{prop: simplified results in a unified form} consists of two parts of statements: Part (i) -- the simplified bounds of all three algorithms into one unified form under no misspecification, and Part (ii) -- the specialization of the unified bound into the special cases of finite function classes and linear function classes.

\subsection*{Part (i): The unified sub-optimality bounds for VS, RO, and PS}
We recall that the first part of \Cref{prop: simplified results in a unified form} is that: $\forall \hat{\pi} \in \{\hat{\pi}^{vs}, \hat{\pi}^{ro}, \hat{\pi}^{ps}\}$, 
\begin{align}
    \sE_{\gD}\subopt_{\pi}(\hat{\pi}) = \tilde{\gO}  \left( \frac{Hb}{\sqrt{K}} \sqrt{\tilde{d}(1/K) \cdot \gC(\pi; 1/\sqrt{K}}) + \frac{H b \sqrt{\ln \vol(\gA)}}{T}\right),
    \label{eq: simplified bound -- explicit}
\end{align}
where 
\begin{align*}
\tilde{d}(1/K)=
    \begin{cases}
        \tilde{d}_{opt}(1/K, T) & \text{ if } \hat{\pi} \in \{\hat{\pi}^{vs}, \hat{\pi}^{ro}\},  \\
        \tilde{d}_{ps}(1/K, T) & \text{ if } \hat{\pi} = \hat{\pi}^{ps},
    \end{cases}
\end{align*}
where we recall in \Cref{section: offline learning guarantees} that 
\begin{align*}
    \tilde{d}_{opt}(\eps,T) &:= \max\{d_{\gF}(\eps), d_{\Pi}(\eps,T)\}, \\
    \tilde{d}_{ps}(\eps,T) &:= \max\{ d_{\gF}(\eps), d_{\Pi}(\eps,T), \frac{d_0(\eps)}{\gamma H b^2}, \frac{d'_0(\eps)}{\gamma H b^2} \},
\end{align*}
and $d_{\gF}(\eps)$, $d_{\Pi}(\eps,T)$, $d_0(\eps)$, and $d'_0(\eps)$ are defined in \Cref{section: effective sizes of function class and policy class}. 
Also recall that for \Cref{prop: simplified results in a unified form}, we assume that there is no misspecification, i.e., $\xi_h = \nu_h = 0, \forall h \in [H]$. 

\paragraph{For $\hat{\pi}^{vs}$.} It follows from \Cref{theorem: gopo-vsc}, where we choose $\eps_c = 1/\sqrt{K}$, and $\eps = 1/K$ that with probability at least $1- 2 \delta$, we have 
\begin{align*}
     &\subopt_{\pi}(\hat{\pi}^{vs}) \\
     &\lesssim \sqrt{ K^{-1} \cdot H \cdot \gC(\pi; 1/\sqrt{K})  ( H b^2 \max\{\tilde{d}_{opt}(1/K, T), \ln (H/\delta)\} + b^2  H ) } + H / \sqrt{K} + \zeta_{opt} \\ 
     &\lesssim \sqrt{ K^{-1} \cdot H^2 b^2 \cdot \gC(\pi; 1/\sqrt{K})  \max\{\tilde{d}_{opt}(1/K, T), \ln (H/\delta)\} }  + \frac{H b \sqrt{\ln \vol(\gA)}}{T}
\end{align*}

Thus we have 
\begin{align}
     \subopt_{\pi}(\hat{\pi}^{vs}) = \gO \left( \frac{Hb}{\sqrt{K}} \sqrt{ \gC(\pi; 1/\sqrt{K})  \cdot \max\{\tilde{d}_{opt}(1/K, T), \ln (H/\delta)\} }+  \frac{H b \sqrt{\ln \vol(\gA)}}{T}\right).
     \label{eq: bound for vs - explicit}
\end{align}

\paragraph{For $\hat{\pi}^{ro}$.} The sub-optimality bound for $\hat{\pi}^{ro}$ is obtained from \Cref{theorem: gopo-roc} with the same parameter setting as that for $\hat{\pi}^{vs}$, where we set $\eps = 1/K$, $\eps_c = 1/\sqrt{K}$, and $T \geq K \ln \vol(\gA)$. Additionally, we shall need to set the regularization parameter $\lambda$. Since the bound in \Cref{theorem: gopo-roc} holds for any $\lambda > 0$, we shall minimize this bound with respect to $\lambda > 0$, which results in the optimal $\lambda$ as 
\begin{align*}
    \lambda_* = \sqrt{\frac{2 K H b^2 \cdot \max\{\tilde{d}_{opt}(1/K, T), \ln(H/\delta)\}}{ H \cdot \gC(\pi, 1/\sqrt{K})}}.
\end{align*}
and the sub-optimality bound as 
\begin{align}
     \subopt_{\pi}(\hat{\pi}^{ro}) = \gO \left( \frac{Hb}{\sqrt{K}} \sqrt{ \gC(\pi; 1/\sqrt{K})  \cdot \max\{\tilde{d}_{opt}(1/K, T), \ln (H/\delta)\} } + \frac{H b \sqrt{\ln \vol(\gA)}}{T}\right). 
     \label{eq: bound for ro - explicit}
\end{align}

\paragraph{For $\hat{\pi}^{ps}$.} We specialize the sub-optimality of $\hat{\pi}^{ps}$ from \Cref{theorem: gopo-psc}. Similar to the case of $\hat{\pi}^{vs}$ and $\hat{\pi}^{ro}$, we set: $\eps = 1/K$, $\eps_c = 1/\sqrt{K}$, and $T \geq K \ln \vol(\gA)$. Additionally, we need to set the failure probability $\delta \in [0,1]$, the learning rate $\gamma \in [0, \frac{1}{144 (e-2) b^2}]$ and the regularization parameter $\lambda > 0$. For $\delta$, we set $\delta = 1/K$. For $\lambda$, we minimize the bound in \Cref{theorem: gopo-psc} with respect to $\lambda$, which results into $\lambda = \lambda_*$ which is give as 
\begin{align*}
    \lambda_* = \gamma \sqrt{\frac{ K H b^2 
    \cdot \max\{\tilde{d}_{ps}(1/K, T), \ln (K \ln (K b^2))\}}{ H \cdot \gC(\pi, 1/\sqrt{K})}},
\end{align*}
turns the sub-optimality bound into 
\begin{align}
    &\sE_{\gD} \subopt_{\pi}(\hat{\pi}^{ps}) \nonumber\\
    &= \gO \left( \frac{Hb}{\sqrt{K}} \sqrt{ \gC(\pi; 1/\sqrt{K})  \cdot \max\{\tilde{d}_{ps}(1/K, T), \ln (K \ln(K b^2))\} } + \frac{H b \sqrt{\ln \vol(\gA)}}{T} \right).
    \label{eq: bound for ps - explicit}
\end{align}
Finally, we choose $\gamma \in [0, \frac{1}{144 (e-2) b^2}]$ to minimize $\tilde{d}_{ps}(\eps,T) = \max\{ d_{\gF}(\eps), d_{\Pi}(\eps,T), \frac{d_0(\eps)}{\gamma H b^2}, \frac{d'_0(\eps)}{\gamma H b^2} \}$, which occurs at $\gamma = \frac{1}{144 (e-2) b^2}$, and thus 
\begin{align}
    \tilde{d}_{ps}(\eps, T) = \max\left\{ d_{\gF}(\eps), d_{\Pi}(\eps,T), \frac{144 (e-2) d_0(\eps)}{ H }, \frac{144 (e-2)d'_0(\eps)}{ H} \right\}. 
    \label{eq: simplified d_ps}
\end{align}

Overall, we have that \Cref{eq: bound for vs - explicit}, \Cref{eq: bound for ro - explicit}, and \Cref{eq: bound for ps - explicit} can be unified into \Cref{eq: simplified bound -- explicit}.

\subsection*{Part (ii): Specializing to the finite function classes and linear function classes }
We consider two common cases. 

\paragraph{Case 1. Finite function class.} We consider the case that $\gF_h$ and $\Pi^{soft}_h(T)$ have finite elements for all $h \in [H]$. Then we have 
$\tilde{d}(\eps) = \gO(\max_{h \in [H]} \max \{ \ln |\gF_h|, \ln |\Pi_h^{soft}(T)| \}), \forall \eps$, due to that $d'_0(\eps) \leq d_0(\eps) \leq H \max_{h \in [H]} \ln |\gF_h|$ and \Cref{eq: simplified d_ps}.

\paragraph{Case 2. Linear function class.}
We consider the case that the function class $\gF_h$ is linear in some (known) feature map $\phi_h: \gS \times \gA \rightarrow \sR^d$. Concretely, the corresponding function class and the policy class defined in \Cref{section: function approximation} are simplified into:
\begin{align*}
    \gF_h &= \{(s,a) \mapsto \langle \phi_h(s,a), w \rangle : \| w\|_2 \leq b\}, \\ 
    \Pi^{soft}_h(T) &:= \left\{ (s,a) \mapsto \frac{\exp(\langle \phi_h(s,a), \theta \rangle)}{\sum_{a' \in \gA}\exp(\langle \phi_h(s,a'), \theta \rangle)}: \| \theta\|_2 \leq \eta T \right\}.
\end{align*}
We have 
\begin{align*}
    d_{\gF}(\eps) &\leq d  \ln(1 + \frac{2b}{\eps}), \\ 
    d_{\Pi}(\eps,T) &\leq d \ln(1 + \frac{16 \eta T}{\eps}), \\ 
    d'_0(\eps) \leq d_0(\eps) &\leq c_1 d H \ln (c_2/\eps), 
\end{align*}
where the first two inequalities use \cite[Lemma~6]{zanette2021provable} and the last inequality follows the discussion in \Cref{section: effective sizes of function class and policy class}. Note that $d_{\Pi}(\eps,T)$ depends only logarithmically in $T$. 
    
\end{proof}

\section{Support Lemmas}
In this section, for convenience, we present some simple yet useful lemmas that our proofs above often refer to. 

The following lemma establishes the variance condition for the squared loss, which is typically used along with Bernstein's inequality. 
\begin{lemma}
Consider any real-valued function class $\gF$. Consider the squared loss $L(f(x),y) = (f(x) - y)^2$. Assume bounded loss $L(f(x),y) \leq M^2$ for any $f \in \gF$, for some $M > 0$. Let $f_*(x) = \sE[y|x]$ and assume that $L(f_*(x),y) \leq B^2$ for some $B > 0$ (we do not require that $f_* \in \gF$). Let $z = (x,y)$ and define 
\begin{align*}
    \gG = \{\phi(\cdot): \phi(z) = L(f(x),y) - L(f_*(x),y), f \in \gF\}.
\end{align*}
Then, for all $\phi \in \gG$, we have 
\begin{align*}
    \sE_y[\phi(z)^2] \leq 2 (M^2 + B^2) \sE_y[\phi(z)], \forall x.
\end{align*}
where $\sE_{y}$ is the expectation taken over $y$ given $x$. 
\label{lemma: variation condition for least squares}
\end{lemma}
\begin{proof}[Proof of \Cref{lemma: variation condition for least squares}]
Consider any $\phi \in \gG$ (with the corresponding $f \in \gF$). For any $x$, we have 
\begin{align*}
    \sE_y[\phi(z)] = (f(x) - f_*(x))^2.
\end{align*}
Thus, we have
\begin{align*}
    \phi(z)^2 &= (f(x) - f_*(x))^2 (f(x) + f_*(x) - 2y)^2 \\ 
    &\leq (f(x) - f_*(x))^2 2 [(f(x) - y)^2 + (f_*(x) - y)^2] \\
    &\leq 2 (M^2 + B^2) (f(x) - f_*(x))^2. 
\end{align*}
The first inequality uses Cauchy-Schwartz. The second inequality uses that $f_* \in \gF$ and $L(f(x),y) \leq M^2, \forall f \in \gF$. Thus, for any $x$, we have 
\begin{align*}
    \sE_y[\phi(z)^2] \leq 2 (M^2 + B^2) (f(x) - f_*(x))^2 = 2 (M^2 + B^2) \sE_y[\phi(z)].
\end{align*}
The equation uses that $ \sE_y[\phi(z)] = (f(x) - f_*(x))^2$.
\end{proof}

The following lemma is a simple decomposition of the value gap in the initial state, typically known as the performance difference lemma in the RL literature.
\begin{lemma}[Performance difference lemma]
For any policy $\pi, \widetilde{\pi}$, we have 
\begin{align*}
    V_1^{\pi}(s_1) - V_1^{\widetilde{\pi}}(s_1) = \sum_{h=1}^H \sE_{\pi} \left[ Q^{\widetilde{\pi}}_h(s_h,a_h) - V^{\widetilde{\pi}}_h(s_h) \right],
\end{align*}
where $\sE_{\pi}$ denotes the expectation over the random trajectory $(s_1,a_1, \ldots, s_h, a_h)$ generated by $\pi$ (and the underlying MDP).
\label{lemma: performance difference}
\end{lemma}
\begin{proof}[Proof of \Cref{lemma: performance difference}]
We simply expand $V_1^{\pi}(s_1) = \sE_{a_1, s_2 | s_1, \pi} [r_1(s_1, a_1) + V_2^{\pi}(s_2)]$ and use recursion to obtain the lemma.
\end{proof}

The following lemma presents a simple connection from a form of a log partition function to the expectation under the infimum realization of the sampling distribution. 
\begin{lemma}
For any density functions $p$ and $p_0$ and any function $f$, we have
\begin{align*}
    \inf_{p}\sE_{x \sim p(x)} \left[f(x) + \ln \frac{p(x)}{p_0(x)} \right] \geq - \ln \sE_{x \sim p_0} \exp(-f(x)).
\end{align*}
\label{lemma: log partition is lower bound for f + log p/p_0}
\end{lemma}
\begin{proof}[Proof of \Cref{lemma: log partition is lower bound for f + log p/p_0}]
Define the density function 
\begin{align*}
    q(x) = \frac{p_0(x) \exp(-f(x))}{Z(f)} \text{ where } Z(f):= \sE_{x \sim p_0(x)} \exp(-f(x)). 
\end{align*}
Then, we have 
\begin{align*}
     \sE_{x \sim p(x)} \left[f(x) + \ln \frac{p(x)}{p_0(x)} \right] &= \sE_{x \sim p(x)}  \ln \frac{p(x)}{q(x)} - \ln Z(f) \\ 
     &=KL[p \| q] - \ln Z(f) \\
     &\geq - \ln Z(f).
\end{align*}
\end{proof}

\end{appendices}
\end{document}